\documentclass[10pt,journal,compsoc]{IEEEtran}

\ifCLASSOPTIONcompsoc
  \usepackage[nocompress]{cite}
\else
  \usepackage{cite}
\fi
\ifCLASSINFOpdf
\else
\fi
\usepackage{bm}
\usepackage{microtype}

\usepackage[T1]{fontenc}    
\usepackage[dvipsnames]{xcolor}
\usepackage{graphicx}
\usepackage{subfigure}
\usepackage{booktabs} 
\usepackage{multirow}
\usepackage{amssymb}
\usepackage{algorithmic}
\usepackage[ruled,linesnumbered]{algorithm2e}
\definecolor{lightseagreen}{rgb}{0.13, 0.7, 0.67}
\definecolor{highlight}{HTML}{355C7D}
\usepackage[colorlinks,linkcolor=Bittersweet,citecolor=CarnationPink]{hyperref}

\usepackage{pifont}

\usepackage{amsthm}
\usepackage{amsmath} 
\usepackage{ragged2e}

\usepackage{subfiles}
\usepackage{hyperref}
\usepackage{caption2}
\usepackage{enumitem}

\DeclareMathOperator*{\argmin}{argmin}
\newcommand{\PreserveBackslash}[1]{\let\temp=\\#1\let\\=\temp}
\newcommand{\norm}[1]{\lVert#1\rVert}
\newtheorem{thm}{Theorem}

\newtheorem{prop}{Proposition}
\newtheorem{defi}{Definition}
\newtheorem{proper}{Property}
\newtheorem{lem}{Lemma}
\newtheorem{rem}{Remark}
\newtheorem{exam}{Example}

\newtheorem*{pfthm4}{Proof of Theorem 4}
\newtheorem*{pfthm5}{Proof of Theorem 5}
\newtheorem*{pfthm6}{Proof of Theorem 6}
\newtheorem*{pfthm7}{Proof of Theorem 7}
\newtheorem*{pfthm8}{Proof of Theorem 8}

\newcommand{\laml}[1]{\lambda_{#1}\left(\mathcal{L}_{\mathcal{G}_{BI}}\right)}

\newcommand{\lamc}[1]{\breve{\lambda}_{#1}}
\def \fk#1 {\mathcal{F}(\boldsymbol{W}^{#1},\boldsymbol{U}^{#1})}
\def \fkt#1 {\widetilde{\mathcal{F}}(\boldsymbol{W}^{#1},\boldsymbol{U}^{#1})}
\def \lk#1 {\mathcal{J}(\boldsymbol{W}^{#1})}

\def \Gbi{\mathcal{G}_{BI}}
\def \LGbi{\mathcal{L}_{\Gbi}}
\newcommand{\lam}[1]{\lambda_{#1}(\LGbi)}
\def \Ebi{\mathcal{E}_{BI}}
\def \Vbi{\mathcal{V}_{BI}}
\def \Abi{\bm{A}_{BI}}

\def \inner{\left<\LGbi, \bm{U}\right>}

\def \c{\bm{\varTheta}_c }
\def \g{\bm{\varTheta}_g }
\def \p{\bm{\varTheta}_p }

\def \ct{\bm{\varTheta}_c^t }
\def \gt{\bm{\varTheta}_g^t }
\def \pt{\bm{\varTheta}_p^t }

\def \ctp1{\bm{\varTheta}_c^{t+1} }
\def \gtp1{\bm{\varTheta}_g^{t+1} }
\def \ptp1{\bm{\varTheta}_p^{t+1} }

\def \regw{\frac{\alpha_2}{2} \norm{\bm{W}}_F^2}
\def \regu{\frac{\alpha_3}{2} \norm{\bm{U}}_F^2}

\AtBeginDocument{%
	\providecommand\BibTeX{{%
			\normalfont B\kern-0.5em{\scshape i\kern-0.25em b}\kern-0.8em\TeX}}}
\newcommand{\Fl}[1]{\mathcal{F}(\bm{W}^{#1}, \bm{U}^{#1} )}
\newcommand{\distFl}[1]{dist(\bm{0}, \partial_{\Theta} \Fl{#1})}
\newcommand{\distFlw}[1]{dist(\bm{0}, \partial_{\bm{W}} \Fl{#1})}
\newcommand{\distFlu}[1]{dist(\bm{0}, \partial_{\bm{U}} \Fl{#1})}
\usepackage{bbm}
\newcommand{\refeq}[1]{Eq.(\ref{#1})}

\usepackage[framemethod=tikz]{mdframed}
\usepackage{lipsum}

\begin{document}
%
\title{Task-Feature Collaborative Learning with Application to  Personalized  Attribute Prediction}
%
%
%
%

\author{Zhiyong~Yang, Qianqian~Xu*,~\IEEEmembership{Senior Member,~IEEE,}
        Xiaochun~Cao,~\IEEEmembership{Senior Member,~IEEE,}
        and~Qingming~Huang*,~\IEEEmembership{Fellow,~IEEE}
\IEEEcompsocitemizethanks{
\IEEEcompsocthanksitem *: corresponding authors
\IEEEcompsocthanksitem Zhiyong Yang is with the State Key Laboratory of Information Security, Institute of Information Engineering, Chinese Academy of Sciences, Beijing 100093, China, and also with the School of Cyber Security, University of Chinese Academy of Sciences, Beijing 100049, China (email: \texttt{yangzhiyong@iie.ac.cn}).\protect\\
\IEEEcompsocthanksitem Qianqian Xu is with the Key Laboratory of
Intelligent Information Processing, Institute of Computing Technology, Chinese
Academy of Sciences, Beijing 100190, China, (email: \texttt{xuqianqian@ict.ac.cn}).\protect\\
\IEEEcompsocthanksitem Xiaochun Cao is with the State Key Laboratory of Information Security, Institute of Information Engineering, Chinese Academy of Sciences, Beijing 100093, China, also with the Cyberspace Security Research Center, Peng Cheng Laboratory, Shenzhen 518055, China, and also with the School of Cyber Security, University of Chinese Academy of Sciences, Beijing 100049, China (email: \texttt{caoxiaochun@iie.ac.cn}).\protect\\
\IEEEcompsocthanksitem Q. Huang is with the School of Computer Science and Technology,
    University of Chinese Academy of Sciences, Beijing 101408, China, also
    with the Key Laboratory of Big Data Mining and Knowledge Management (BDKM),
    University of Chinese Academy of Sciences, Beijing 101408, China,  also
    with the Key Laboratory of Intelligent Information Processing, Institute of
    Computing Technology, Chinese Academy of Sciences, Beijing 100190, China, and also with Peng Cheng Laboratory, Shenzhen 518055, China
    (e-mail: \texttt{qmhuang@ucas.ac.cn}).\protect\\
\IEEEcompsocthanksitem  This work was supported in part by the National Key R\&D Program of China under Grant 2018AAA0102003, in part by National Natural Science Foundation of China: 61620106009, U1636214, U1736219, 61971016, 61931008, 61836002, 61672514 and 61976202, in part by Key Research Program of Frontier Sciences, CAS: QYZDJ-SSW-SYS013, in part by the Strategic Priority Research Program of Chinese Academy of Sciences, Grant No. XDB28000000, in part by Beijing Natural Science Foundation (4182079), and in part by Youth Innovation Promotion Association CAS.
   }
   }
%
%

\markboth{To Appear in IEEE TRANSACTIONS ON PATTERN ANALYSIS AND MACHINE INTELLIGENCE}%
{Yang \MakeLowercase{\textit{et al.}}: Bare Demo of IEEEtran.cls for Computer Society Journals}
%



\IEEEtitleabstractindextext{%
\begin{abstract}
\justifying
As an effective learning paradigm against insufficient training samples, Multi-Task Learning  (MTL) encourages knowledge sharing across multiple related tasks so as to improve the overall performance.  In MTL, a major challenge springs from the phenomenon that sharing the knowledge with dissimilar and hard tasks, known as \emph{negative transfer}, often results in a worsened performance. Though a substantial amount of studies have been carried out against the negative transfer, most of the existing methods only model the transfer relationship as task correlations, with the transfer across features and tasks left unconsidered.  Different from the existing methods, our goal is to alleviate negative transfer collaboratively across features and tasks. To this end, we propose a novel multi-task learning method called Task-Feature Collaborative Learning (TFCL). Specifically,  we first propose a base model with a heterogeneous block-diagonal structure regularizer to leverage the collaborative grouping of features and tasks and suppressing inter-group knowledge sharing. We then propose an optimization method for the model. Extensive theoretical analysis shows that our proposed method has the following benefits: (a) it enjoys the global convergence property and (b) it provides a block-diagonal structure recovery guarantee.  As a practical extension, we extend the base model by allowing overlapping features and differentiating the hard tasks. We further apply it to the personalized attribute prediction problem with fine-grained modeling of user behaviors. Finally, experimental results on both simulated dataset and real-world datasets demonstrate the effectiveness of our proposed method.
\end{abstract}

\begin{IEEEkeywords}
Block Diagonal Structural Learning, Negative Transfer, Multi-task Learning, Global Convergence.
\end{IEEEkeywords}}

\maketitle

\IEEEdisplaynontitleabstractindextext

%
\IEEEpeerreviewmaketitle

\IEEEraisesectionheading{\section{Introduction}\label{sec:introduction}}

%
%
%
%

\IEEEPARstart{I}{t} is well-known that the success of machine learning methods rests on a large amount of training data.  However, when training samples \emph{are hard to collect}, how to improve model performance with \emph{small datasets} then becomes a great challenge.  When facing multiple related tasks together, Multi-Task Learning (MTL) \cite{mtl} is well-known as a good solution against such a challenge. In a general sense, MTL refers to the learning paradigm where multiple tasks are trained jointly under certain constraints leveraging knowledge transfer across some/all of the tasks. As typical evidence for the wisdom behind MTL, an early study \cite{limit2} shows that training a large number of similar tasks together could significantly improve the generalization performance when the data annotations of each individual task are insufficient. Nowadays, such wisdom has been widely adopted by the machine learning community, which makes MTL a crucial building block for a plethora of applications ranging from scene parsing \cite{cv1}, attribute learning \cite{cv2}, text classification \cite{nlp1}, sequence labeling \cite{nlp2}, to travel time estimation \cite{dm1}, \emph{etc}.

The fundamental belief of MTL lies in that sharing knowledge among multiple tasks often results in an improvement in generalization performance. Based on the belief, a great number of studies have been carried out to explore the problem of how to share valuable knowledge across different tasks. The early studies on this topic argue that knowledge should be shared across all the tasks. For example, in the work of \cite{simp},  knowledge is transferred  by sharing common and sparse features across all the tasks.  However, \cite{non-overlapping1} later points out that if not all the tasks are indeed related,  sharing common features with dissimilar and hard tasks often results in performance degradation, which is termed as \textit{negative transfer}. To address this issue, recent studies in the odyssey against \textit{negative transfer} usually fall into two brunches.\\
\indent  One line of studies casts the tasks grouping problem as a clustering method. At the very beginning, \cite{clus_early} proposes an MTL algorithm which first constructs the task clusters and then learns the model parameters separately. Seeing that this stage-wise method  could not guarantee the optimality of the learned clusters and parameters, a significant number of studies start to explore how to integrate clustering and multi-task learning into a unified framework.  Generally speaking, such work could be classified  into two categories. The first class of methods adopts a Bayesian learning framework, which assumes that the task-specific parameters subject to cluster-leveraging priors such as mixtures of Gaussian prior \cite{MG} and Dirichlet process prior \cite{Dp1,Dp2,Dp3}.   The second class of methods formulates the clustering problem as regularization terms.  More specifically, such terms  are developed to: (a)penalize small between-cluster variance and large within-cluster variance, (b) relax mix-integer programmings  \cite{non-overlapping1, non-overlap2},  (c) encourage structural sparsity \cite{metag,structual,k-support}, or (d) encourage latent task representation \cite{overlap1,overlap2}.\\
\indent The other line of studies realizes that knowledge transfer should not be treated symmetrically. In fact, transferring knowledge from easy to hard  tasks is generally safe, while transferring knowledge along the opposite direction is the major source of negative transfer and should be suppressed.  Motivated by this, \cite{amtl} proposes the first MTL method to leverage asymmetry. In this work, the task parameters are assumed to lie in the column space spanned by themselves, and the asymmetry is then realized by leveraging sparse representation coefficients coming from each task. Since then, some improvements have been made on this framework via (a)  making the penalty adaptive to the correlation between predictors \cite{self2}; (b) latent task representation \cite{self1,trace_lasso}; (c) grouping constraint\cite{asym-group} and d) robust constraints \cite{asym-robust}.\\
\indent For the majority of existing methods, the negative transfer issue is only modeled as task correlation/grouping. With the following motivating examples, it could be seen that, even when the tasks are reasonably grouped, sharing redundant features across different task groups still bears the risk for over-fitting. This suggests that \emph{negative transfer might as well take place across features and tasks}.
\begin{exam}
Consider the personalized learning problem, where the prediction of the decision results coming from a given user is regarded as a task, and the features capture different concepts of the given object that a user might be interested in.  When making decisions toward an object, users often form disentangled groups in terms of their points-of-interests (color/shape/texture, etc.). Consequently, blindly sharing irrelevant features across groups is dangerous. 
\end{exam}

\begin{exam}
In bioinformatics, seeking out connections between genetic markers and phenotypes is often regarded as a crucial problem. MTL could be applied to this problem, for example, when we expect to simultaneously predict a set of different diseases from genetic clues. Here every specific disease prediction is regarded as a task, and features come from the expressions of different markers. Under this scenario, since  different groups of diseases involve different biological functions, it is natural to observe them correlated with different sets of genes. Sharing common genes across disease groups is then not reasonable, and might lead to negative transfer even when we have a good grouping of the diseases.   
\end{exam}
\begin{exam}
Another example comes from the natural language processing tasks. For instance, consider the topic classification problem where identifying a specific topic is regarded as a task and the embeddings of words in a document are regarded as features. It is often observed that different groups of topics are relevant to different subsets of words. The Sport-related topics often involve keywords such as \emph{athletics}, \emph{soccer}, \emph{gymnastics}, while the Politics related keywords often include \emph{governance}, \emph{election}, \emph{parliaments}. Then sharing common words across these two topic groups then bears a high risk of over-fitting.
\end{exam}

 Motivated by these examples, our goal in this paper is to seek out solutions for \textit{negative transfer} in a more \emph{general} manner such that the features could come into play for the grouping process. To realize our goal, we need to \emph{include} the co-partition of tasks and features as an important component of MTL. To this end, we formulate a new learning framework named Task-Feature Collaborative Learning (TFCL). Specifically, we construct our framework with three steps.\\
\noindent\textbf{ {Step 1}}. In the first step,  we propose a base model that achieves the co-grouping target with a novel regularizer based on   block-diagonal structure learning of a bipartite graph with features and tasks being its nodes.\\
\noindent \textbf{{Step 2}}. Since the resulting optimization problem of the base model, denoted in short as $(\bm{P})$, is non-smooth and non-convex, we propose a surrogate problem  $(\bm{P}^\star)$ as an approximation.  Through developing an optimization algorithm for $(\bm{P}^\star)$, we prove that it can simultaneously achieve the global convergence for $(\bm{P})$ and $(\bm{P}^\star)$ under certain assumptions.  Besides the convergence analysis, it is also noteworthy that the intermediate solution produced by the optimization method implicitly provides an embedding for each feature and task. With these embeddings, we further show that the optimization
algorithm could leverage the expected block-diagonal structure if the parameters are carefully chosen. This naturally leverages a grouping effect across task/feature, where transferring across groups is suppressed. \\
\noindent \textbf{{Step 3}}. With the base model elaborated, we then turn our focus to a more comprehensive model and target at an application problem called \emph{ personalized attribute prediction}, where the personalized attribute annotation prediction for a given user is regarded as a task. To obtain a flexible model, we simultaneously consider (a) the consensus factor that captures the popular interests shared by the users, which allows group overlapping (b) the co-grouping factor in our base model (c) the abnormal factor that excludes outlier users (tasks) from grouping. We also prove that this method inherits all the theoretical properties of TFCL.\\
In a nutshell, the main contributions of this paper are three-fold:

\begin{enumerate}
\item [(C1)] In the core of the base model of TFCL framework lies  a novel block-diagonal regularizer  leveraging the task-feature co-partition. This allows us to explore a more general negative transfer effect simultaneously at the task- and feature-level. 

\item[(C2)]  An optimization algorithm is designed for the base model with a theoretical guarantee for the global convergence property. Moreover, we provide a theoretical guarantee for leveraging the expected block-diagonal structure.

\item[(C3)]  Finally, we propose a more practical extension for the personalized attribute prediction problem based on TFCL with enhanced flexibility.  
\end{enumerate}

\begin{figure*}[t]
  \centering
  \includegraphics[width=0.95\textwidth]{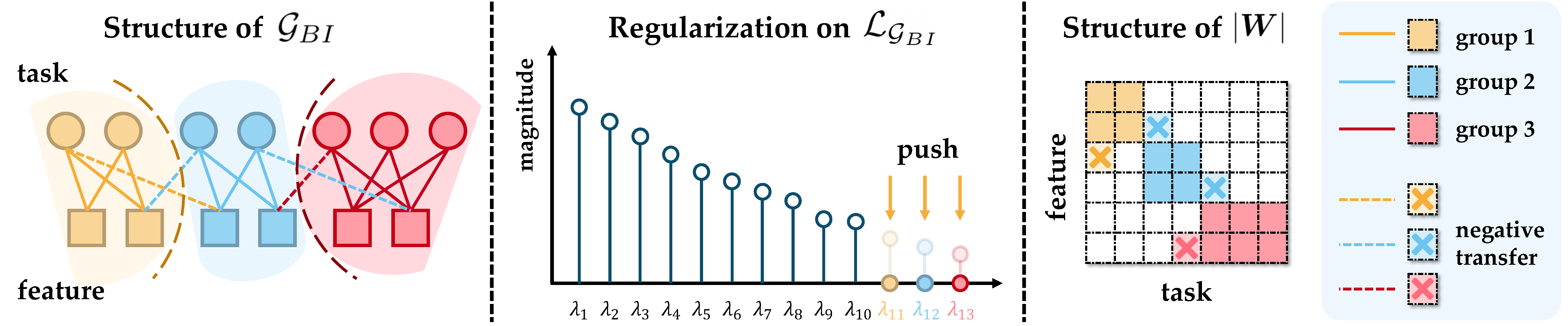}
  
    \caption{  
{\textbf{Illustration of the Base Model of the  Task-Feature Collaborative Learning Framework}}. \textbf{Left}: We form the task-feature relations as an auxiliary bipartite graph $\Gbi$ with tasks and features being the nodes, and $\LGbi$ being the Graph Laplacian. To separate all the tasks and features into $k$ groups, we  expect to cut $\Gbi$ into $k$ connected components.  \textbf{Middle}: If we reconsider this requirement from the Graph Laplacian, as is shown in Thm. \ref{thm:graph}, it is equivalent to force the smallest $k$ eigenvalues of $\LGbi$ to be zero. Since directly doing this is intractable, we turn to minimize their sum as a relaxation, which gives birth to a novel regularizer based on Thm. \ref{thm:eig}. \textbf{Right}: Now we shift our attention to the model parameters $\bm{W}$. The proposed regularizer facilitates a generalized block-diagonal structure (up to permutations) toward $\bm{W}$, with each block containing a specific group of nodes in $\Gbi$. In the next section, based on Prop. \ref{prop:sol}, Thm. \ref{eigsol}-\ref{thm:global}, we will construct an optimization method for TFCL with global convergence guarantee.  Moreover, we will also show in Thm. \ref{thm:group} that, negative transfer across blocks, marked as crosses here, could be effectively suppressed based on the algorithm.
}
  \label{fig:illu}
\end{figure*}
\section{Related Work}
In this section, we briefly review the recent advances in block-diagonal structural learning,  multi-task learning and personalized attribute prediction that are closely related to our study.

\noindent \textbf{Block-Diagonal Structural Learning.} The idea to learn block-diagonal structures could be traced back to the clustering problem. For the clustering problem, a set of data points are required to be grouped into several clusters in an unsupervised manner. As a representative type of method, graph-based clustering methods (\emph{e.g}. spectral clustering \cite{spec1,spec2} and subspace clustering \cite{SSC,LRR}) solve this problem in a two-stage way: (a)  a proper sample-sample affinity matrix is first obtained to capture the correlations among different sample points; (b) Given the affinity matrix, the clustering problem is formulated as a graph partitioning problem via minimizing some spectral relaxations of the normalized cut. Under this framework, if the affinity matrix has a clear block-diagonal structure, then each of the block components naturally forms a cluster. Consequently, leveraging the block-diagonal structure of the affinity matrix could significantly improve the performance of such graph-based clustering methods. Motivated by this intuition, researchers start to find implicit regularization terms to preserve the block-diagonal structural properties of the affinity matrix  \cite{s3c,sscomp,DDSSC,SCTwist,ssqp,lrsc,spsc1}. However, as suggested by \cite{lublock}, the implicit regularizers could not deal with the off-diagonal noises from the null spaces of the feature inputs.  Then  \cite{diag1,diag2} present the first trail to develop explicit block-diagonal regularizers in the graph-based clustering framework as a better solution against this issue. Most recently, this framework has been successfully extended to subspace clustering frameworks to embrace self-expression \cite{lublock,xie2017implicit,yang2019split}. Along this line of research, the most related studies to our work come from the explicit block-diagonal regularizations. However, they differ significantly with our work from two dimensions. First, they target at homogeneous sample grouping, where the block-diagonal property is merely limited to square matrices with the $i$-th column and $i$-th row representing the same sample point. In this paper, the task-feature co-grouping pursuit calls for a more generalized definition of block-diagonality. To this end, we propose a generalized block-diagonal structural learning framework which is available for arbitrary size matrices where the $i$-th column/row refers to heterogeneous type of nodes (task/feature in our case). 
Second, concerning the optimization method,  they only provide a subsequence convergence guarantee, leaving the global convergence property an open problem. By contrast, we will show that our proposed method could guarantee the global convergence property with a specific construction of an auxiliary surrogate problem.
We will have a closer look at the connection between the related work and our method in Sec.\ref{sec:disscus}. \\
\noindent \textbf{Multi-task Learning.} In the introduction, we have provided a brief review of the majority of methods attacking negative transfer issue in MTL. Here, we further discuss MTL methods that are closely related to our work. Firstly, from the structural learning perspective, the asymmetric MTL methods mentioned in the introduction section \cite{amtl,self1,asym-robust,asym-group} also leverage block-diagonal structures. Just as mentioned in the last paragraph, they only consider homogeneous block-diagonal structures at the task-level, which does not meet our requirement of heterogeneous block-diagonal structures across tasks and features. From the task-feature collaborative learning perspective, to the best of our knowledge, there are only two MTL studies that also explicitly learn the task-feature co-grouping structures. As a beginning trial, \cite{cluster} explores how different features play a role in multi-task relationships. Specifically, it designed a novel multi-task model via leveraging feature-specific task clusters. However, the features in \cite{cluster} are considered separately, with the complicated feature correlations left unconsidered. \cite{comtl} turns to learn the feature-task correlations based on a co-clustering guided regularization. However, there is no direct guarantee to ensure that the regularization scheme could explicitly leverage the block-diagonal structure. More importantly, it does not provide an explicit connection between the feature-task co-clustering and the negative transfer across tasks and features.

\noindent \textbf{Personalized Attribute Predictions.} Attribute learning has long been playing a central role in many machine learning and computer vision problems. While most attribute learning methods adopt consensus annotations, recently, with the rise of the crowdsourcing platform, there is an emerging wave to study how to train user-specific models for personalized annotations. An early trial presented in \cite{user1} learns user-specific attributes with an adaption process. More precisely, a general model is first trained based on a large pool of data. Then a small user-specific dataset is employed to adapt the trained model to user-specific predictors. \cite{user2} argues that one attribute might have different interpretations for different groups of persons. Correspondingly, a shade discovery method is proposed therein to leverage group-wise user-specific attributes. The common issue of these methods is that they only adopt a stage-wise training scheme.  Most recently, the work presented in \cite{user3} starts an early trial to jointly learn personalized annotations across multiple attributes.  In this paper, we will have a closer look at the negative transfer issue in this application with a fine-grained modeling of the user-feature correlations based on the proposed task-feature collaborative learning method.
\section{Task-Feature Collaborative Learning: The General Framework}
In this section,  we propose the base model for Task-Feature Collaborative Learning (TFCL), which suppresses negative transfer across tasks and features. 
 In a nutshell, a summary of our method is illustrated in Fig.\ref{fig:illu}. In this section, our main assumption is  that tasks and features could be simultaneously clustered into different groups. \textit{Nonetheless,  our work does not restrict to the co-grouping assumption. In section \ref{sec:per}, we will extend TFCL to include outlier tasks and consensus features. }\\

\noindent\textbf{Notations}. The notations adopted in this paper are enumerated as follows. $\mathbb{S}_m$ denotes the set of all symmetric matrices in $\mathbb{R}^{m \times m}$.  Given $\bm{A} \in \mathbb{S}_N$, we number the $N$ eigenvalues $\bm{A}$ in an ascending order as $\lambda_1(\bm{A}) \le \lambda_2(\bm{A}) \le \cdots \le  \lambda_N(\bm{A})$. $\left<\cdot,\cdot\right>$ denotes the inner product for two matrices or two vectors. Given two matrices $\boldsymbol{A}$ and $\boldsymbol{B}$, $\boldsymbol{A} \oplus \boldsymbol{B}$ denotes the direct sum of two matrices, and we say $\boldsymbol{A} \succeq \boldsymbol{B} $, if $\boldsymbol{A} - \boldsymbol{B}$ is positive semi-definite. For distributions, $\mathbb{U}(a,b)$ denotes the uniform distribution and $\mathbb{N}(\mu, \sigma^2)$ denotes the normal distribution.
For a given matrix $ \bm{A} \in \mathbb{R}^{m \times n}$, the null space is defined as $\mathcal{N}(\bm{A}) =\{\bm{x} \in \mathbb{R}^{n}: \bm{A}\bm{x} = \bm{0}\}$.  Given $\bm{A} \in \mathbb{S}_{m}$, if $\lambda$ is an eigenvalue of $\bm{A}$, $\mathbb{EIG}_{A}(\lambda) = \mathcal{N}(\bm{A} - \lambda\bm{I})$ is the subspace spanned by the eigenvectors associated with $\lambda$. Given a matrix $\bm{A}  = [\bm{a}_1, \cdots, \bm{a}_n]$, we denote $\bm{A}_{m:n} = [\bm{a}_m, \cdots, \bm{a}_n]$. We have $\iota_\mathcal{A}(x)= 0$, if $x \in \mathcal{A}$, and  $\iota_\mathcal{A}(x)= +\infty$, otherwise.

\noindent\textbf{Standard multi-task learning paradigm}. Given $T$ tasks to be learned simultaneously, we denote the training data as: \[\mathcal{S} = \left\{(\boldsymbol{X}^{(1)}, \boldsymbol{y}^{(1)}), \cdots, (\boldsymbol{X}^{(T)}, \boldsymbol{y}^{(T)})\right\}.\] For $\mathcal{S}$, $\boldsymbol{X}^{(i)} \in \mathbb{R}^{n_{i} \times d}$ is the input feature matrix for the $i$-th task, where  $n_{i}$ denotes the number of instances and $d$ represents the feature dimension. Each row of $\boldsymbol{X}^{(i)} $ represents the feature vector for a corresponding instance in the task. $\boldsymbol{y^{(i)}}$ is the corresponding response or output variable for the $i$-th task. We train a linear model $\boldsymbol{g}^{(i)}(x) = \boldsymbol{W}^{(i)^\top}\boldsymbol{x}$ for each involved task. The \textit{parameter matrix} $\boldsymbol{W} \in \mathbb{R}^{d \times T}$ as the concatenation of task coefficients in the form  $\boldsymbol{W} = [\boldsymbol{W}^{(1)}, \cdots, \boldsymbol{W}^{(T)}]$.  Following the standard MTL learning paradigm, $\boldsymbol{W}$ could be solved from a regularized problem
$
\argmin_{\boldsymbol{W}} \mathcal{J}(\boldsymbol{W}) + \alpha \cdot \varOmega(\boldsymbol{W})
$, where $\ell_i$ denotes the empirical risk for the $i$-th task, $\mathcal{J}(\boldsymbol{W}) = \sum_i \ell_i$,  $\varOmega(\boldsymbol{W})$ denotes a proper regularizer. In the rest of this section, we derive a proper regularizer that suppresses the negative transfer collaboratively from both the task- and feature-level.
\\
\\
\noindent Under the linear model, for the $i$-th task, the prediction of the response could be written as $\hat{{y}} =  \boldsymbol{W}^{(i)^\top}\boldsymbol{x} = \sum_{j = 1}^d  W_{ij}x_j$. Accordingly, if $W_{ij} = 0$, $\hat{{y}}$ becomes irrelevant to the $j$-th feature. While for those nonzero $W_{ij}$s, $\hat{{y}}$ tends to have a stronger dependence on the features with a greater value of $|W_{ij}|$. In this way, $|W_{ij}|$ provides a proper expression of the correlation between feature $i$ and task $j$. To alleviate negative transfer from both dissimilar tasks and dissimilar features, our basic setting in TFCL is to automatically separate the tasks and features into a given number of groups, where each group only contains similar tasks and features. In this scenario, negative transfer comes when $W_{ij} \neq 0$ if feature $i$ and task $j$ are not in the same group. Inspired by this, \textit{our goal is  to search for a simultaneous partition of tasks and features into $k$ groups, where we expect $|W_{ij}|$ to vanish when feature $i$ and task $j$ are not in the same group}. At first glance, formulating this constraint as a regularizer is difficult. However, if we turn to define this constraint on an auxiliary bipartite graph, then a simple regularization term   realizing this constraint could be constructed. Specifically, we define a  bipartite graph $\Gbi = (\Vbi,\Ebi,\bm{A}_{BI})$. The vertices of $\Gbi$ include all the tasks and features. Denote $\mathcal{V}_{T}$ as the set of all tasks and $\mathcal{V}_{F}$ as the set of all features, the vertex set $\Vbi$ is defined as $\Vbi=\mathcal{V}_{T}\cup \mathcal{V}_{F}$. The affinity matrix $\Abi$ is in the form 
$
\bm{A}_{{BI}} = \begin{bmatrix}
\boldsymbol{0} & \boldsymbol{|\boldsymbol{W}|}\\
\boldsymbol{|\boldsymbol{W}|^\top} & \boldsymbol{0}\\
\end{bmatrix}
$,
then the edge set naturally becomes $\Ebi=\{(i,j)| \bm{A}_{{BI}_{{i,j}}} > 0 \}$. Besides the graph itself, we also employ the  graph Laplacian matrix defined as $\LGbi = diag(\Abi\boldsymbol{1})-\Abi$.

With $\Gbi$ defined, we could then find insight from spectral graph theory. In fact, to guarantee a simultaneous grouping of the tasks and features, it suffices to guarantee that the induced bipartite graph $\Gbi$ has $k$ connected components. Then the following theorem states that this is equivalent to restrict the dimension of the null space of $\LGbi$.

\begin{thm} \label{thm:graph}\cite{spectral}
	If $\dim(\mathcal{N}(\LGbi))=k$, i.e, the 0 eigenvalue of $\LGbi$ has a multiplicity of $k$ if and only if the bipartite graph $\Gbi$ has k connected components. Moreover, denote $\mathcal{G}(i)$ as the set of tasks and features belonging to the $i$-th component, we then have $\mathbb{EIG}_{\LGbi}(0) = span(\boldsymbol{\iota}_{\mathcal{G}(1)},\boldsymbol{\iota}_{\mathcal{G}(2)}, \cdots, \boldsymbol{\iota}_{\mathcal{G}(N)})$, where $\boldsymbol{\iota}_{\mathcal{G}(i)} \in \mathbb{R}^{(d+T) \times 1}$, $[\boldsymbol{\iota}_{\mathcal{G}(i)}]_j =1 \  \text{if} \ j \in \mathcal{G}(i) $, otherwise $[\boldsymbol{\iota}_{\mathcal{G}(i)}]_j =0 $.
\end{thm}

The theorem above implies a way to realize our goal: adopting a regularizer to force the smallest $k$ eigenvalues to  be zero. let $N= d + T$ denote the total number of nodes in $\Gbi$.  We have that the regularizer is equivalent to force $rank(\LGbi) =N-k$, which is known to be an NP-hard problem. In this case, we turn to  minimize the sum of the bottom $k$ eigenvalues $\sum_{i=1}^k \lam{i}$  as a tractable relaxation. According to the well-known extremal property of eigenvalues suggested by Fan \cite{fan}, the sum of the $k$  smallest eigenvalues of $\LGbi$ could be formulated as a minimization problem:
\begin{equation*}
\sum_{i=1}^k \lambda_i(\mathcal{L}_{\Gbi}) =  \min_{\boldsymbol{E}} tr(\boldsymbol{E}\mathcal{L}_{\Gbi}\boldsymbol{E}^\top), s.t. \ \boldsymbol{E}^\top\boldsymbol{E} = \boldsymbol{I}_{k}.
\end{equation*}
At first glance, the above problem is not convex due to the non-convex constraint $\boldsymbol{E}^\top\boldsymbol{E} = \boldsymbol{I}_{k}$. Recently a convex formulation of the problem starts drawing attention from machine learning and computer vision community \cite{lublock}. The nature behind this new formulation attributes to the following theorem.
\begin{thm}\label{thm:eig}
	Let  $\Gamma=\{\boldsymbol{U}: \boldsymbol{U} \in \mathbb{S}_{N},\ \boldsymbol{I} \succeq \boldsymbol{U}  \succeq \boldsymbol{0}, tr(\boldsymbol{U})=k \} $, then $\forall \boldsymbol{A} \in \mathbb{S}_N$:
	\begin{equation*}
	\sum_{i=1}^k\lambda_i(\boldsymbol{A}) =  \min_{\boldsymbol{U} \in \Gamma} \left<\boldsymbol{A}, \ \boldsymbol{U}\right>,
	\end{equation*}
	with an optimal value reached at $\boldsymbol{U} = \boldsymbol{V}_k\boldsymbol{V}^\top_k$, where $\boldsymbol{V}_k$ represents the eigenvectors of the smallest $k$ eigenvalues of $\boldsymbol{A}$.
\end{thm}
\begin{proof}
  
  In this proof, we denote the eigenvalue decomposition of $\boldsymbol{A}$ as
  \begin{equation}
  \boldsymbol{A} = \boldsymbol{Q} \Lambda \boldsymbol{Q}^\top,  \  \Lambda = diag(\lambda_1(\boldsymbol{A}), \cdots, \lambda_N(\boldsymbol{A})).
  \end{equation}
  For any element $\boldsymbol{U}$ in the feasible set $\Gamma$, we have:
  $
  \left<\boldsymbol{A}, \boldsymbol{U}\right> = \sum_{i}C_{ii}\lambda_i(\boldsymbol{A}),
  $
  where $\boldsymbol{C} = \boldsymbol{Q}^\top \boldsymbol{U} \boldsymbol{Q}$. Since $\boldsymbol{C}$ has the same eigenvalues as $\boldsymbol{U}$, we have $\boldsymbol{C} \in \Gamma$ if and only if $\boldsymbol{U} \in \Gamma$. 
  Then we have:
  \begin{equation}\label{equi}
  \min_{\boldsymbol{U} \in \Gamma}\left<\boldsymbol{A}, \boldsymbol{U}\right>  \Longleftrightarrow  \min_{\boldsymbol{C} \in \Gamma}\sum_{i}C_{ii}\lambda_i(\boldsymbol{A}).
  \end{equation}
  Define $\boldsymbol{e}^i \in \mathbb{R}^{N \times 1} $, $\boldsymbol{e}^i_i = 1 $ and $\boldsymbol{e}^i_s = 0$ , if $s \neq i$, then we reach the fact that:
  \[C_{ii} = \dfrac{\boldsymbol{e}^{i^\top}\boldsymbol{C}\boldsymbol{e}^{i}}{\boldsymbol{e}^{i^\top}\boldsymbol{e}^{i}}.\] We could then attain the following inequality based on the extremal property of the top/bottom eigenvalue of $\boldsymbol{C}$:
  \begin{equation}
  \begin{split}
  0 &\le \lambda_1(\boldsymbol{C}) = \min\limits_{\boldsymbol{x}} \dfrac{\boldsymbol{x}^{\top}\boldsymbol{C}\boldsymbol{x}}{\boldsymbol{x}^{\top}\boldsymbol{x}} \le C_{ii}\\ &\le  \max\limits_{\boldsymbol{x}} \dfrac{\boldsymbol{x}^{\top}\boldsymbol{C}\boldsymbol{x}}{\boldsymbol{x}^{\top}\boldsymbol{x}} = \lambda_N(\boldsymbol{C})  \le 1.
  \end{split}
  \end{equation}
  The minimum of (\ref{equi}) is reached at $\sum_{i=1}^k\lambda_i(\boldsymbol{A}) $ when $C_{ii} = 0, i > k $, $C_{ii} = 1, i \le k$.
  This directly shows that $\sum_{i=1}^k \lambda_i(\boldsymbol{A}) =  \min_{\boldsymbol{U} \in \Gamma} \left<\boldsymbol{A}, \ \boldsymbol{U}\right>$.
  
  Now it only remains to prove that $\boldsymbol{U} = \boldsymbol{V}_k \boldsymbol{V}^\top_k$ is an optimal solution by showing $\sum_{i=1}^k \lambda_i(\boldsymbol{A}) = \left< \bm{A}, \bm{U}\right> $. Since $\boldsymbol{V}_k$ is the eigenvectors associated with the smallest $k$ eigenvalues of $\boldsymbol{A}$, we have $\boldsymbol{Q} = [\boldsymbol{V}^\bot_k, \boldsymbol{V}_k]$, where $\boldsymbol{V}^\bot_k$ denotes the eigenvectors associated with the largest $N-k$ eigenvalues, and we have $\boldsymbol{V}^\top_k\boldsymbol{V}^\bot_k = \boldsymbol{0}$ and  $\boldsymbol{V}^{\bot^\top}_k\boldsymbol{V}_k = \boldsymbol{0}$. In this sense, we obtain:
  \begin{equation}
  \begin{split}
  \bm{C} &= \boldsymbol{Q}^\top \boldsymbol{U}\boldsymbol{Q} =\begin{bmatrix}
 \boldsymbol{V}^\top_k \\ 
    \boldsymbol{V}^{\bot^\top}_k \\
  \end{bmatrix} \boldsymbol{V}_k \boldsymbol{V}^\top_k [ \boldsymbol{V}_k, \boldsymbol{V}^{\bot}_k] \\ &= \begin{bmatrix}
  \boldsymbol{I}_k  & \boldsymbol{0}\\ 
  \boldsymbol{0} & \boldsymbol{0}
  \end{bmatrix}
  \end{split}.
  \end{equation}   
  Then the proof follows that $ \sum_{i} C_{ii} \lambda_i(\bm{A}) = \sum_{i=1}^k \lambda_i(\bm{A})$
  
\end{proof}

Combining the multi-task empirical loss $\mathcal{J}(\boldsymbol{W})$, the regularization term proposed above, and an $\ell_2$ penalty on $\bm{W}$, we reach  our proposed optimization problem $(\boldsymbol{P})$:  
\begin{equation*}
(\boldsymbol{P}) \ \min_{\boldsymbol{W},\boldsymbol{U}\in \Gamma}  \ \mathcal{J}(\boldsymbol{W}) + 
\alpha_1 \cdot \left<\LGbi,\boldsymbol{U}\right> + \frac{\alpha_2}{2} \cdot \norm{\bm{W}}_F^2. 
\end{equation*}

\section{Optimization}
\indent In this section, instead of solving $(\boldsymbol{P})$ directly, we first propose an optimization method to solve a surrogate problem $(\boldsymbol{P}^\star)$ written as 
\begin{equation*}
(\boldsymbol{P}^\star) \ \min_{\boldsymbol{W},\boldsymbol{U}\in \Gamma}  
\left\{\begin{aligned}
& \mathcal{J}(\boldsymbol{W}) + 
\alpha_1 \cdot \left<\LGbi,\boldsymbol{U}\right> + \frac{\alpha_2}{2} \cdot \norm{\bm{W}}_F^2\\
& +  \frac{\alpha_3}{2} \cdot \norm{\bm{U}}_F^2
\end{aligned}\right\}.
\end{equation*}

We will soon see that, under certain conditions, our algorithm, though originally targeted at $(\boldsymbol{P}^\star)$, surprisingly produces a  sequence that simultaneously convergences to a critical point of  $(\boldsymbol{P})$ and  $(\boldsymbol{P}^\star)$.

Since the term $\left<\LGbi,\boldsymbol{U}\right>$ is non-smooth and non-convex with respect to $\bm{W}$, we adopt a Proximal Gradient Decent (PGD)\cite{fista} framework in our algorithm. As a basic preliminary, we assume that the gradient of the loss function i.e. $\nabla_{\boldsymbol{W}} \mathcal{J}(\boldsymbol{W})$ is $\varrho$-Lipschitz continuous. Following the PGD framework, for each iteration step $t$, given a constant $C > \varrho$ and the historical solution $\boldsymbol{W}^{t-1}$, the parameter $\boldsymbol{W}^t$ and $\boldsymbol{U}^t$ could be updated from the following problem: 
\begin{equation*}
(\boldsymbol{Prox}) \ \min_{\boldsymbol{W},\boldsymbol{U} \in \Gamma}
\left\{
\begin{aligned}
  &\dfrac{1}{2} \left\norm{\boldsymbol{W} -\widetilde{\boldsymbol{W}}^t \right}_F^2 + \frac{\alpha_1}{C} \left<\LGbi,\boldsymbol{U}\right> \label{Ptheta}\\ 
&+   \frac{\alpha_2}{2C} \cdot \norm{\bm{W}}_F^2 + \frac{\alpha_3}{2C}\norm{\bm{U}}_F^2 \\
\end{aligned}
 \right\},
\end{equation*}	
where 
$\widetilde{\boldsymbol{W}}^t = \boldsymbol{W}^{t-1} - \dfrac{1}{C}\nabla_{\boldsymbol{W}}\mathcal{J}(\boldsymbol{W}^{t-1})$.
\subsection{Subroutines}
Solving $\boldsymbol{(Prox)}$ involves two subroutines, one is to optimize $\boldsymbol{U}$ with  $\boldsymbol{W}$  given, and the other is to optimize $\boldsymbol{W}$ with $\boldsymbol{U}$ given.

{\color{white}dsadsa}\\
\noindent \textbf{Update $\boldsymbol{U}$, fix $\boldsymbol{W}$}: This involves the following  subproblem:
\begin{equation}\label{eq:usub}
\min_{\bm{U}} \inner + \frac{\alpha_3}{2C}\norm{\bm{U}}_F^2, \ \ s.t. \  \ \bm{U} \in \Gamma
\end{equation}
 Unfortunately, Thm. \ref{thm:eig} only gives a solution to this problem when $\alpha_3 =  0$. This degenerates to an ordinal truncated eigenvalue minimization problem which has been widely adopted by historical literatures \cite{yang2019split,liu2019robust,xie2017implicit,lublock}. In the forthcoming theorem, inspired by \cite{sdpopt,stopca}, we show that, with a moderate magnitude of $\alpha_3$, the subproblem still has a closed-form solution. More interestingly,  we show that this is also a solution for $\alpha_3 = 0$, which is illustrated as Fig.\ref{fig:eigillu}.
\begin{thm}\label{eigsol}
Let $\lam{0} = 0, \lam{N+1} = +\infty$.  Let $\bm{V} = [\bm{v}_1, \bm{v}_2, \cdots, \bm{v}_N]$ be the associated eigenvectors  for $\lam{1}, \cdots, \lam{N}$. Furthermore, set 
\begin{align*}
p &= \max\{i: \lam{i} < \lam{i+1}, 0 \le i < k\}\\
q &=\min\{i: \lam{i} < \lam{i+1}, i \ge k \}.\\
\Delta{p} &= \lam{p+1} - \lam{p}, \\
\Delta{q} &= \lam{q+1} - \lam{q}.\\
\breve{\delta}(\LGbi)  &= 
\min\{\Delta{p}, \Delta{q}\} \\
\end{align*}
For all $\LGbi \neq \bm{0}$  and  $0 \le \frac{\alpha_3}{2C} < \breve{\delta}(\LGbi)$, the optimal solution of (\ref{eq:usub}) is:\\
\begin{equation}\label{eq:opt}
\begin{array}{lll}
\bm{U}^\star= \bm{V}\tilde{\Lambda}\bm{V}^\top, & \tilde{\Lambda} = diag(\bm{c}), & {c}_i = \begin{cases}
1 & i\le p , \\
\frac{k-p}{q-p} & q \ge i >p, \\
0 & \text{otherwise}.
\end{cases}
\end{array}
\end{equation}
\end{thm} 

\begin{figure}[h]
\centering
      \includegraphics[width=0.9\columnwidth]{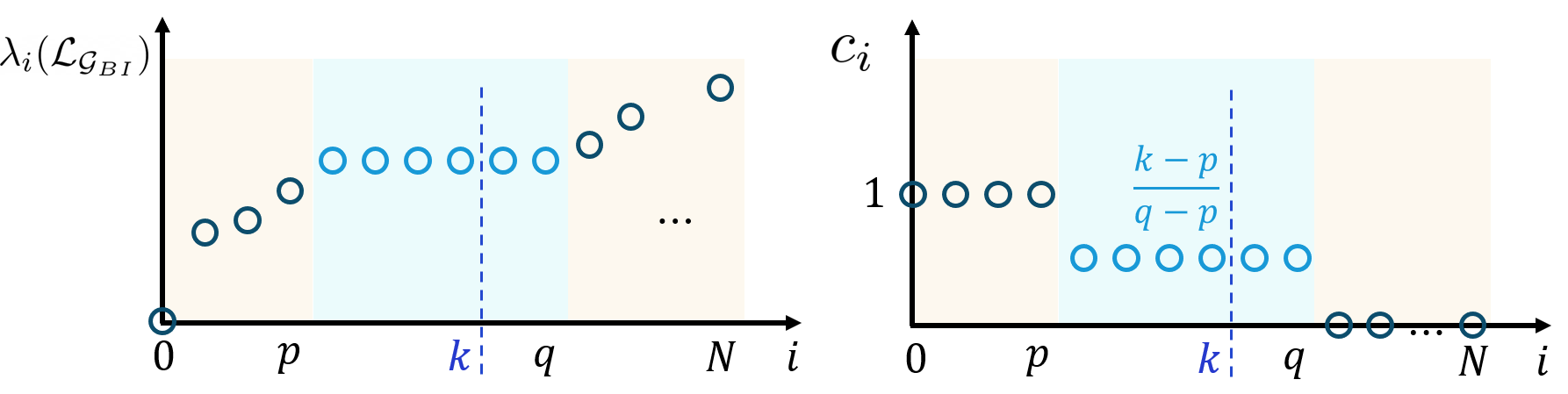}    
  \caption{\label{fig:eigillu} \textbf{Illustration of the Solution in Thm. \ref{eigsol}.} In this  figure, we plot the values of $l_i$ with respect to the  corresponding eigenvalues.  We see that Thm. \ref{eigsol} considers the multiplicity of $\lam{k}$. This makes our algorithm stable even when the eigengap $\lam{k+1} - \lam{k}$ is zero.    }

\end{figure}
\noindent We have three interesting remarks for this theorem.
\begin{rem}[Grouping Effect under an ideal condition] We now provide a remark for the grouping power of $\bm{V}$. Under an ideal case, we assume that the bipartite graph has exactly $k$ connected components. Since $c_i = 0 , \forall i > k$, only $\bm{V}_{1:k}$ is relevant to the computation of $\bm{U}^\star$. We then investigate the grouping power of $\bm{f}_i \in \mathbb{R}^{k}, i =1,2,\cdots, N$, which is denoted as the transpose of the $i$-th row of $\bm{V}_{1:k}$. We define $\mathcal{G}(1) ,\cdots, \mathcal{G}(k)$, and the corresponding nodes in each component as $n_{\mathcal{G}(1)}, \cdots, n_{\mathcal{G}(k)}$, respectively. According to Thm. \ref{thm:graph}, up to some orthogonal transformation, $\boldsymbol{f}_i \in \mathbb{R}^{k \times 1}$ becomes: 
	\begin{equation}
	{f}_{i,j} =\begin{cases}
	\dfrac{1}{\sqrt{n_{\mathcal{G}(j)}}}, &i \in \mathcal{G}(j) \\
	0, &otherwise
	\end{cases}.
	\end{equation}
	In this way, we see that $\boldsymbol{f}_i$ has a strong discriminative power indicating which group the underlying task/feature belongs to. In Sec.\ref{sec:group}, we will revisit this property with a detailed theoretical analysis with more practical considerations.
\end{rem}

\begin{rem} 
Different from the original result in Thm. \ref{thm:eig} that only considers the case when $\alpha_3 = 0$, Thm. \ref{eigsol} allows  $\alpha_3 > 0$ which provides strong convexity to the $\bm{U} $-subproblem.  This  makes global convergence property available  to the algorithm. More interestingly, we could also prove that the final algorithm converges globally to the critical points for both $(\bm{P}^\star)$ itself and the original problem $(\bm{P})$. The readers will soon see this in the next subsection.
\end{rem}
\begin{rem} \label{rem:ident}
Here we define $\bm{V}_{a:b}$ as the eigenvectors associated with $\lam{a}, \cdots, \lam{b}$. As shown in Fig.\ref{fig:eigillu}, our algorithm can work even when the eigengap $\lam{k+1} -\lam{k}$ vanishes. Note that, whenever $\lam{k}  = \lam{k+1}$, the solution of Thm. \ref{thm:eig} is not well-defined. In this case, $ \lam{k} $ must have a multiplicity greater than 1. Without loss of generality,  we assume that  $\LGbi$ has $s$ distinct eigenvalues $ [\lamc{1}, \cdots \lamc{s}] $  in its first $k$ smallest eigenvalues $[\lam{1}, \cdots \lam{k}]$, where $1 \le s < k$.  In fact,  $\bm{V}_{1:k}$ could not span the whole subspace $\oplus_{i=1}^s\mathbb{EIG}_{\LGbi}(\breve{\lambda}_s)$ in this case (it only contains $k$ out of $q$ bases of this subspace).  In this sense, the solution is not identifiable since  $\bm{V}_{1:k}\bm{V}_{1:k}^\top$ is not unique toward changes (either through  rotations or through different ways to select $k$ bases out of the $q$ bases) of the eigenvectors. This means that  we can observe completely different results when the subset of eigenvectors  is chosen differently. As for a practical example, we construct a bipartite graph with an affinity matrix:  
\[
A = \begin{bmatrix}
\bm{0}& \bm{W}\\
\bm{W}^\top& \bm{0}\\
\end{bmatrix}, ~ ~ \text{and} ~ ~
\bm{W} = \begin{bmatrix}
1 & 1 & & & &\\
1& 1 &  & & &\\
 &  &2 &2 & &\\
&  &2 & 2& &\\
&  && & 3&3 \\
\end{bmatrix}.
\]
Obviously, the multiplicity of zero eigenvalue for the corresponding Graph Laplacian matrix is 3. We denote the eigenvectors as $\bm{v}_1,\bm{v}_2, \bm{v}_3$. Let $k =2$, next, we now show that the outer product $\bm{V}\bm{V}^\top$ is not unique. To do this, picking $\bm{V}_{1:2} = [\bm{v}_1,\bm{v}_2]$ and  $\bm{V}_{2:3} = [\bm{v}_2,\bm{v}_3]$, we calculate the corresponding outer products $\bm{V}_{1:2}\bm{V}_{1:2}^\top$ and $\bm{V}_{2:3}\bm{V}_{2:3}^\top$ and visualize them in Fig.\ref{fig:subspace}. We see that the matrices are completely different, making the preceding subproblem ill-defined since it leads to completely different solutions.\\ 
It is interesting to remark here that this issue will not take place if we employ Thm.\ref{eigsol} instead. In this case, we have $\bm{U}^\star = \bm{V}_{1:p-1}\bm{V}_{1:p-1}^\top + \frac{k-p}{q-p}\bm{V}_{p:q}\bm{V}_{p:q}^\top$. Moreover, by the definition of q and p, we know that $\bm{V}_{p:q}$ spans $\mathbb{E}_1= \mathbb{EIG}_{\LGbi}(\lamc{s})$ and obviously $\bm{V}_{1:p-1}$ spans $\mathbb{E}_2 = \bigoplus_{i=1}^{s-1} \mathbb{EIG}_{\LGbi}(\lamc{i})$. This implies that $\bm{V}_{1:p-1}\bm{V}_{1:p-1}^\top$ forms the orthogonal projector onto $\mathbb{E}_1$ and $\bm{V}_{p:q}\bm{V}_{p:q}^\top$ forms the orthogonal projector onto $\mathbb{E}_2$. According to the basic properties of orthogonal projectors, we know $\bm{U}^\star$ is well-defined and invariant. To sum up, the advantage of Thm.\ref{eigsol} against traditional variational formulations of  $\sum_{i=1}^k \lam{i}$ is shown in Tab.\ref{tab:eigform}. Note that all three formulations therein yield the same optimal value. However, they have different optimal solutions with different degrees of stability.
\end{rem}

\begin{figure}[h]
     \centering
     \subfigure[$\bm{V}_{1:2}\bm{V}_{1:2}^\top$]{
      \includegraphics[width=0.45\columnwidth]{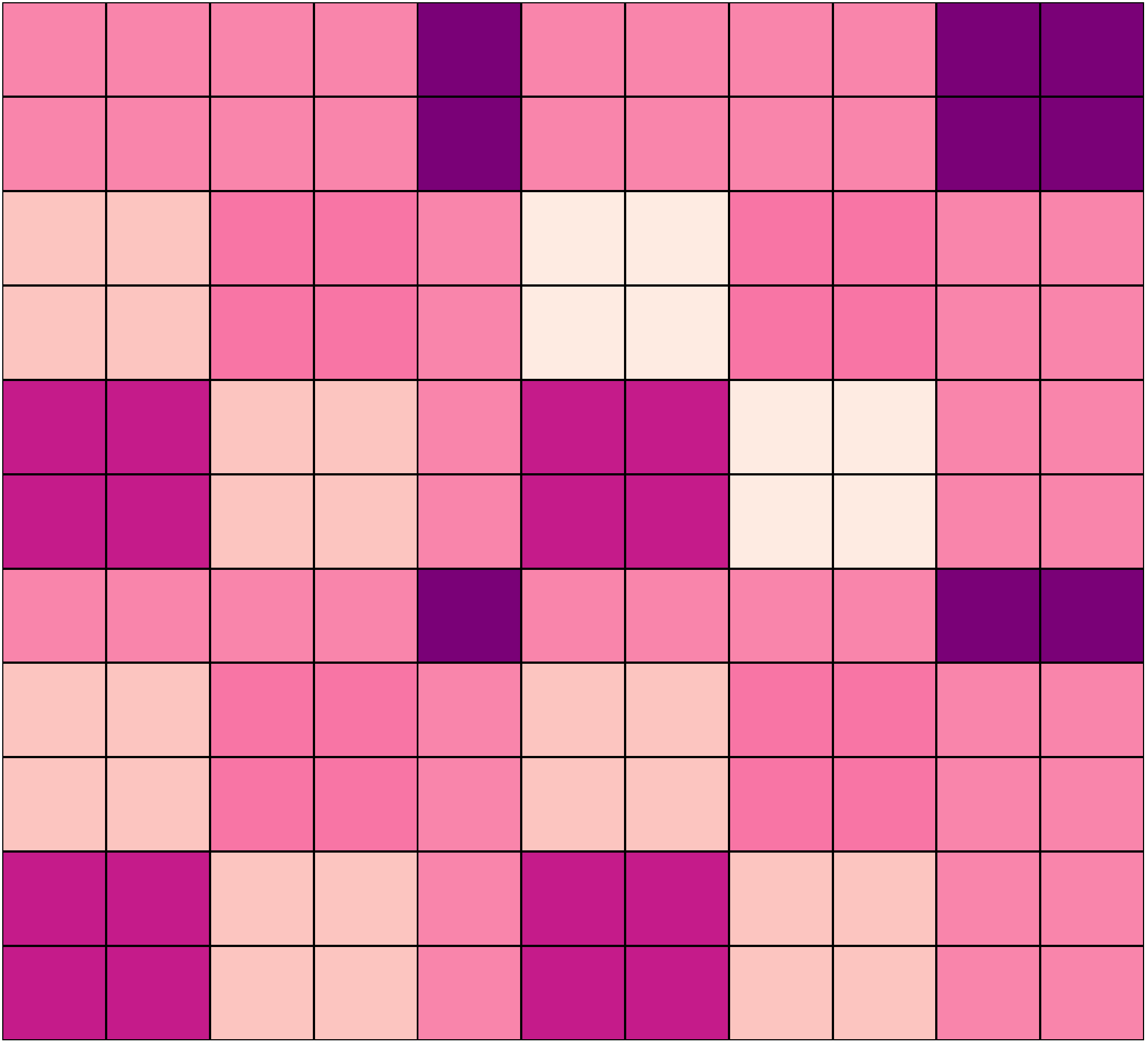}  
     }
     \subfigure[$\bm{V}_{2:3}\bm{V}_{2:3}^\top$]{
      \includegraphics[width=0.45\columnwidth]{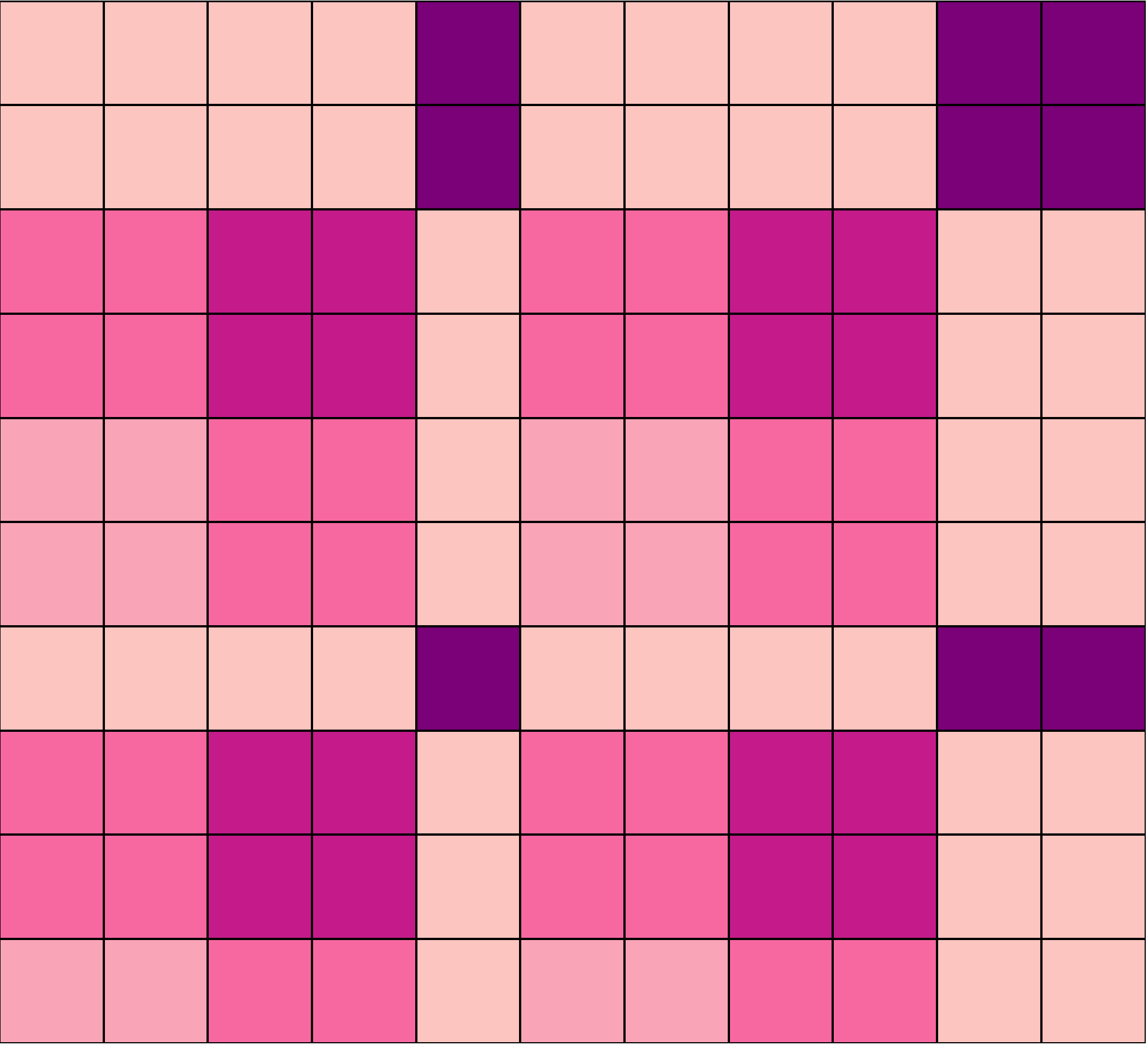}  
     } 
  \caption{Visualizations of the eigenvector outer-products, which shows that $\bm{V}\bm{V}^\top$ is not identifiable when we need to pick 2 out of 3 bases from the eigenspace of zero.}
  \label{fig:subspace}
\end{figure}

{

\begin{table}[htbp]
  \centering
  \caption{\label{tab:eigform} Different formulations of $\sum_{i=1}^k \lam{i}$, where Ident. represents the identifiability of $\bm{U} = \bm{V}\bm{V}^\top$ when the eigengap $\lam{k+1} - \lam{k}$ vanishes. }
  \scriptsize
    \begin{tabular}{lccc}
    \toprule
   &  Convex & 
    $
    \begin{array}{l}
    \text{Strongly} \\
    \text{Convex}
    \end{array}$

    &  
Ident. \\

\midrule
$\begin{array}{l}
      \min_{\bm{V}} \ \  tr(\bm{V} \LGbi \bm{V}^\top)\\
     s.t.  \  \ \bm{V}\bm{V}^\top = \bm{I}_k
    \end{array}    $

    & \ding{53}  & \ding{53} & \ding{53} \\
\midrule

$\begin{array}{l}
 \min_{\bm{U}}   \ \  \inner  \\
 s.t.  \ \  \bm{U} \in \Gamma 
\end{array}$

 & \ding{52} & \ding{53} & \ding{53} \\
\midrule
$
\begin{array}{l}
 \min_{\bm{U}}  \  \ \inner + \lambda \cdot \norm{\bm{U}}_F^2\\ 
  s.t.  \  \ \bm{U} \in \Gamma \\
 0 \le \lambda \le \breve{\delta}(\LGbi) \\
 (Ours)
\end{array} $  &  

 \ding{52} & \ding{52} & \ding{52} \\
\bottomrule
    \end{tabular}%
\end{table}%

}
\noindent Now we proceed to solve the next subproblem.\\
\noindent \textbf{Update $\boldsymbol{W}$, fix $\boldsymbol{U}$}: The following proposition shows that when  $\boldsymbol{U}$ is fixed, one could cast the $\boldsymbol{W}$ subproblem as a specific elastic net proximal mapping problem:

\begin{prop}\label{prop:sol}
	The optimal solution of $\boldsymbol{W}$ subproblem of $(\boldsymbol{Prox})$ is: 
	\begin{equation}\label{eq:wsol}
\bm{W}^\star = sgn(\bm{\widetilde{W}})\left(\left|\dfrac{\bm{\widetilde{W}}}{1+\frac{\alpha_2}{C}}\right| - \frac{\alpha_1}{C+\alpha_2} \bm{D}\right)_+,
	\end{equation} where
	${D}_{ij} = \norm{\boldsymbol{f}_i - \boldsymbol{f}_{d+j} }^2 $.
\end{prop}

\begin{proof}
  With the fact that 
  \begin{equation}
  \begin{split}
  &\inner \\
  &= \left<diag(\begin{bmatrix} 0 &|\boldsymbol{W}| \\ |\boldsymbol{W}|^\top &0
  \end{bmatrix}\boldsymbol{1})-\begin{bmatrix} 0 &|\boldsymbol{W}| \\
  |\boldsymbol{W}|^\top &0
  \end{bmatrix},\boldsymbol{U}\right> \\
  & = \left<diag(\boldsymbol{U})\boldsymbol{1}^\top -\boldsymbol{U},\begin{bmatrix} 0 &|\boldsymbol{W}| \\
  |\boldsymbol{W}|^\top &0
  \end{bmatrix} \right>,
  \end{split}
  \end{equation}
 we could reformulate the problem as:
  \begin{equation*}
  \min_{\boldsymbol{W}}  \left\{\begin{aligned}
& \ \frac{1}{2}||\boldsymbol{W} -\widetilde{\boldsymbol{W}}||_F^2+   \frac{\alpha_1}{C} \cdot \left<\Delta^{(1)} + \Delta^{{(2)}^\top}, |\boldsymbol{W}|\right> \\ &+\frac{\alpha_2}{2C}\cdot \norm{\bm{W}}_F^2\\
  \end{aligned}\right\},
  \end{equation*}
  where
\begin{align}
&\Delta= diag(\boldsymbol{U})\boldsymbol{1}^\top-\boldsymbol{U},\\
&\Delta^{(1)}= \Delta(1:d,(d+1):end),\\
&\Delta^{(2)}= \Delta((d+1):end,1:d).
\end{align} 
\noindent Furthermore, we have
  \[ \Delta^{(1)}_{ij} + \Delta^{(2)}_{ji} =U_{ii} + U_{d+j,d+j} -U_{i, d+j}- U_{d+j, i} = \norm{\boldsymbol{f}_i - \boldsymbol{f}_{d+j}}_2^2. \] 
  
\noindent Since the objective function is $\left(1+ \frac{\alpha_2}{C}\right)-$strongly convex, it is easy to see that the optimal solution is unique. Then the proof follows the proximal mapping of the $\ell_1$ norm \cite{fista}.
\end{proof}
With the embedding vectors fixed, the algorithm turns to learn $\boldsymbol{W}$ with a sparsity-inducing strategy, where ${W}_{ij}$ is activated if the magnitude of $\widetilde{W}_{ij}$ dominates the embedding distance between feature $i$ and task $j$. Moreover, the following remark reveals how transfer takes place across features and tasks.
\begin{rem}
	(\ref{eq:usub}) enjoys an alternative formulation in the following:
	\begin{equation}\label{eq:alter}
	\begin{split}
	\argmin_{\boldsymbol{W}} \left< \bm{D}, |\boldsymbol{W}|\right>
	\ s.t. \boldsymbol{W} \in \mathcal{B}_{c(\alpha)}(\boldsymbol{W},\widetilde{\boldsymbol{W
	}}^t),
	\end{split}
	\end{equation}
	where $\mathcal{B}_{c(\alpha)} = \left\{\boldsymbol{W}: \norm{\boldsymbol{W} - \widetilde{\boldsymbol{W}}^t}_F^2 \le c(\alpha_1), \norm{\bm{W}}_F^2 \le c(\alpha_2
	)  \right\}$. It is noteworthy that (\ref{eq:alter}) shares a striking resemblance with the discrete optimal transport problem seeking the smallest cost transporting information across two sets: features and tasks. Borrowing insights from the optimal transport problem \cite{optimal}, the transfer costs between feature $i$ and task $j$ are measured via the $\ell_2$ distance between their embeddings ${D}_{ij}$. Since tasks/features belonging to the same group tend to share very similar embeddings, the intra-group transfer is encouraged via a small transportation cost ${D}_{ij}$. On the contrary, negative transfer across different groups is penalized with a much larger transportation cost ${D}_{ij}$. Different from existing MTL studies, this shows that our method also models negative transfer issue from the perspective of task-feature transfer.  
\end{rem}
\subsection{Theoretical Analysis}
\subsubsection{Convergence Analysis}\label{sec:conv}
With the subroutines clarified, we now turn to discuss the optimization method in a global view. To reach a critical point of $(\boldsymbol{Prox})$, we have to alternatively optimize $\boldsymbol{U}$ and $\boldsymbol{W}$ until convergence before changing the reference point $\widetilde{\bm{W}}^{t}$. This induces a bi-level looping: \textit{an outer loop} is responsible for changing the reference point and \textit{an inner loop}  is responsible for solving $\boldsymbol{W}^k$ and $\boldsymbol{U}^k$ given the reference point, which significantly increases the computational burden. However, we find practically that one round of the  inner loop is sufficient to leverage convergence property. We summarize this method in Alg.~\ref{alg:opt}. Now we can prove the global convergence property for Alg.~\ref{alg:opt}. And more importantly, this property holds for both the surrogate loss and the original loss.

\begin{thm}[\textbf{Global Convergence Property for Alg.~\ref{alg:opt} with respect to } $(\bm{P}^\star)$]\label{thm:glob_star}
	\ Let $\{\bm{W}^t, \bm{U}^t\}$ be the sequence generated by Alg. \ref{alg:opt}. Furthermore, assume that $\mathcal{J}(\cdot)$ is a definable function with $\mathcal{J}(\bm{W})$ lower bounded away from $\bm{0}$, and  with $\rho$-Lipschitz continuous gradient. Then pick $C > \rho$, and $0 < \alpha_3 < 2C\min_t \breve{\delta}(\LGbi^t)$, for all finite and feasible  initialization, the following facts hold:
	\begin{enumerate}[itemindent=0pt, leftmargin =15pt,label={(\arabic*)}]

		\item The parameter sequence $\{\boldsymbol{W}^t, \boldsymbol{U}^t\}_t$ converges to a critical point $(\bm{W}^*, \bm{U}^*)$ of the problem $(\bm{P}^\star)$.
		\item The loss sequence converges to the loss of critical point $(\bm{W}^*,\bm{U}^*)$ of the problem $(\bm{P}^\star)$.
		\item For all $t \in \mathbb{N}$, there exists a subgradient  $\bm{g}_t$, such that when $T \rightarrow +\infty$, $\dfrac{1}{T} ({\sum\limits_{t=1}^T\norm{\bm{g}_t}^2}) \rightarrow 0$ with rate $\mathcal{O}(\frac{1}{T})$.
	\end{enumerate}
\end{thm}

\begin{thm}(\textbf{Global Convergence Property for Alg.~\ref{alg:opt} with respect to $\bm{P}$}).\label{thm:global}
	\ Under the same condition as Thm. \ref{thm:glob_star}, the  sequence $\{\bm{W}^t, \bm{U}^t\}_t$ generated by Alg.~\ref{alg:opt} also satisfies  (1)-(3)  with respect to the original problem $(\bm{P})$.
\end{thm}

\begin{rem}
With the help of Thm. \ref{eigsol}, we can reach the global convergence property for the original problem in Thm. \ref{thm:global}. Unfortunately, this will not always hold if we adopt Thm. \ref{thm:eig} directly. One reason is that it is hard to guarantee the identifiability discussed in Rem.\ref{rem:ident} of the $\bm{U}^t$ sequence even if $\bm{W}^t$ converges to a critical point. Another reason is that, without Thm. \ref{eigsol} it is hard to meet the sufficient descent condition which is required in the global convergence property (see our appendix). By contrast, our algorithm, though developed to solve $(\bm{P}^\star)$ ,  could also convergence to a critical point of $(\bm{P})$. This shows that optimizing the surrogate loss does not affect the quality of the solution. Moreover, since the choice of  $\alpha_3$ is irrelevant to the algorithm as long as it lies in $[0, 2C\min_t\breve{\delta}(\LGbi^t) )$, we do not need to tune this hyperparameter explicitly. 
\end{rem}

\begin{algorithm}[t]
  \caption{TFCL for $(\bm{P})$}
  \label{alg:opt}
  \begin{algorithmic}
  \STATE {\bfseries Input:} Dataset $\mathcal{S}$, $\alpha_1$, $\alpha_2$ $k$, $C (C > \varrho)$.
\STATE {\bfseries Output:} Solution $\boldsymbol{W}$, $\boldsymbol{U}$.
\STATE Initialize $\boldsymbol{W}^0$, $\boldsymbol{U}^0 \in \Gamma$, $t=1$.
\REPEAT 
\STATE \textbf{\texttt{U SUBROUTINE}}:\\

\ \ \STATE \ \  {Calculate} $\LGbi$ with $\bm{W}^{t-1}$.\\
{\color{white}dsadsa}\\
\STATE  \ \ $\bm{V}^{t} \gets $ {the eigenvectors of} $\LGbi$.\\

\ \ \STATE \ \  $\bm{U}^{t} \gets \bm{V}^{t}\widetilde{\Lambda}\bm{V}^{t}{^\top},$ {according to Thm. \ref{eigsol}. }\\ 
{\color{white}dsadsa}\\
\STATE \textbf{\texttt{W SUBROUTINE}}:\\

\ \ \STATE \ \ {Calculate} $\bm{R}^t$ with $\bm{U}^{t}$.\\

\ \ \STATE \ \ $\widetilde{\boldsymbol{W}}^t \gets  \boldsymbol{W}^{t-1} - (1/C)\cdot \nabla_{\boldsymbol{W}}\mathcal{J}(\boldsymbol{W}^{t-1})$. \\

\ \ \STATE \ \  $\bm{W}^{t} \gets sgn(\bm{\widetilde{W}}^{t})\left(\left|\dfrac{\bm{\widetilde{W}}^{t}}{1+\frac{\alpha_2}{C}}\right| - \frac{\alpha_1}{C+\alpha_2} \bm{D}^t\right)_+$.\\

\STATE  $t \gets t + 1$ .\\

\UNTIL{Convergence}

\STATE $\boldsymbol{W} = \boldsymbol{W}^{t-1}$, $\boldsymbol{U} = \boldsymbol{U}^{t-1}$.
  \end{algorithmic}
\end{algorithm}

\subsubsection{Task-Feature Grouping Effect} \label{sec:group}
In this subsection we show how the proposed algorithm differentiates  task-feature groups in the model weights. Specifically, we have the following theorem.\\
\indent Under an ideal case, if $\sum_{i=1}^k \lam{i} = 0$, then according to Thm.\ref{thm:graph}, we know that $\Gbi$ must be k-connected and $\bm{W}_{i,j} \neq 0$ if and only if feature $i$ and task $j$ are in the same component of this bipartite group.\\
\indent More practically, we often observe $\sum_{i=1}^k \lam{i} \neq 0$, which makes the arguments above unavailable. Instead of assuming $\sum_{i=1}^k \lam{i} = 0$, for a well-trained model, it is reasonable to assume that the final objective function is small at the very end of the algorithm. Motivated by this, the following theorem shows that we can still recover the grouping structure when this much weaker assumption holds.

\begin{thm}(\textbf{Grouping Effect})\label{thm:group}
Assume that Alg.~\ref{alg:opt} terminates at the $\mathcal{T}$-th iteration with $\mathcal{F}(\bm{W}^{\mathcal{T}-1}, \bm{U}^{\mathcal{T}-1}) \le \epsilon_{\mathcal{T}-1}$. Denote $\mathsf{Supp}(\bm{A}) = \left\{(i,j): A_{i,j} \neq 0 \right\}$, $\mathcal{H}_K = \left\{\bm{W}: \norm{\bm{W}}_F \le K \right\},~ C_0 = \left(\dfrac{2}{\alpha_2} \cdot \epsilon_{\mathcal{T}-1}  \right)^{1/2}.$  We further assume that for all $\infty> \kappa >0$, $\sup_{\norm{\bm{W}}_F \le \kappa}\norm{\nabla_{\boldsymbol{W}}\mathcal{J}(\boldsymbol{W})}_\infty \le \varpi(\kappa) < \infty$ and that there is a matrix $\bm{W}^\star \in \mathcal{H}_{C_0}$, where the corresponding bipartite graph $\mathcal{G}^\star$ has  $k$ connected components with a Graph Laplacian matrix $\LGbi^\mathcal{T}$ giving the ground-truth grouping. Moreover, denote $n_i$ as the number of nodes in the $i$-th group of the graph, and $n^\uparrow_1 = \max_{i} n_i$, $n^\uparrow_2 = \max_{j, n_j \le n^\uparrow_1} n_i$. With the following notations:
\begin{equation*}
\begin{split}
&\kappa_0 =  C_0 + \dfrac{\varpi(C_0)}{C},~ \delta_1 = \frac{C}{\alpha_1} \kappa_0,~ \delta_2 = \frac{C}{\alpha_1} \delta_0,~\beta =  \dfrac{1}{n^\uparrow_1} +  \dfrac{1}{n^\uparrow_2},\\
& \rho = \frac{C_0}{\lambda_{k+1}\left(\LGbi^\mathcal{T}\right)}, ~ \xi=  \rho \cdot(\sqrt{d+T} + \sqrt{2}),  
\end{split}
\end{equation*}
we have:
	\begin{enumerate}
	\item[(a)] \textbf{(no-false-positive-grouping)} If  $\lambda_{k+1}(\LGbi^\mathcal{T}) > \lambda_{k}(\LGbi^\mathcal{T}) > 0$, $\frac{\sqrt{2}}{32} \cdot \beta > \xi $, and $ 8\sqrt{2}\xi < \delta_1 < \beta- 8\sqrt{2} \xi $, we have:

	\begin{equation*}\label{eq:g1}
   \mathsf{Supp}(\bm{W}^\mathcal{T}) \subseteq \big\{(i,j): \mathcal{G}(i) = \mathcal{G}(j) \big\} =  \mathsf{Supp}(\bm{W}^\star) ,
	\end{equation*}
	where  $\mathcal{G}(i)$ is the corresponding connected component in the bipartite graph $\mathcal{G}^\star$ that i belongs to.
    \item[(b)] \textbf{(correct-grouping) }If we further assume that $\min_{(i,j)}|\widetilde{\boldsymbol{W}}^{\mathcal{T}}_{i,j}| \ge \delta_{0}> 0$, \[ 8\sqrt{2}\xi < \min \left\{\delta_1,\delta_2 \right\} \le   \max \left\{\delta_1,\delta_2 \right\}  < \beta - 8\sqrt{2} \xi, \] we get that:

	\begin{equation*}\label{eq:g2}
   \mathsf{Supp}(\bm{W}^\mathcal{T}) =  \mathsf{Supp}(\bm{W}^\star).
	\end{equation*}

	\end{enumerate}
\end{thm}

\begin{rem} We have the following remarks for the theorem:
\begin{enumerate}
\item[(a)] The assumption that one can find a $\bm{W}^\star \in \mathcal{H}_{C_0}$ consistent with the GT structure is always achievable, since if  $\bm{W}^\star \notin \mathcal{H}_{C_0}$, one can pick $\bm{W}' = C_0 \cdot \frac{\bm{W}^\star}{||\bm{W}^\star||_F}$ instead that locates within the F-norm ball with the same support set.
\item[(b)] Thm. \ref{thm:group} states that if the $k$-th eigengap of $\LGbi$ exists, the hyperparameters are chosen as $\alpha_2 = o\left(\frac{(d+T)\cdot \epsilon_{\mathcal{T}-1}}{\beta^2 \cdot \lambda_{k+1}(\LGbi^\mathcal{T})^2}\right), \alpha_1 = \mathcal{O}(C\kappa_0)$ and the inequality of $\xi,\delta_1$ is ensured, we can reach the no-false-positive grouping, where $|W^\mathcal{T}_{ij}|$ is activated as nonzero only if feature $i$ and task $j$ belong to the same group in GT. Moreover, when we have extra assumptions on the intermediate variable $|\widetilde{W}^\mathcal{T}_{ij}|$, with $\alpha_1 = \mathcal{O}(C\cdot(\delta_0 \vee \kappa_0)), ~ \alpha_2 = o\left(\frac{(d+T)\cdot \epsilon_{\mathcal{T}-1}}{\beta^2 \cdot \lambda_{k+1}(\LGbi^\mathcal{T})^2}\right)$ , we can hopefully reach a correct grouping, where $|W^\mathcal{T}_{ij}|$ is activated as nonzero if and only if feature $i$ and task $j$ belong to the same group in GT.
\end{enumerate}
\end{rem}


\subsection{Discussion}\label{sec:disscus}
To end this section, we provide a discussion on the relationship between our base model and the work of \cite{lublock} and \cite{comtl}.
Similar to \cite{comtl}, we adopt the major assumption that features and tasks should be grouped into different clusters. However, our model differs significantly from this work. In \cite{comtl}, co-clustering is realized by the k-means assumption, without an explicit guarantee for leveraging the block-diagonal structure. Inspired by \cite{lublock}, we provide an explicit regularizer from the spectral graph theory for the clustering problem. In our work, we generalize the regularization in \cite{lublock} to a bipartite graph and apply it to the MTL problem. More importantly, we also provide a generalized closed-form solution for the variational form of truncated eigenvalue sum problem as shown in Thm.\ref{eigsol}.  As an important property, it finds a specific global solution of the original problem (which is not strongly convex) as the unique solution for the strongly convex problem \eqref{eq:usub}. This not only makes $\bm{V}_k\bm{V}_k^\top$ identifiable but also leads to the global convergence result which is missing in \cite{lublock} and \cite{comtl}. Last but not least, compared with \cite{lublock}, we also adopt a different optimization method. This could make the optimization procedure simpler. Moreover, it offers us a chance to figure out a close connection with the optimal transport problem (OT), which suggests that the $\bm{W}$ subproblem approximates the OT problem with the distances between spectral embeddings as the transportation cost. Moreover, this also leads to our analysis of the grouping effect with practical considerations as shown in Thm.\ref{thm:group} and Thm.\ref{thm:group_app}, which is also new compared with the existing studies.

\section{Personalized Attribute Prediction}\label{sec:per}
\indent So far, we have developed a novel multi-task learning method based on the desire of task-feature collaborative learning. In this section, we extend this method to a specific application problem which we call personalized attribute prediction. In this problem, we are given a set of personal annotations on visual attributes (e.g., \emph{smile} for human faces, \emph{comfortable} for shoes) for a variety of images, which are collected on the crowdsourcing platforms. Our goal here is to predict the user-specific annotations for unknown images, so that the results cater for personalized demands which often span a wide spectrum. This is a problem that greatly matches the multi-task learning scenario, since each user typically annotates only a limited amount of images. 
\subsection{Extended Model}\label{sec:app_model}
To model the personalized annotation process, we regard each user's annotation prediction as a single task. For a given attribute, we assume that there are $T$ users who take part in the annotation. Further, we assume that the $i$-th user labeled $n_i$ images with $n_{+,i}$ positive labels and $n_{-,i}$ negative labels. In this setting $\boldsymbol{X}^{(i)} \in \mathbb{R}^{n_{i} \times d}$ becomes the input features for images that the $i$-th user labeled, whereas $\boldsymbol{y^{(i)}} \in \{-1,1\}^{n_{i}}$ becomes the corresponding label vector. If $y^{(i)}_k =1 $, then the user thinks that the given attribute presents in the $s$-th image, otherwise we have $y^{(i)}_k =-1 $. Moreover, we denote $\mathcal{S}_{+,i}= \{k \ |  \ y^{(i)}_k = 1\}$  and $\mathcal{S}_{-,i}= \{k \ |  \ y^{(i)}_k = -1\}$. The diversity in personalized annotations allows us to employ different models for different users. In the spirit of this, we adopt a linear learner  $\boldsymbol{g}^{(i)}(\boldsymbol{x})= \boldsymbol{W}^{(i)^\top}\boldsymbol{x}$ for each task (user) $i$. 

A naive way to solve this problem is to learn user-specific models separately. However, adopting completely independent models might lead to disastrous over-fitting due to the limited amount of annotations from each user. To prevent this issue, we apply a coarse-to-fine decomposition for $\boldsymbol{W}$:
\begin{equation}
\boldsymbol{W}^{(i)} = \boldsymbol{\varTheta}_{c} +\boldsymbol{\varTheta}^{(i)}_g + \boldsymbol{\varTheta}^{(i)}_p.
\end{equation}
Here, the coarse-grained component $\boldsymbol{\varTheta}_{c}$ captures the consensus pattern shared across all users. This pattern typically consists of common sense and the superficial semantic information that are easy to be accepted by almost all the users. The finer-grained component $\boldsymbol{\varTheta}_{g}= [\boldsymbol{\varTheta}^{(1)}_{g}, \cdots, \boldsymbol{\varTheta}^{(T)}_{g}]$ captures the grouping pattern where our TFCL method works. Specifically, it interprets the majority of the diversity in the results, where different groups of users (tasks) tend to favor different results based on different priorities of the features. Considering the negative transfer issue, sharing information across dissimilar users and features clearly lead to over-fitting. Our block-diagonal regualarizer then come into play to restrict the structure of  $\boldsymbol{\varTheta}_{g}$ against negative transfer. The finest-grained component $\boldsymbol{\varTheta}_{p}= [\boldsymbol{\varTheta}^{(1)}_{p}, \cdots, \boldsymbol{\varTheta}^{(T)}_{p}]$ captures the 
personalized patterns that are completely user-specific. $\boldsymbol{\varTheta}_{p}$ is not available for all users. Rather, it is only activated for the hard tasks corresponding to the extremely personalized users and malicious users. In this way,  $\boldsymbol{\varTheta}_{p}$ offers us a chance to separate the abnormal tasks from the co-grouping factor, which keeps the model away from negative transfer from hard tasks.  To sum up, we have a concluding remark on how this decomposition scheme increases the flexibility of our base model.
\begin{rem}
In the extended model, there are two extra terms: $\boldsymbol{\varTheta}_{p}$ and $\boldsymbol{\varTheta}_{c}$. As a task-wise sparse parameter,  $\boldsymbol{\varTheta}_{p}$ serves as a detector for hard tasks. This is beneficial to suppress negative transfer. In fact, the negative effect might come from transfer knowledge from hard tasks (with poor performance) to easy tasks (with good performance), and having non-zero $\boldsymbol{\varTheta}_{p}$ helps to remove the hard task from grouping with easy tasks.  Meanwhile, $\boldsymbol{\varTheta}_{c}$ is a common factor that allows different groups to share overlapping features. With these  two extra components, we then reach a comprehensive model with sharing, grouping and the effect of hard task considered. 
\end{rem}
With the decomposition scheme, we then provide an objective function for the proposed model. To realize the functionality of the three components, we provide different regularization based on their characterizations. For the common factor $\boldsymbol{\varTheta}_c$, we simply adopt the most widely-used $\ell_2$ regularization. For $\boldsymbol{\varTheta}_{g}$, we employ our task-feature  collaborative learning framework. To do this, we reformulate $\Gbi$ with the users, features and $\boldsymbol{\varTheta}_{g}$. Moreover, the graph Laplacian is defined as $\bm{\varTheta}_{\mathcal{G}}$  which is obtained by replacing $\bm{W}$ in the original $\LGbi$ with $\g$.  For $\boldsymbol{\varTheta}_{p}$, we adopt the $\ell_{1,2}$-norm to induce column-wise (user-wise) sparsity. 
Finally, the empirical loss for task $i$ is denoted as $\ell_i$. As the standard preference learning paradigm, for a given user, we expect that the positive labeled instances could always have a higher rank than negative ones, so that the instances having the top predicted score always hit the user's comprehension about the attribute. This motivates us to optimize the Area Under roc Curve (AUC) metric in our model. Specifically, we adopt the squared surrogate loss for AUC \cite{onepass}: 
\begin{equation*}
\begin{split}
&\mathcal{J}(\c,\g,\p) = \sum_{i=1}^T \ell_i,\\
&\ell_i = \sum\limits_{x_p \in \mathcal{S}_{+,i}}\sum\limits_{x_q \in \mathcal{S}_{-,i}} \frac{s\Big(\boldsymbol{g}^{(i)}(\boldsymbol{x}_p) -\boldsymbol{g}^{(i)}(\boldsymbol{x}_q) \Big)}{n_{+,i}n_{-,i}}.\\
\end{split}
\end{equation*}
where $s(t) = (1-t)^2$. Note that the reasons for choosing the squared surrogate loss are two-fold. Theoretically, it is proved in the previous literature \cite{consis} that square loss results in a Bayesian optimal classier that is consistent with the true 0-1 AUC loss 
    \[\sum_{i}\sum_{x_p \in \mathcal{X_+}}\sum_{x_q \in \mathcal{X}_-} \frac{1}{n_+n_-} \cdot I\left[\bm{g}^{(i)}(\bm{x}_p) > \bm{g}^{(i)}(\bm{x}_q)\right],\] in an asymptotic sense. Practically, as discussed in the next subsection, we can easily accelerate the calculation of AUC loss. With all the above-mentioned settings, our objective function could be written in the form:
\begin{equation*}\label{genform}\small
\begin{split}
(\boldsymbol{Q}) \min_{\boldsymbol{\varTheta}, \boldsymbol{U} \in \Gamma} & ~ \left\{\begin{split}
&\mathcal{J}(\c,\g,\p)+\frac{\alpha_1}{2}\norm{\c}_2^2+  \alpha_2 \left<\bm{\varTheta}_{\mathcal{G}}, \boldsymbol{U}\right> \\ & +\frac{\alpha_3}{2} \norm{\g}_{F}^2 + \alpha_4 \norm{\p}_{1,2} +  \iota_{\Gamma}(U)
\end{split}\right\} 
\end{split}.
\end{equation*}



\subsection{Extended Optimization}
In this subsection, we will provide a fast extended algorithm to optimize $(\bm{Q})$. We first define the AUC comparison graph. Then we provide acceleration methods to speed-up loss and gradient evaluation. Finally, we provide an extended optimization method to solve problem $(\bm{Q})$.

\noindent \textbf{AUC comparison graph}. To begin with, we provide an AUC comparison graph to represent the sparse comparisons to calculate AUC. For each user $i$, the graph is defined as $\mathcal{G}_{AUC}^{(i)} = (\mathcal{V}^{(i)}, \mathcal{E}^{(i)},\mathcal{W}^{(i)})$. Here $\mathcal{V}^{(i)}$ denotes the set of vertices consist of all the samples that user $i$ labeled. Similarly, $\mathcal{E}^{(i)}$ represents the edge set $\{(j,k): y^{(i)}_j \neq  y^{(i)}_k  \}$. Moreover, for all edges $(j,k) \in \mathcal{E}^{(i)}$, we have a weight matrix $\mathcal{W}^{(i)} $  such that $\mathcal{W}^{(i)}_{j,k} = \frac{1}{n_{+,i}n_{-,i}}$.  Given $\mathcal{W}^{(i)}$,  the Laplacian matrix  $\mathcal{L}_{AUC}^{(i)}$ of $\mathcal{G}_{AUC}^{(i)}$ 
could be expressed as: $\mathcal{L}_{AUC}^{(i)} = diag(\mathcal{W}^{(i)}\boldsymbol{1}) - \mathcal{W}^{(i)}.$\\
\textbf{Loss Evaluation}. With the definition of  $\mathcal{L}_{AUC}^{(i)}$, we could reformulate the empirical loss $\mathcal{J}$ as:
\begin{equation*}
\begin{split}
\ell_i& = \sum\limits_{x_p \in \mathcal{S}_{+,i}}\sum\limits_{x_q \in \mathcal{S}_{-,i}} \frac{s\Big(\boldsymbol{g}^{(i)}(\boldsymbol{x}_p) -\boldsymbol{g}^{(i)}(\boldsymbol{x}_q) \Big)}{n_{+,i}n_{-,i}}\\
&=\frac{1}{2} (\boldsymbol{\widetilde{y}}^{(i)}-\hat{\boldsymbol{y}}^{(i)})^\top \mathcal{L}_{AUC}^{(i)}(\boldsymbol{\widetilde{y}}^{(i)}-\hat{\boldsymbol{y}}^{(i)})
\end{split}
\end{equation*}
\begin{equation*}
\mathcal{J}(\c,\g,\p)  =  \sum_{i=1}^{n_u} \ell_i,
\end{equation*}
where $\boldsymbol{\widetilde{y}}^{(i)} = \frac{1+ \boldsymbol{y}^{(i)}}{2}$, $\boldsymbol{\hat{y}}^{(i)} = \boldsymbol{X}^{(i)}(\boldsymbol{\varTheta}_c + \boldsymbol{\varTheta}^{(i)}_g + \boldsymbol{\varTheta}^{(t)}_p )$.\\
\noindent \textbf{Gradient Computation}: According to the Quadratic formulation of the AUC loss, we can calculate the gradients $\nabla_{\boldsymbol{\varTheta}_c}\mathcal{J}$,
 $\nabla_{\boldsymbol{\varTheta}_g}\mathcal{J}$,  $\nabla_{\boldsymbol{\varTheta}_p}\mathcal{J}$, as follows:
\[\begin{array}{ll}
\nabla_{\boldsymbol{\varTheta}_c}\mathcal{J} &=  \sum_{i} \boldsymbol{X}^{(i)\top}\mathcal{L}_{AUC}^{(i)}\left(\boldsymbol{X}^{(i)}\boldsymbol{W}^{(i)}  - \bm{Y}^{(i)} \right),\\ 
\nabla_{\boldsymbol{\varTheta}^{(i)}_g}\mathcal{J} &= \boldsymbol{X}^{(i)\top}\mathcal{L}_{AUC}^{(i)}\left(\boldsymbol{X}^{(i)}\boldsymbol{W}^{(i)}  - \bm{Y}^{(i)} \right),\\ 
\nabla_{\boldsymbol{\varTheta}^{(i)}_p}\mathcal{J} &= \boldsymbol{X}^{(i)\top}\mathcal{L}_{AUC}^{(i)}\left(\boldsymbol{X}^{(i)}\boldsymbol{W}^{(i)}  - \bm{Y}^{(i)} \right). 
\end{array}\]

\noindent \textbf{Efficient Computation}. According to the definition of $\mathcal{G}_{AUC}^{(i)} $ and $\mathcal{E}^{(i)}$, the affinity matrix of $\mathcal{G}_{AUC}^{(i)} $ could be written as: 
 \begin{equation*}
 \mathcal{W}^{(i)} = \frac{1}{n_+n_-}[\boldsymbol{\tilde{y}^{(i)}}(\boldsymbol{1}-\boldsymbol{\tilde{y}^{(i)}})^\top + (1-\boldsymbol{\tilde{y}^{(i)}})(\boldsymbol{\tilde{y}^{(i)}})^\top ]. 
 \end{equation*}
 Correspondingly, $\mathcal{L}_{AUC}^{(i)}$ could be simplified as:
 \begin{equation}\label{key}
 \mathcal{L}_{AUC}^{(i)} = diag\left(\frac{\boldsymbol{\tilde{y}^{(i)}}}{n_{+,i}}+ \frac{\boldsymbol{1} - \boldsymbol{\tilde{y}^{(i)}}}{n_{-,i}}\right) -\mathcal{W}^{(i)}
\end{equation}
Denote $\bm{R}^{(i)} = \bm{X}^{(i)}\bm{W}^{(i)} - \bm{Y}^{(i)}$, now we are ready to speed-up the loss evaluation $\sum_i \bm{R}^{(i)^\top}  \mathcal{L}_{AUC}^{(i)} \bm{R}^{(i)}$.
We have:
\begin{equation}\label{eq:loss_eve}
\begin{split}
\bm{R}^{(i)^\top}  \mathcal{L}_{AUC}^{(i)}\bm{R}^{(i)}   = &\bm{R}^{(i)^\top}  \Bigg(diag\left(\frac{\boldsymbol{\tilde{y}^{(i)}}}{n_{+,i}}+ \frac{\boldsymbol{1} -
	 \boldsymbol{\tilde{y}^{(i)}}}{n_{-,i}}\right)\Bigg) \bm{R}^{(i)}  \\
	  &- {\bm{R} }^{(i)}_+{{\bm{R} }^{(i)}_-} - {{\bm{R} }^{(i)}_-}{{\bm{R} }^{(i)}_+},
\end{split}
\end{equation}
where 
\begin{equation*}
\begin{split}
{{\bm{R}}_+} =\frac{1}{n_{+,i}} \bm{R}^{(i)^\top} \boldsymbol{\tilde{y}}^{(i)},\  \ {{\bm{R}}_-} =\frac{1}{n_{-,i}} \bm{R}^{(i)^\top}(\boldsymbol{1} -\boldsymbol{\tilde{y}}^{(i)}).\\
\end{split}
\end{equation*}
Similarly, we have the following simplification for the gradients:
\begin{equation}\label{eq:grad_eve}
\begin{split}
\bm{X}^{(i)^\top}\mathcal{L}_{AUC}^{(i)}\bm{R}^{(i)} = &\bm{X}^{(i)^\top} \Bigg(diag\left(\frac{\boldsymbol{\tilde{y}^{(i)}}}{n_{+,i}}+ \frac{\boldsymbol{1} -
	 \boldsymbol{\tilde{y}^{(i)}}}{n_{-,i}}\right)\Bigg)\\  &- {\bm{X}_+^{(i)}}\bm{R}_-^{(i)} - {\bm{X}_-^{(i)}}\bm{R}_+^{(i)}.
\end{split}
\end{equation}
where
\begin{equation*}
\begin{split}
{\bm{X}_+^{(i)}} =\frac{1}{n_{+,i}} \bm{X}^{(i)^\top} \boldsymbol{\tilde{y}}^{(i)},\  \ {\bm{X}_-^{(i)}} =\frac{1}{n_{-,i}} \bm{X}^{(i)^\top} (\boldsymbol{1} -\boldsymbol{\tilde{y}}^{(i)}).\\
\end{split}
\end{equation*}
From \refeq{eq:loss_eve} and \refeq{eq:grad_eve}, we know that the complexity for computing loss and gradient could be reduced from at most $\mathcal{O}(\sum_i n_i^2)$ and $\mathcal{O}(\sum_i n_i^2\cdot d)$ respectively to $\mathcal{O}(\sum_i n_i \cdot d)$and $\mathcal{O}(n_i\cdot d)$ respectively. Applying this rule, we can compute the loss function and its gradients with a linear time w.r.t. the sample size.\\ 
Next, we extend the optimization method proposed in the last section to solve the problem.

\noindent \textbf{Optimization}. It could be proved that the empirical loss $\mathcal{J} = \sum_{i=1}^T\ell_i$ has Lipschitz continuous gradients with respect to $\boldsymbol{\Theta}= [\boldsymbol{\Theta}_c, vec(\boldsymbol{\Theta}_g), vec(\boldsymbol{\Theta}_p)]$ with bounded input $\boldsymbol{X}$. We denote the Lipschitz constant as $\varrho_{\varTheta}$ (see our appendix).
Picking $C > \varrho_{\varTheta} $ at each iteration of $t$, we could solve the parameters with the following subproblems:
\begin{equation}
(\boldsymbol{Prox_g})\ \argmin_{\boldsymbol{\Theta}_g,\boldsymbol{U} \in \Gamma}  
\left\{
\begin{split}
 &\dfrac{1}{2} \left\norm{\boldsymbol{\varTheta}_g-\widetilde{\boldsymbol{\varTheta}}_g^t \right}_F^2 + \frac{\alpha_1}{C} \left<\boldsymbol{\varTheta}_{\mathcal{G}},\boldsymbol{U}\right> \\
&+ \frac{\alpha_2}{2C} \norm{\g}_F^2\\
\end{split}
\right\},
\label{p2}
\end{equation}

\begin{equation}
\begin{split}
(\boldsymbol{Prox_c})& \ \argmin_{\boldsymbol{\Theta}_c}  \dfrac{1}{2} \left\norm{\boldsymbol{\varTheta}_c -\widetilde{\boldsymbol{\varTheta}_c^t} \right}_2^2 + \frac{\alpha_3}{2C} \norm{\boldsymbol{\varTheta}_c}_2^2 \label{p1},
\end{split}
\end{equation}	
	
\begin{equation}
\begin{split}
  (\boldsymbol{Prox_p})& \ \argmin_{\boldsymbol{\Theta}_p}
\dfrac{1}{2} \left\norm{\boldsymbol{\varTheta}_p -\widetilde{\boldsymbol{\varTheta}_p^t} \right}_F^2 + \frac{\alpha_4}{C} \norm{\boldsymbol{\varTheta}_p}_{1,2} \label{p3},
\end{split}
\end{equation}	
where
\begin{equation}\label{eq:tg}
\boldsymbol{\widetilde{\varTheta}}_g^t= \boldsymbol{\varTheta}_g^{t-1} - \dfrac{1}{C}\nabla_{\boldsymbol{\varTheta}_g}\mathcal{J}(\boldsymbol{\varTheta}^{t-1}),
\end{equation}

\begin{equation}\label{eq:tc}
\boldsymbol{\widetilde{\varTheta}}_c^t= \boldsymbol{\varTheta}_c^{t-1} - \dfrac{1}{C}\nabla_{\boldsymbol{\varTheta}_c}\mathcal{J}(\boldsymbol{\varTheta}^{t-1}),
\end{equation}

\begin{equation}\label{eq:tp}
\boldsymbol{\widetilde{\varTheta}}_g^t= \boldsymbol{\varTheta}_p^{t-1} - \dfrac{1}{C}\nabla_{\boldsymbol{\varTheta}_p}\mathcal{J}(\boldsymbol{\varTheta}^{t-1}).
\end{equation}

{Similar to Alg.\ref{alg:opt}, we adopt Alg.\ref{alg:opt_app} to optimize the parameters.\\

\indent In the end of this section, we show that Alg. \ref{alg:opt_app}    inherits the theoretical merits of Alg. \ref{alg:opt}.  

\begin{thm} \label{thm:conv_app} Denote by  \[
	 \widetilde{\mathcal{F}} =\left\{\begin{split}
&\mathcal{J}({\c},\g,\p)+  \alpha_1 \left<\bm{\varTheta}_{\mathcal{G}}, \boldsymbol{U}\right>  + \frac{\alpha_2}{2} \norm{\g}_{F}^2\\ & +\frac{\alpha_3}{2}\norm{{\c}}_2^2+   \alpha_4 \norm{\p}_{1,2} + \iota_{\Gamma}(\bm{U})
	\end{split}\right\} \] the loss function and denote by $(\ct,\gt,\pt,\boldsymbol{U}^t)$ the parameter obtained at iteration $t$. If  the data is bounded in the sense that:
	$\forall i, ~\norm{\boldsymbol{X}^{(i)}}_2 =\vartheta_{X_i} < \infty, ~n_{+,i} \ge 1, ~n_{-,i} \ge 1$, then pick  $C > \varrho_{\Theta}$, where $\varrho_{\varTheta}  
	=3T\sqrt{(2T+1)}\max_{i} \left\{\dfrac{n_i\vartheta^2_{X_i}}{n_{+,i}n_{-,i}}\right\}$, the following properties hold for Alg. \ref{alg:opt_app}:
	\begin{enumerate}[itemindent=0pt, leftmargin =15pt,label={(\arabic*)}]
	
	\item The parameter sequence $(\ct,\gt,\pt, \boldsymbol{U}^t)$ converges to a critical point $(\c^\star,\g^\star,\p^\star,\boldsymbol{U}^\star)$ of the problem $\bm{Q}$.
	\item The loss sequence $\{\widetilde{\mathcal{F}}_t\}_t$ converges to a critical point $\widetilde{\mathcal{F}}^\star$ of the problem $\bm{Q}$.
	\item For all $t \in \mathbb{N}$, there exists a subgradient  $\bm{g}_t$, such that when $T \rightarrow +\infty$, $\dfrac{1}{T} ({\sum\limits_{t=1}^T\norm{\bm{g}_t}^2}) \rightarrow 0$ with rate $\mathcal{O}(\frac{1}{T})$.
	\end{enumerate}
\end{thm}

\begin{thm}\label{thm:group_app}
Assume that Alg.~\ref{alg:opt_app} terminates at the $\mathcal{T}$-th iteration with $\tilde{\mathcal{F}}_{\mathcal{T}-1} \le \epsilon^\mathcal{A}_{\mathcal{T}-1}$, where $\tilde{\mathcal{F}}_{\mathcal{T}-1}$ is the objective function at the $\mathcal{T}-1$ iteration. Furthermore, assume that there is a matrix $\bm{\varTheta_g}^\star \in \mathcal{H}_{C^\mathcal{A}_0}$, where the corresponding bipartite graph $\mathcal{G}^\star$ has k connected components with a Graph Laplacian matrix $\bm{\varTheta}_{\mathcal{G}}^\mathcal{T}$ giving the ground-truth grouping. Moreover, denote $n_i$ as the number of nodes in the $i$-th group of the graph, and $n^\uparrow_1 = \max_{i} n_i$, $n^\uparrow_2 = \max_{j, n_j \le n^\uparrow_1} n_i$. With the following notations:
\begin{equation*}
 C^\mathcal{A}_0 = \left(2 \cdot \dfrac{ \epsilon^\mathcal{A}_{\mathcal{T}-1}}{\alpha_2}\right)^{1/2}, \kappa^\mathcal{A}_0 =  C^\mathcal{A}_0 +  \frac{\varkappa(\xi_c,\xi_g,\xi_p)}{C}, \delta^\mathcal{A}_1 = \frac{C}{\alpha_1} \kappa^\mathcal{A}_0
\end{equation*}

\begin{equation*}
\delta^\mathcal{A}_2 = \frac{C}{\alpha_1} \delta^\mathcal{A}_0 , ~~\xi^\mathcal{A}= (\sqrt{d+T} + \sqrt{2}) \cdot \frac{C^\mathcal{A}_0}{\lambda_{k+1}(\bm{\varTheta}^\mathcal{T}_\mathcal{G})}
\end{equation*}

\begin{equation*}
\begin{split}
&\xi_c =\left(2 \cdot \frac{ \epsilon^\mathcal{A}_{\mathcal{T}-1} }{\alpha_3}\right)^{1/2}, \  \xi_g = C^\mathcal{A}_0, \ \xi_p = {\frac{ \epsilon^\mathcal{A}_{\mathcal{T}-1} }{\alpha_4}}, ~\beta^\mathcal{A}=  \dfrac{1}{n^\uparrow_1} +  \dfrac{1}{n^\uparrow_2}, \\
& \varkappa(\xi_c,\xi_g,\xi_p)
    = \dfrac{n_i\vartheta_{X_i}}{\sqrt{n_{+,i}}n_{-.i}}  \sum_{i=1}^T \left((\xi_c+\xi_g+\xi_p)\dfrac{\vartheta_{X_i}}{\sqrt{n_{+,i}}}+ 1 \right),
\end{split}
\end{equation*}

\noindent the following facts hold for the grouping effect of $\bm{\varTheta}_g$ in Alg.\ref{alg:opt_app}~:
	\begin{enumerate}
	\item[(a)] \textbf{(no-false-positive-grouping)} If $\lambda_{k+1}(\bm{\varTheta}_{\mathcal{G}}^\mathcal{T}) > \lambda_{k}(\bm{\varTheta}_{\mathcal{G}}^\mathcal{T}) \ge 0$, $\frac{\sqrt{2}}{32}\cdot \beta^\mathcal{A} > \xi^\mathcal{A} $, and $ 8\sqrt{2}\xi^\mathcal{A} < \delta^\mathcal{A}_1 < \beta^\mathcal{A}- 8\sqrt{2} \xi^\mathcal{A} $, we have:

	\begin{equation*}\label{eq:g1}
   \mathsf{Supp}(\bm{\varTheta}_g^\mathcal{T}) \subseteq \big\{(i,j): \mathcal{G}(i) = \mathcal{G}(j) \big\} =  \mathsf{Supp}(\bm{\varTheta}_g^\star) ,
	\end{equation*}
	where  $\mathcal{G}(i)$ is the corresponding connected component in the bipartite graph $\mathcal{G}^\star$ that i belongs to.
    \item[(b)] \textbf{(correct-grouping) }If we further assume that $\min_{(i,j)}|\widetilde{\boldsymbol{\varTheta}}^{\mathcal{T}}_{i,j}| \ge \delta^\mathcal{A}_{0}> 0$, and $ 8\sqrt{2}\xi^\mathcal{A} < \min \left\{\delta^\mathcal{A}_1,\delta^\mathcal{A}_2 \right\} \le \max \left\{\delta^\mathcal{A}_1,\delta^\mathcal{A}_2 \right\}   < \beta^\mathcal{A}- 8\sqrt{2} \xi^\mathcal{A} $,  we get that :
    
      \begin{equation*}\label{eq:g2}
       \mathsf{Supp}(\bm{\varTheta}_g^\mathcal{T}) =  \mathsf{Supp}(\bm{\varTheta}_g^\star).
      \end{equation*}
    	\end{enumerate}
\end{thm}

\begin{algorithm}[!h]
  \caption{TFCL for $(\bm{Q})$}
  \label{alg:opt_app}
  \begin{algorithmic}
    \STATE {\bfseries Input:} Dataset $\mathcal{S}$, $\alpha_1$, $\alpha_2$, $\alpha_3$, $\alpha_4$, $k$, $C  (C> \varrho_\varTheta)$.
    \STATE {\bfseries Output:} Solution $\c$, $\g$, $\p$, $\boldsymbol{U}$.
    \STATE Initialize  $\boldsymbol{\varTheta}_c^0$,   $\boldsymbol{\varTheta}_g^0$, $\boldsymbol{\varTheta}_p^0$,  $\boldsymbol{U}^0 \in \Gamma$, $t=1$.
    \REPEAT
    \STATE Calculate $\boldsymbol{\widetilde{\varTheta}}_c^t$, $\boldsymbol{\widetilde{\varTheta}}_g^t$, and $\boldsymbol{\widetilde{\varTheta}}_g^t$, respectively from Eq.(\ref{eq:tc})-Eq.(\ref{eq:tp}).
    \STATE Solve $\boldsymbol{\varTheta}^t_c$ from (\ref{p1}).
    \STATE Solve $\boldsymbol{\varTheta}^t_p$ from (\ref{p3}).
    \STATE Invoke Alg.\ref{alg:opt} with  $\mathcal{S}$, $\alpha_1 = \alpha_2$, $\alpha_2 = \alpha_3$, $k$, $C$, return $
    \g^{t}, \bm{U}^t$.
    \STATE  $t = t + 1$ .
    \UNTIL{Convergence}
    \STATE $\boldsymbol{\varTheta}_c = \boldsymbol{\varTheta}_c^{t-1}$, $\boldsymbol{\varTheta}_g = \boldsymbol{\varTheta}_g^{t-1}$,$\boldsymbol{\varTheta}_p = \boldsymbol{\varTheta}_p^{t-1}$, $\boldsymbol{U} = \boldsymbol{U}^{t-1}$.
  \end{algorithmic}
\end{algorithm}
\section{Experiments}
In this section, we explore the performance of
our algorithm on both synthetic and real data. In Section \ref{sec:set} - Section \ref{sec:comp},  we
first elaborate the settings and competitors adopted in our experiments.
Then, in Section \ref{sec:sim}, we investigate the performance of our proposed algorithm on a simulated dataset. Subsequently,  in Section \ref{sec:shoes} - Section \ref{sec:sun}, we present experimental results showing how
our method performs on real-world personalized annotation datasets.
\subsection{Experimental Settings}\label{sec:set}
We adopt the average of user-wise AUC score as our evaluation method. For all the experiments, hyper-parameters are tuned based on the training and validation set (account for 85\% of the total instances), and the results on the test set are recorded. The experiments are done with 15 repetitions for each involved algorithm. For the competitors, given the prediction $[\hat{\bm{y}}^{(1)}, \cdots, \hat{\bm{y}}^{(T)}]$, we adopt the instance-wise squared loss: 
\[
 \sum_i \frac{1}{2} \cdot ||\bm{y}^{(i)} - \hat{\bm{y}}^{(i)} ||_2^2
\]  
as the loss function. For our proposed algorithm, we adopt the squared AUC loss as the final loss function to improve the performance. Meanwhile, we also record how our proposed method works with instance-wise squared loss function as an intermediate result for fairness.

\subsection{Competitors}\label{sec:comp}
Now we briefly introduce competitors adopted in this paper.
\begin{itemize}
	\item \textbf{LASSO} \cite{lasso} where each task learner is regularized with an $\ell_1$-norm constraint.
	\item \textbf{rMTFL} \cite{rMTFL} assumes that the model $\boldsymbol{W}$ can be decomposed into two components: a consensus component and a group-sparse component.
	\item \textbf{RAMUSA} \cite{RAMUSA} adopts a capped trace norm regularizer to minimize only the singular values smaller than an adaptively tuned threshold.
	\item \textbf{CoCMTL} \cite{comtl} realizes the task-specific co-clustering via minimizing the truncated sum-of-squares of the singular values of the task matrix.
	\item \textbf{NC-CMTL} \cite{tclog} explores shared information among different tasks with a non-convex low-rank spectral regularizer and a robust re-weighting scheme.
	\item \textbf{VSTGMTL} \cite{vstg} implements simultaneous variable selection and learning with a low-rank decomposition.
	\item \textbf{AMTL} \cite{amtl} provides asymmetric transfer between tasks with a sparse selection on the asymmetric transfer matrix.
\end{itemize}

\begin{figure}[h]
 \centering     
      \includegraphics[width=0.9\columnwidth]{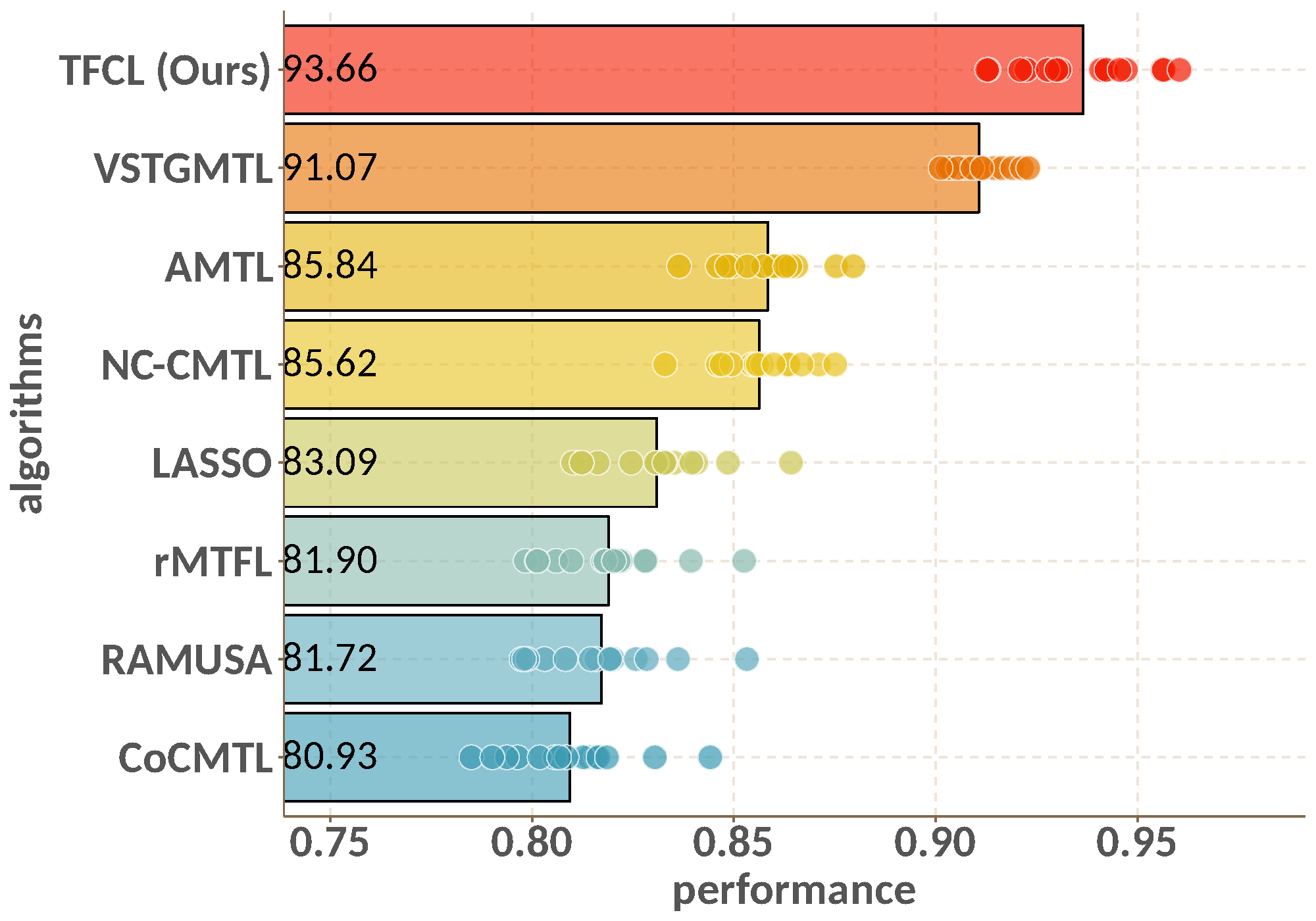}    
  \caption{AUC ($\uparrow$) comparison on the Simulated Dataset }\label{fig:box}
\end{figure}

\begin{figure}[h]
  \begin{center}
    \subfigure[obj]{\label{fig:conv:obj}
      \includegraphics[width=0.3\columnwidth]{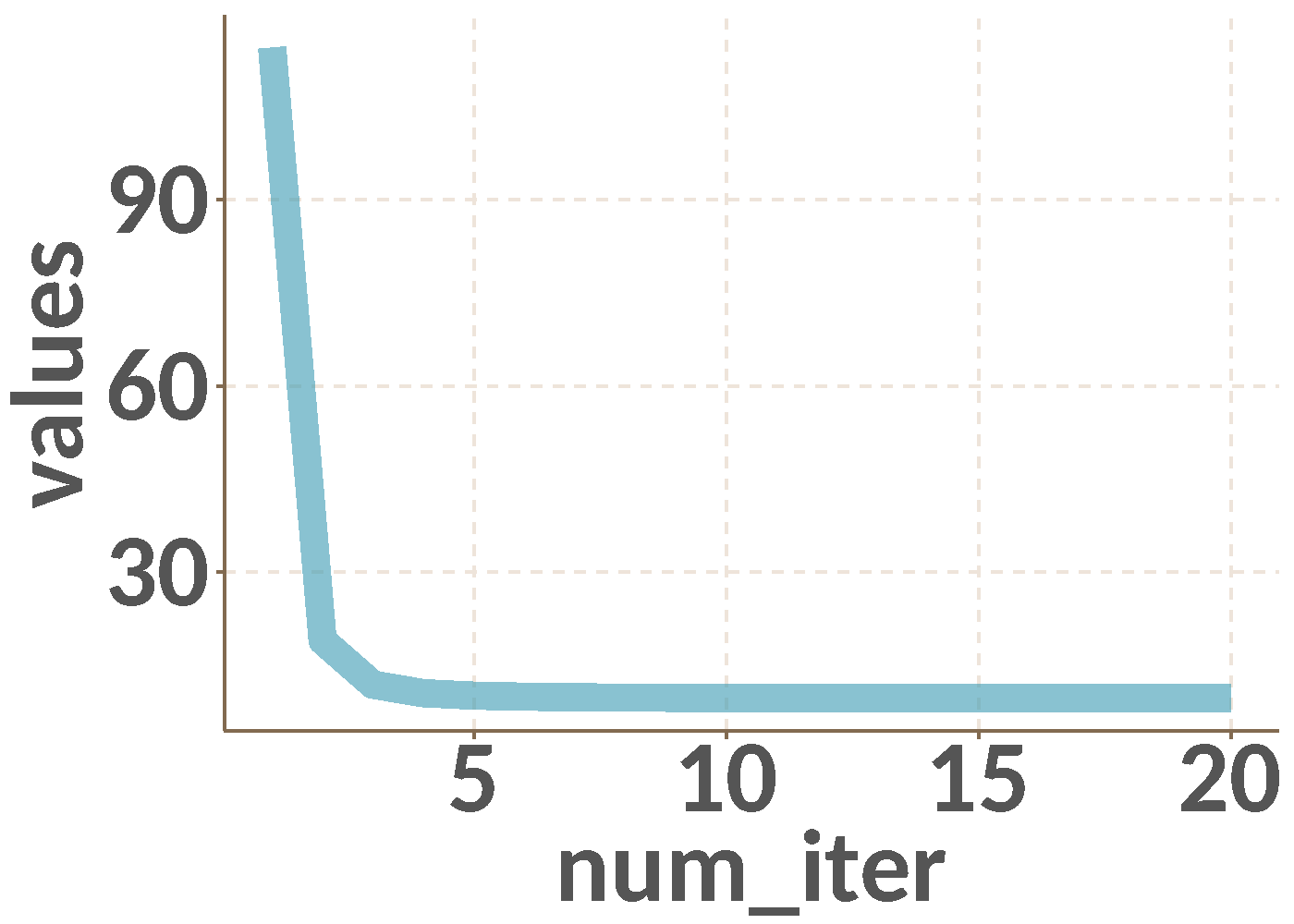}  
    }
    \subfigure[$||\bm{W}^t - \bm{W}^{t-1}||$]{\label{fig:conv:dw}
      \includegraphics[width=0.3\columnwidth]{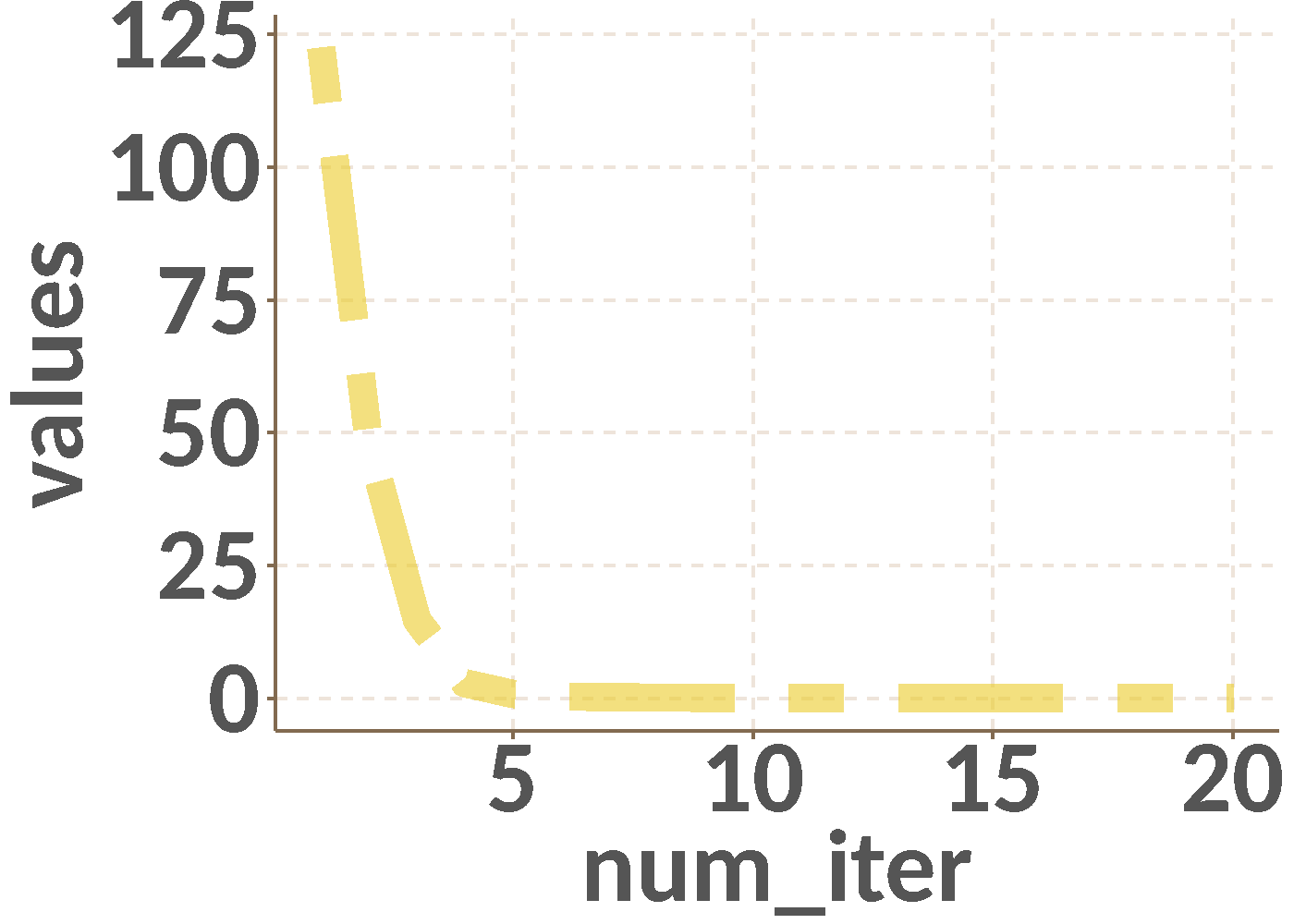} 
    }
    \subfigure[$||\bm{U}^t - \bm{U}^{t-1}||$]{\label{fig:conv:du}
      \includegraphics[width=0.3\columnwidth]{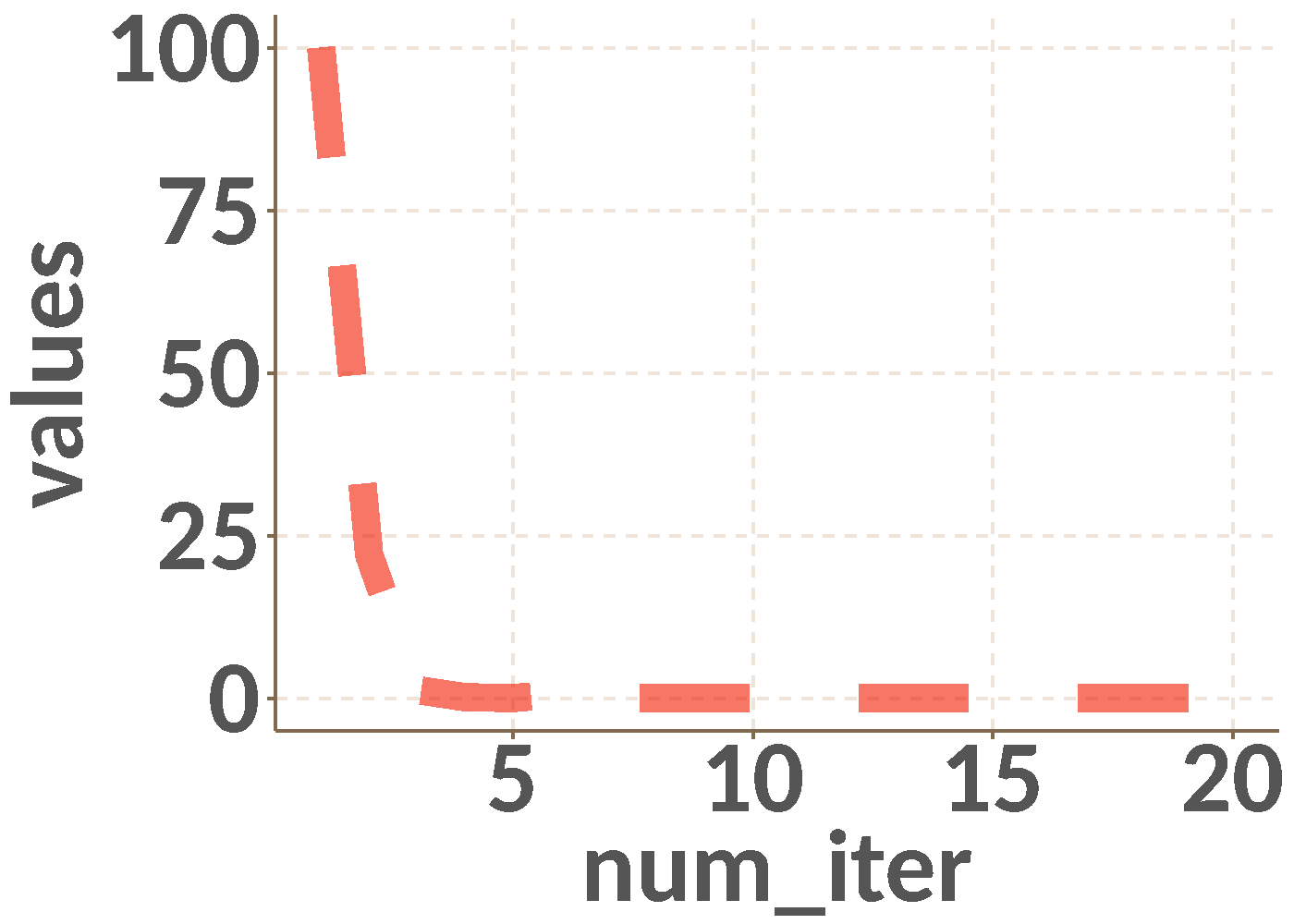} 
    }\  
  \end{center}
  \caption{\label{fig:conv}Convergence curves for (a) loss function, (b) parameter $\bm{W}$ in terms of the difference between two successive iterations  $||\bm{W}^t - \bm{W}^{t-1}||$, (c) $||\bm{U}^t - \bm{U}^{t-1}||$. }
\end{figure}

\begin{table}[htbp]
  \centering
  \setlength{\belowcaptionskip}{10pt}%
  \caption{Ablation Study for simulated dataset}
    \begin{tabular}{c|cc}
    \multicolumn{1}{c|}{\multirow{2}[1]{*}{Algorithm}} & \multicolumn{1}{c}{TFCL} & \multicolumn{1}{c}{TFCL} \\
          & \multicolumn{1}{c}{ours} & \multicolumn{1}{c}{w/o AUC loss} \\
    \midrule
    AUC   & \textbf{93.66} & 92.46 \\
    \end{tabular}%
  \label{tab:abl}%
\end{table}%

\begin{figure}[ht]  
\begin{centering}
\subfigure[iter 0]{
  \includegraphics[width=0.3\columnwidth]{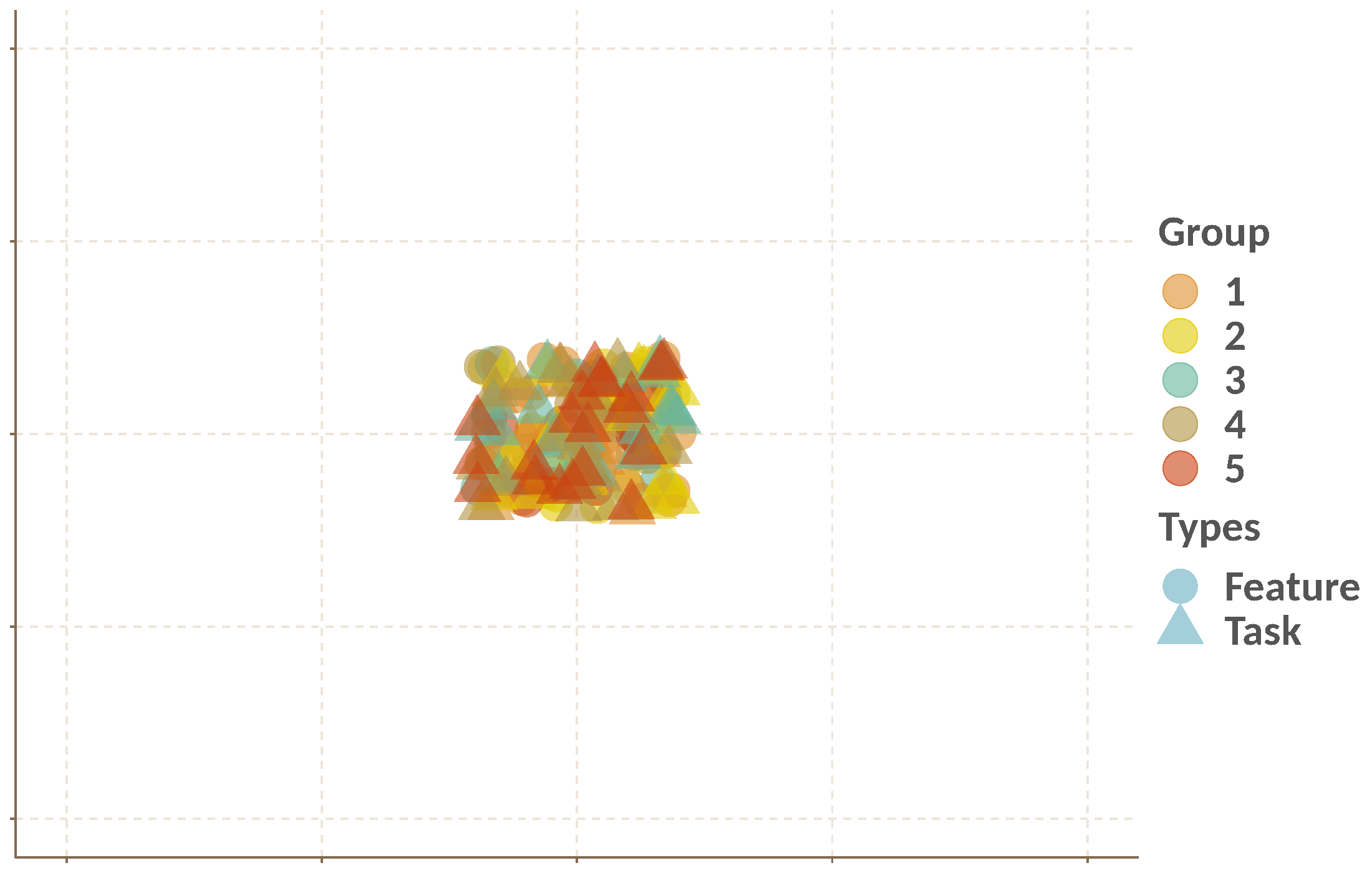} 
}
\subfigure[iter 1]{
  \includegraphics[width=0.3\columnwidth]{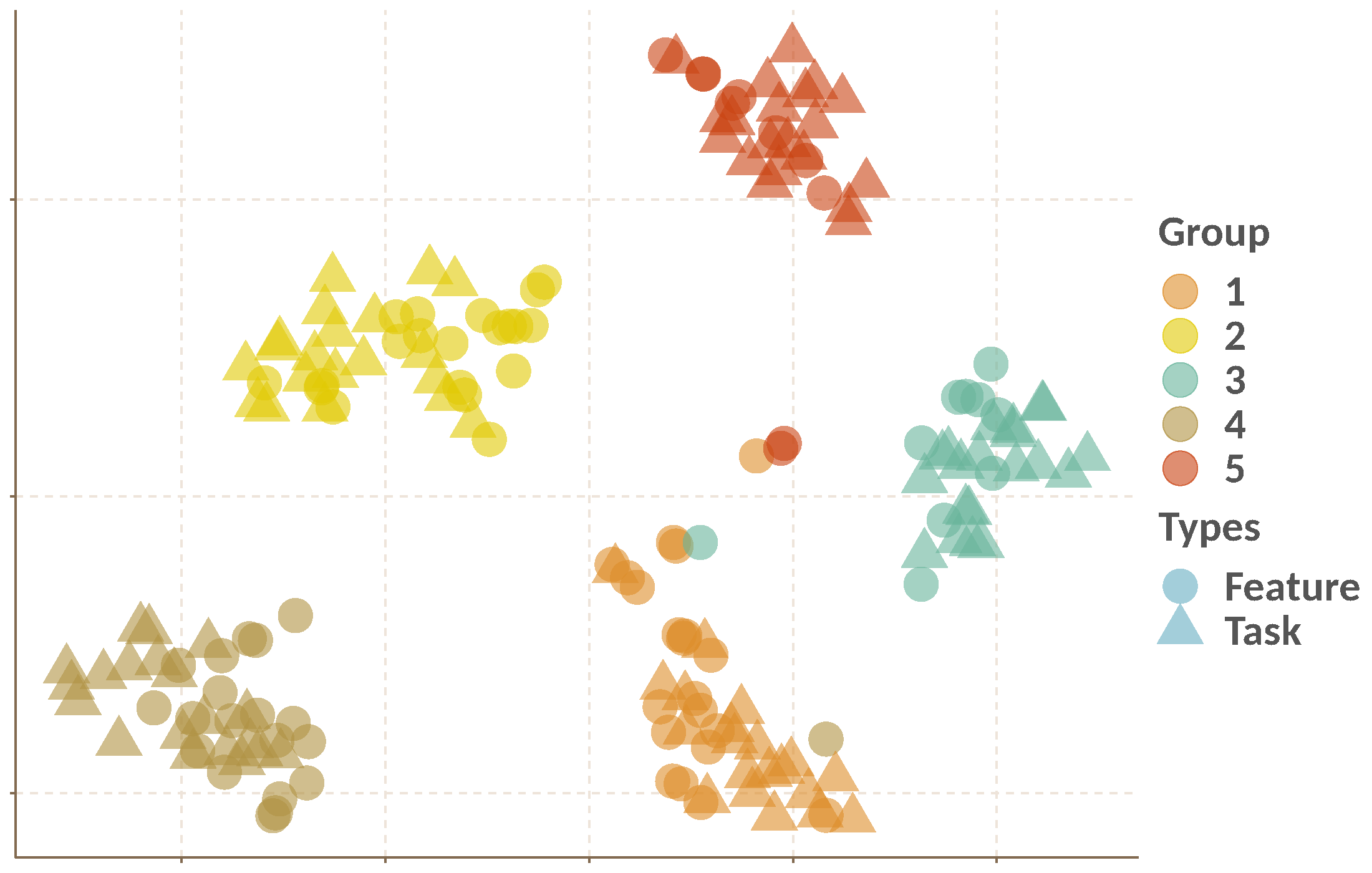} 
}
\subfigure[iter 2]{
  \includegraphics[width=0.3\columnwidth]{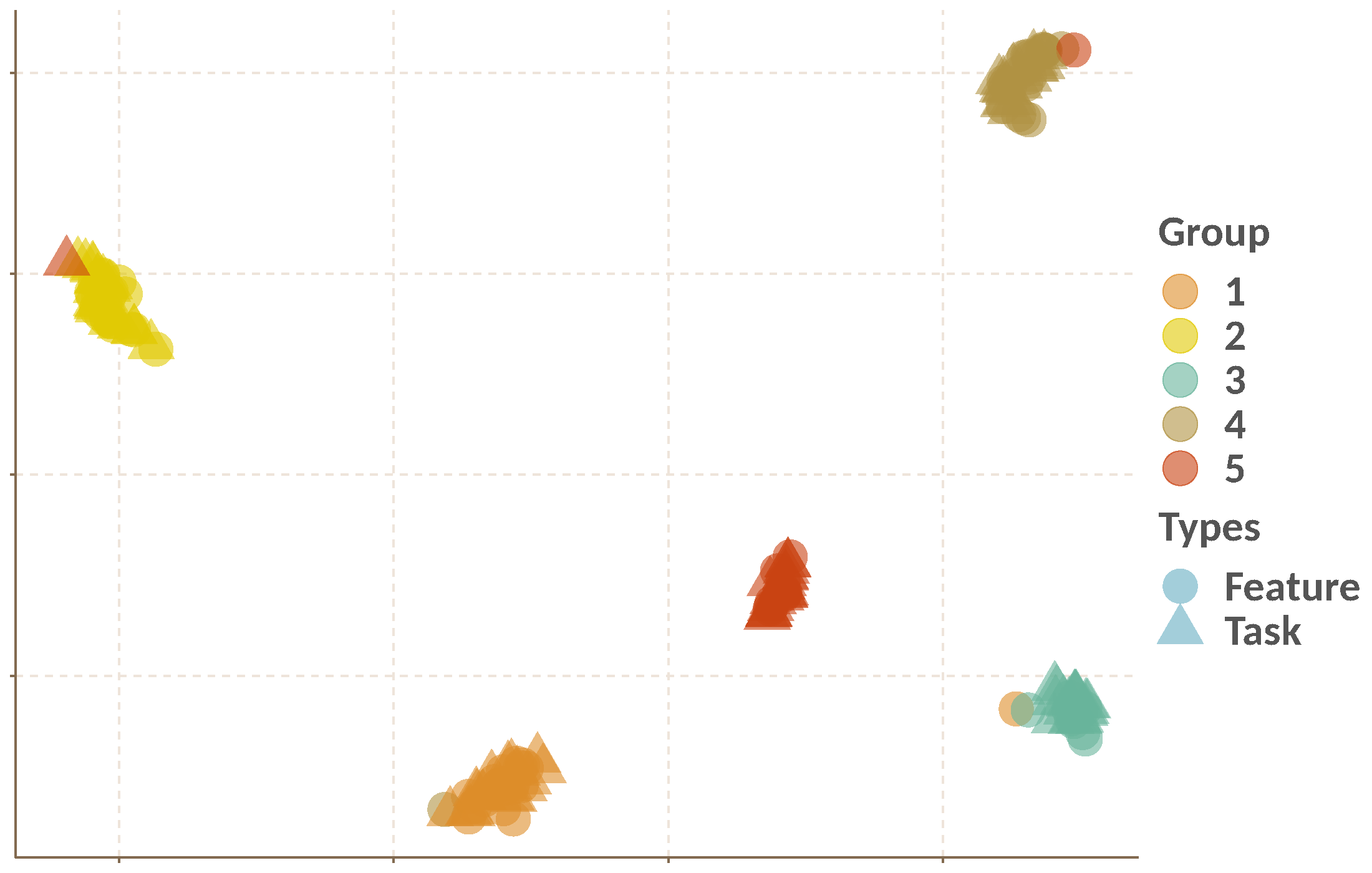} 
}\

\subfigure[iter 3]{
  \includegraphics[width=0.3\columnwidth]{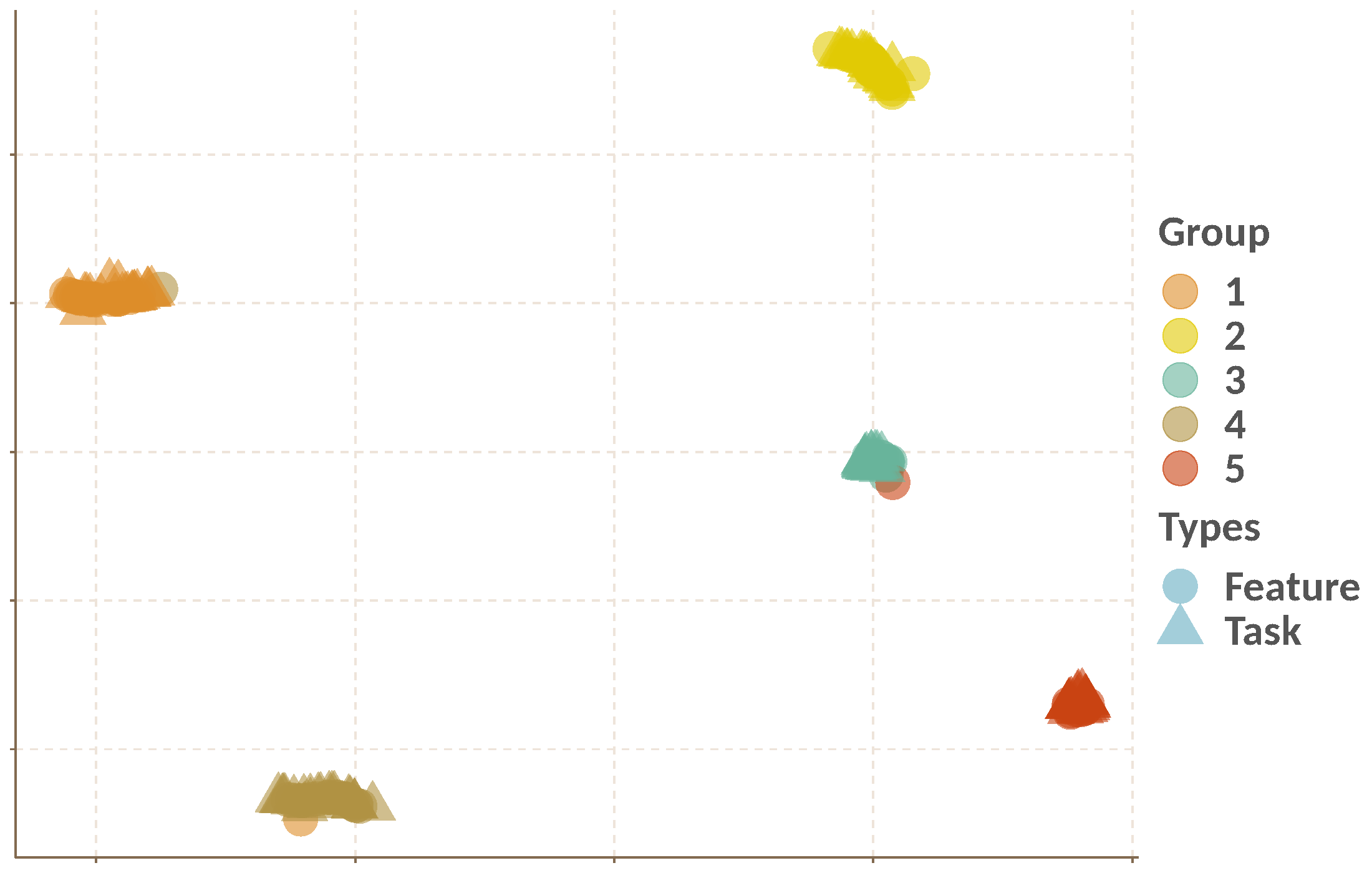} 
}
\subfigure[iter 4]{
  \includegraphics[width=0.3\columnwidth]{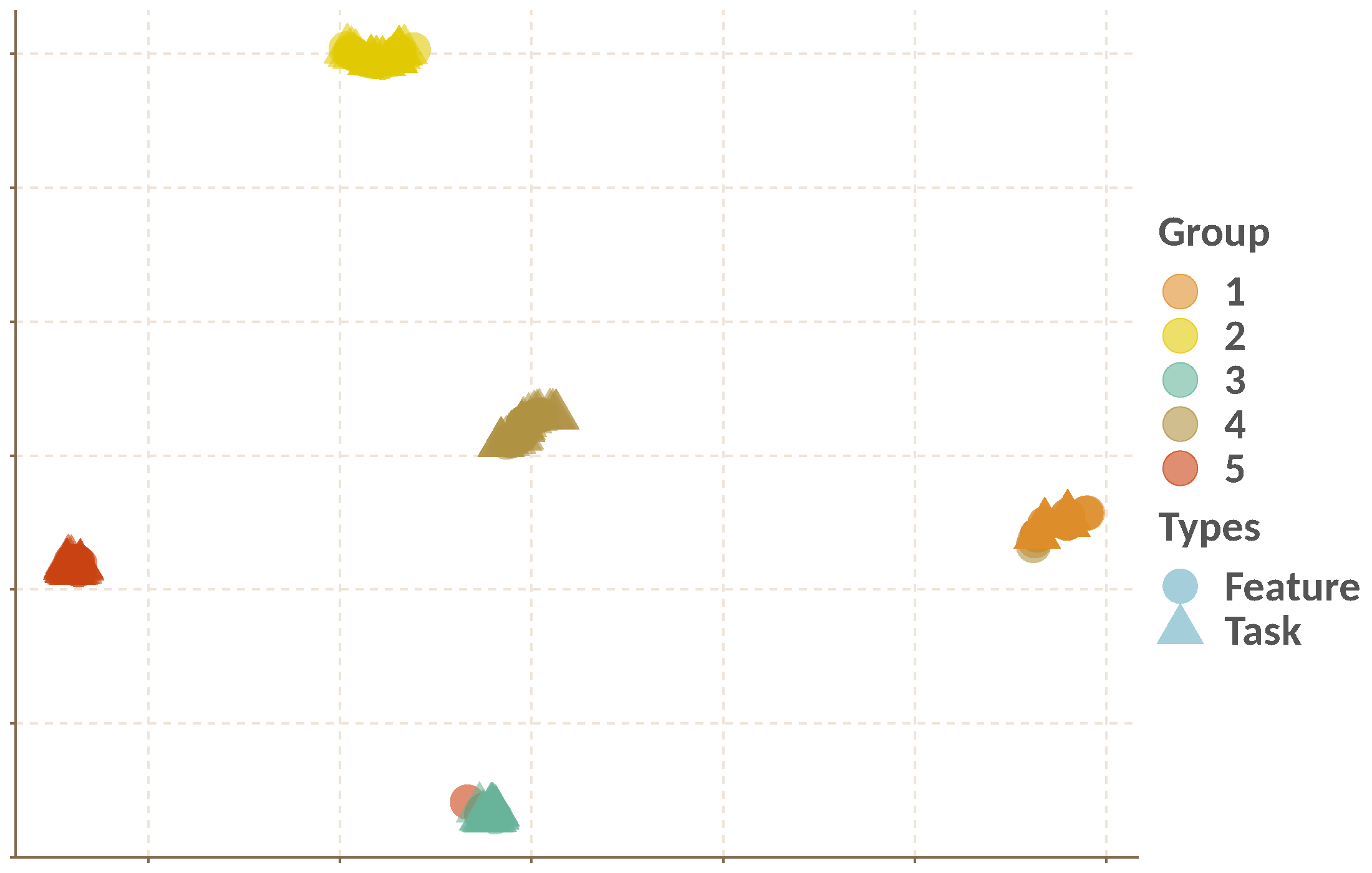} 
}
\subfigure[iter 5]{
  \includegraphics[width=0.3\columnwidth]{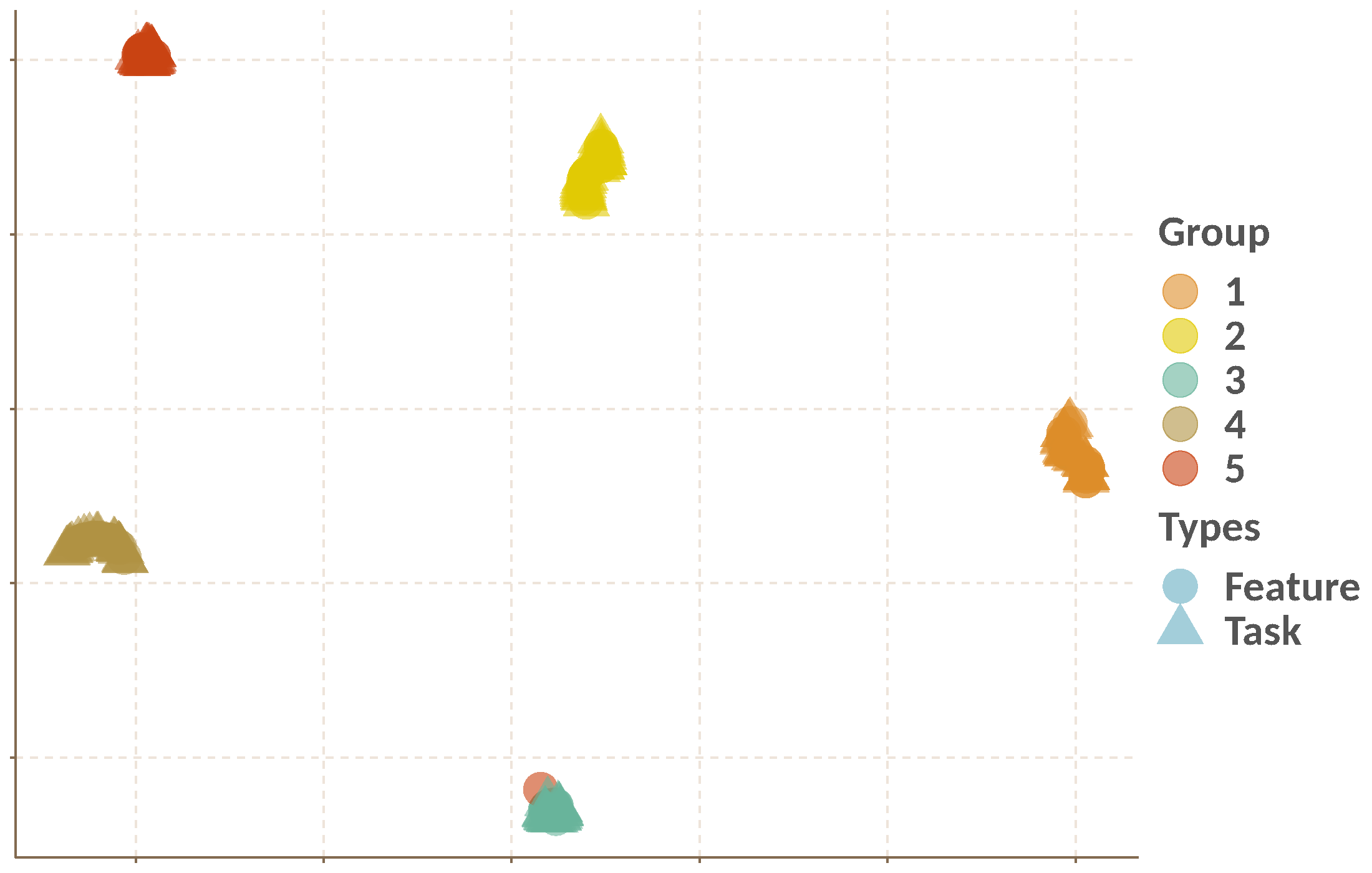} 
}
\end{centering}

\caption{\label{fig:embed}\textbf{Evolution of Spectral Embeddings}. We plot the corresponding embeddings $\bm{f}_1 \cdots, \bm{f}_{d+T}$  in the first five iterations in this group of figures. The results suggest that spectral embeddings rapidly form stable and clear clusters after the second iteration.}
\end{figure}

\begin{figure}[h]
  \centering
     \subfigure[CoCMTL]{
      \includegraphics[width=0.3\columnwidth]{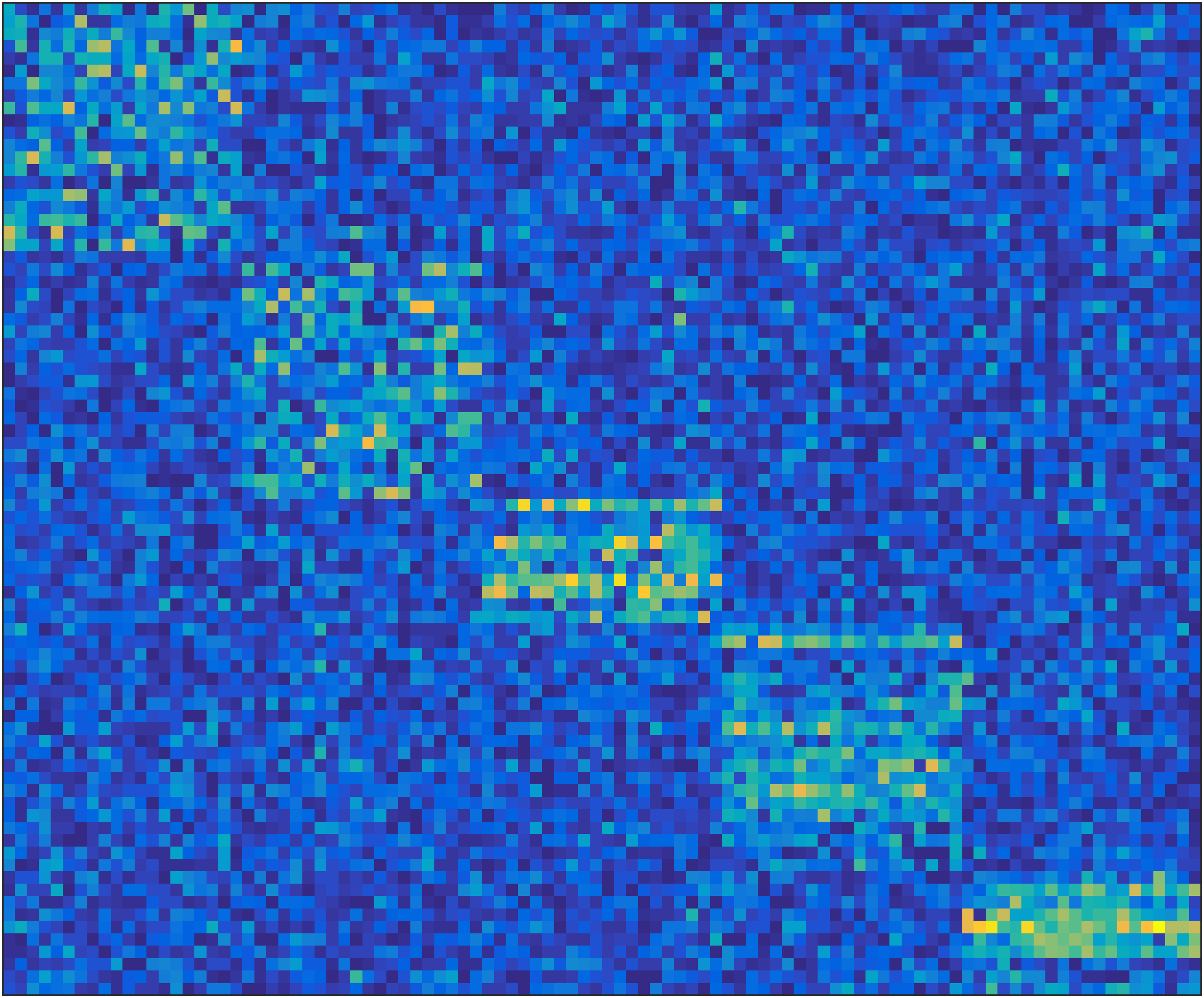}  
     }
     \subfigure[RAMUSA]{
      \includegraphics[width=0.3\columnwidth]{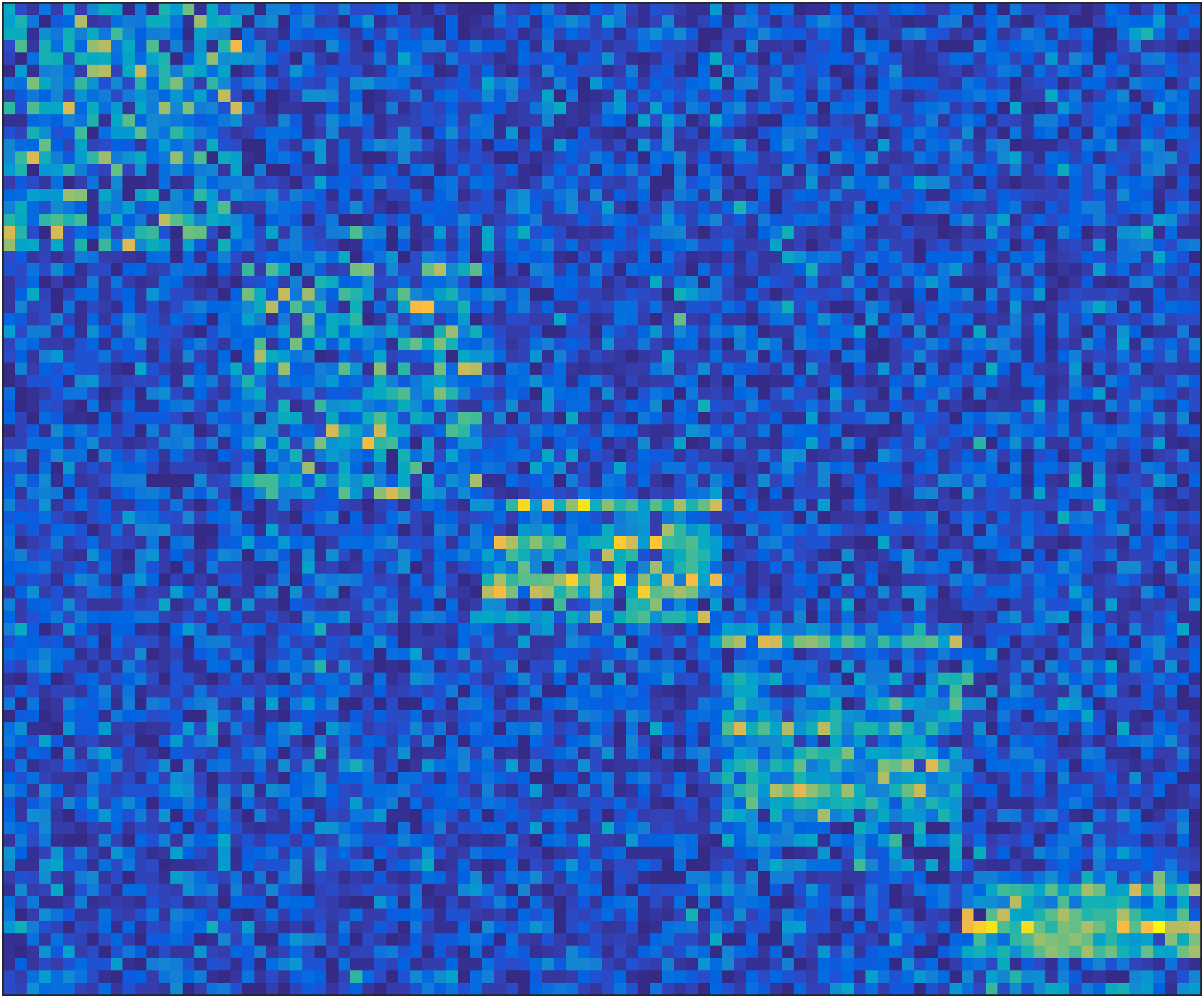} 
     }
     \subfigure[rFTML]{
      \includegraphics[width=0.3\columnwidth]{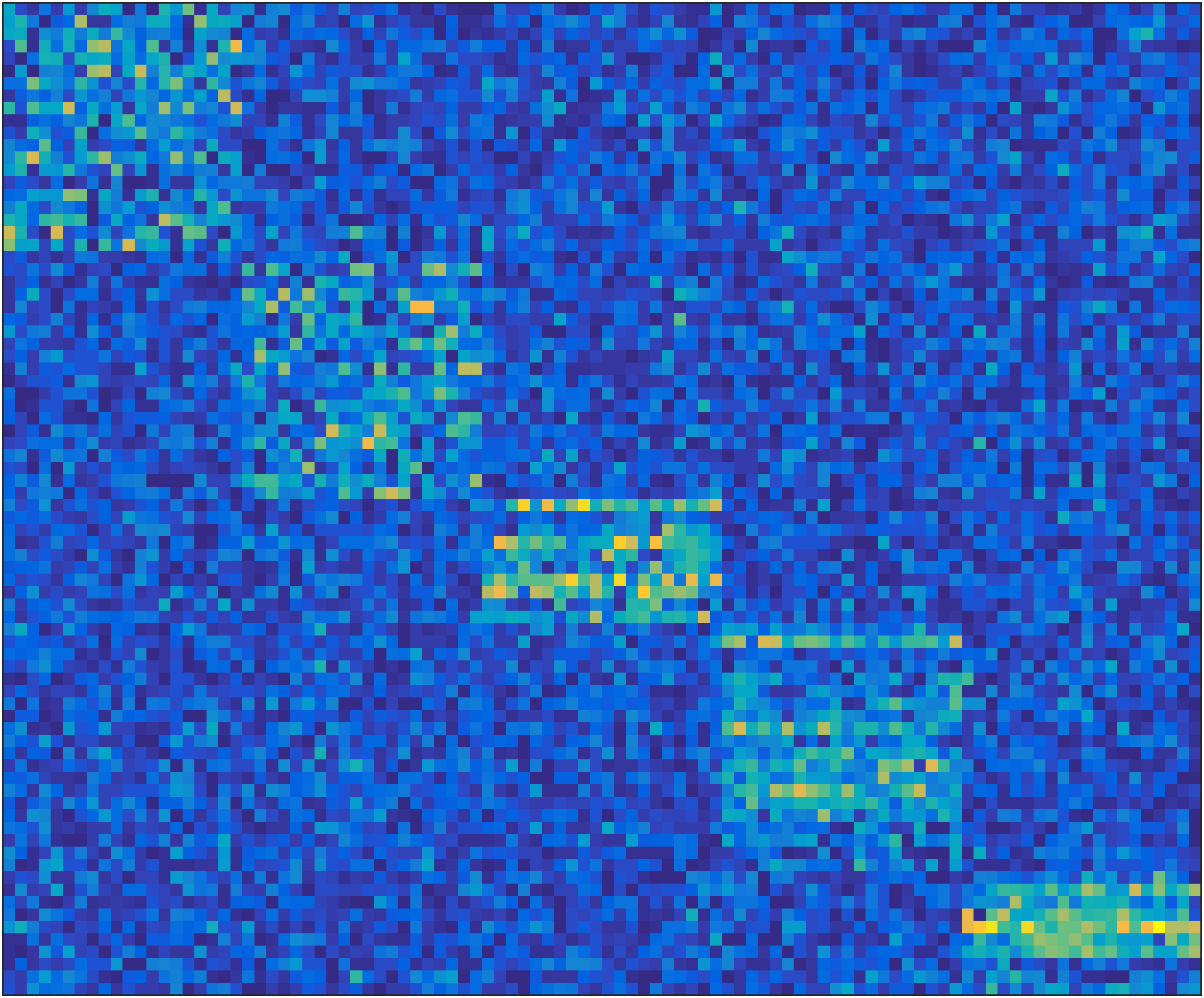}  
     }
  
  \subfigure[LASSO]{
    \includegraphics[width=0.3\columnwidth]{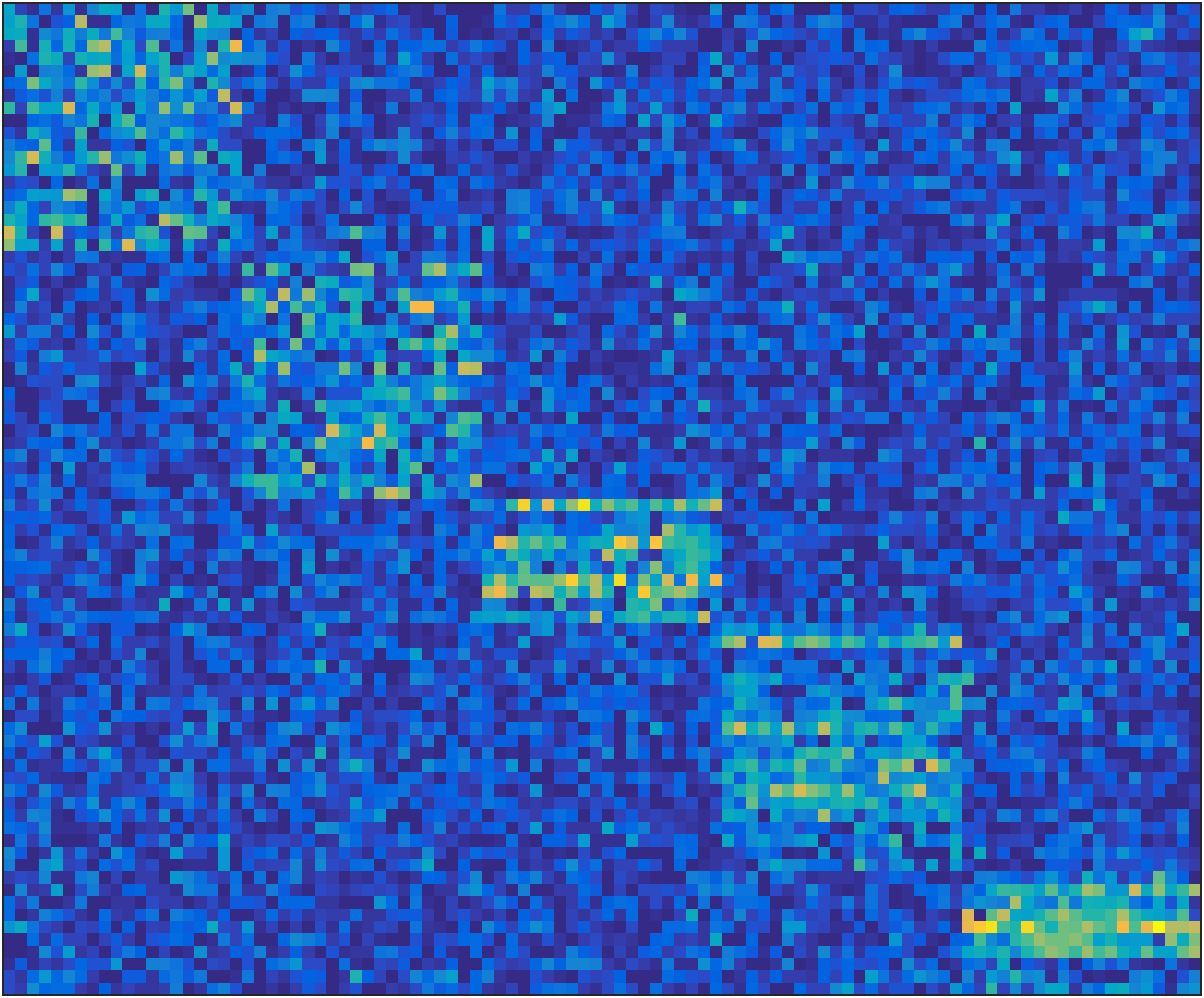}  
  }
  \subfigure[NC-CMTL]{
    \includegraphics[width=0.3\columnwidth]{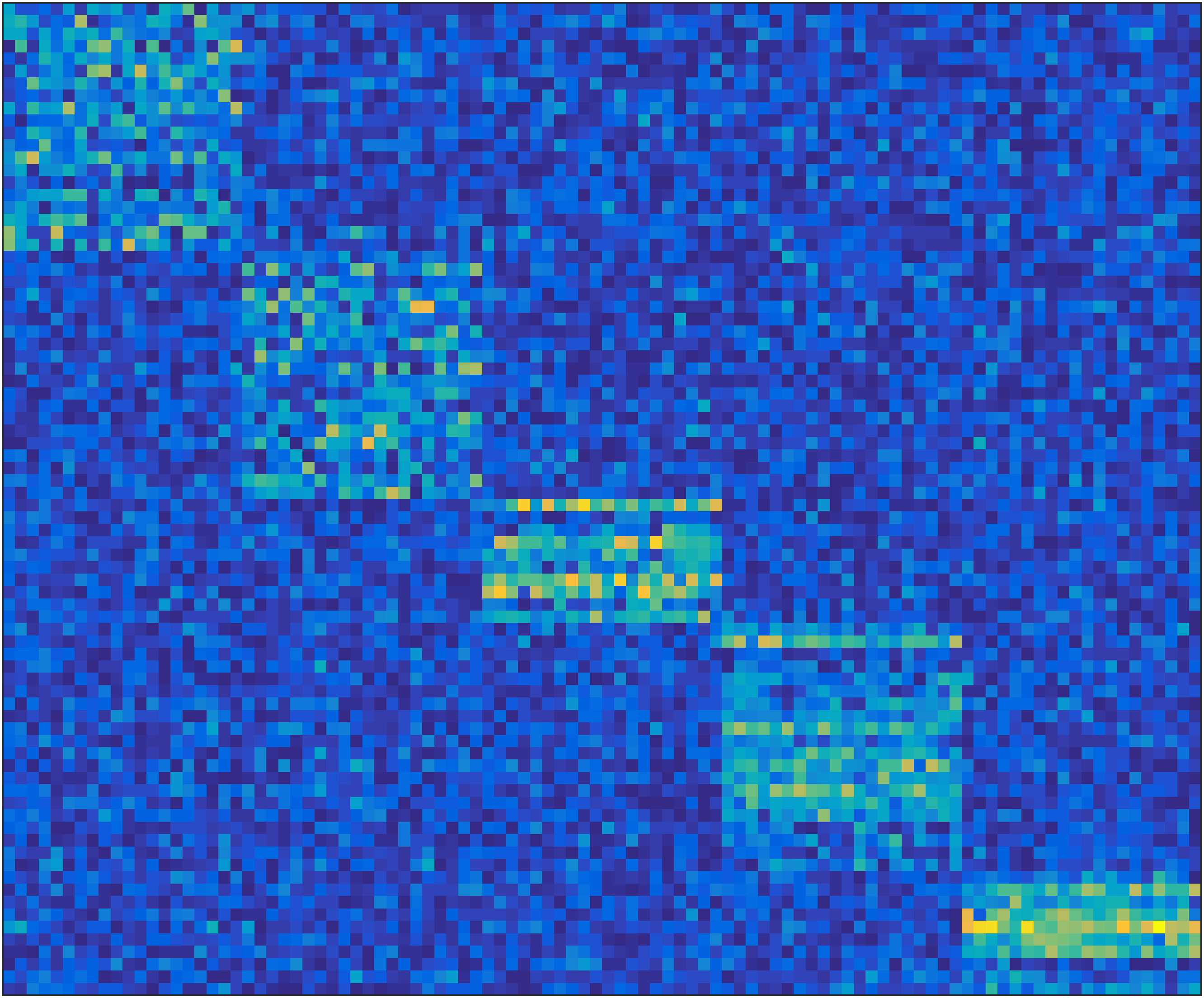}  
  }
  \subfigure[AMTL]{
    \includegraphics[width=0.3\columnwidth]{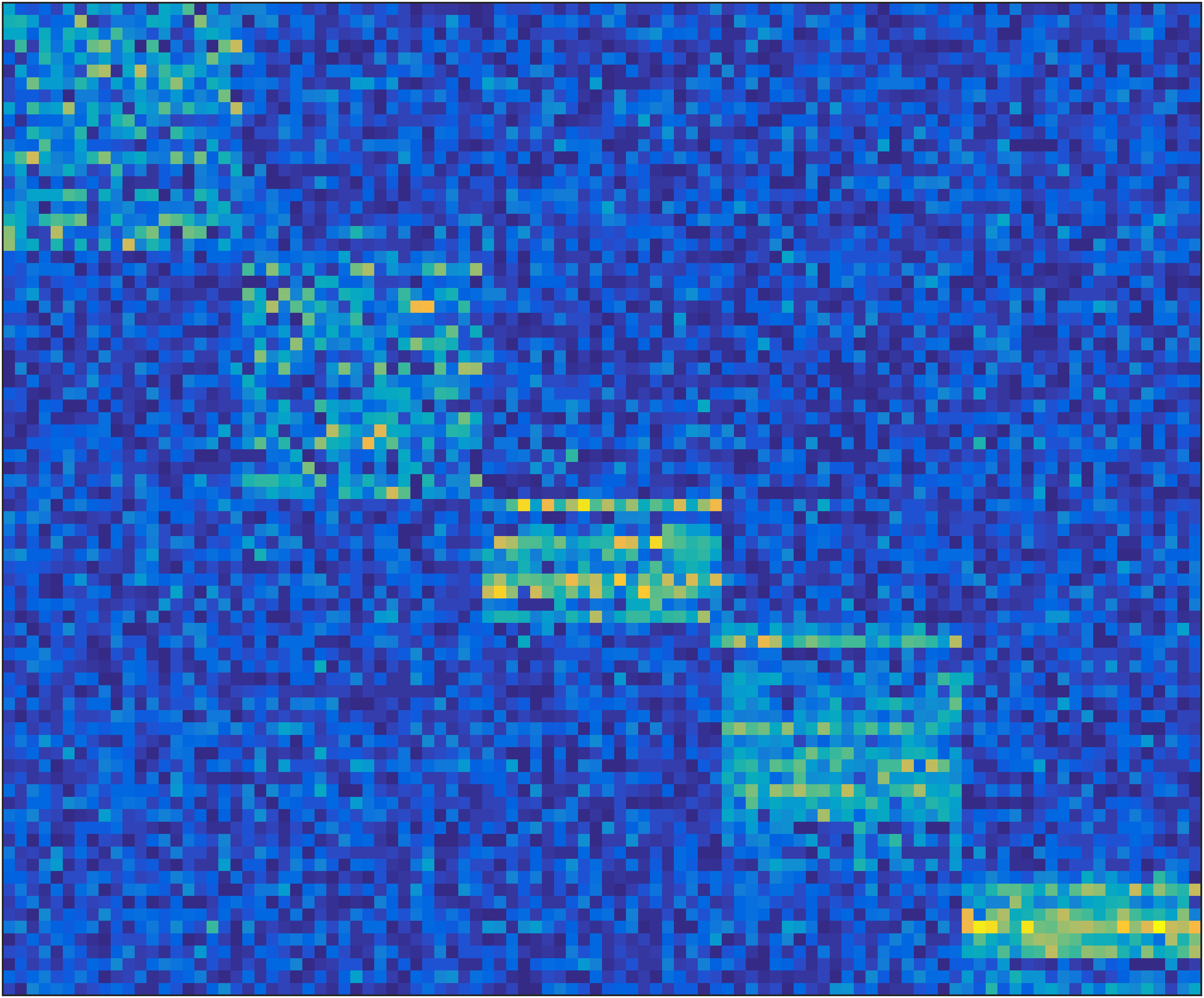} 
  }
  
  \subfigure[VSTGML]{
      \includegraphics[width=0.3\columnwidth]{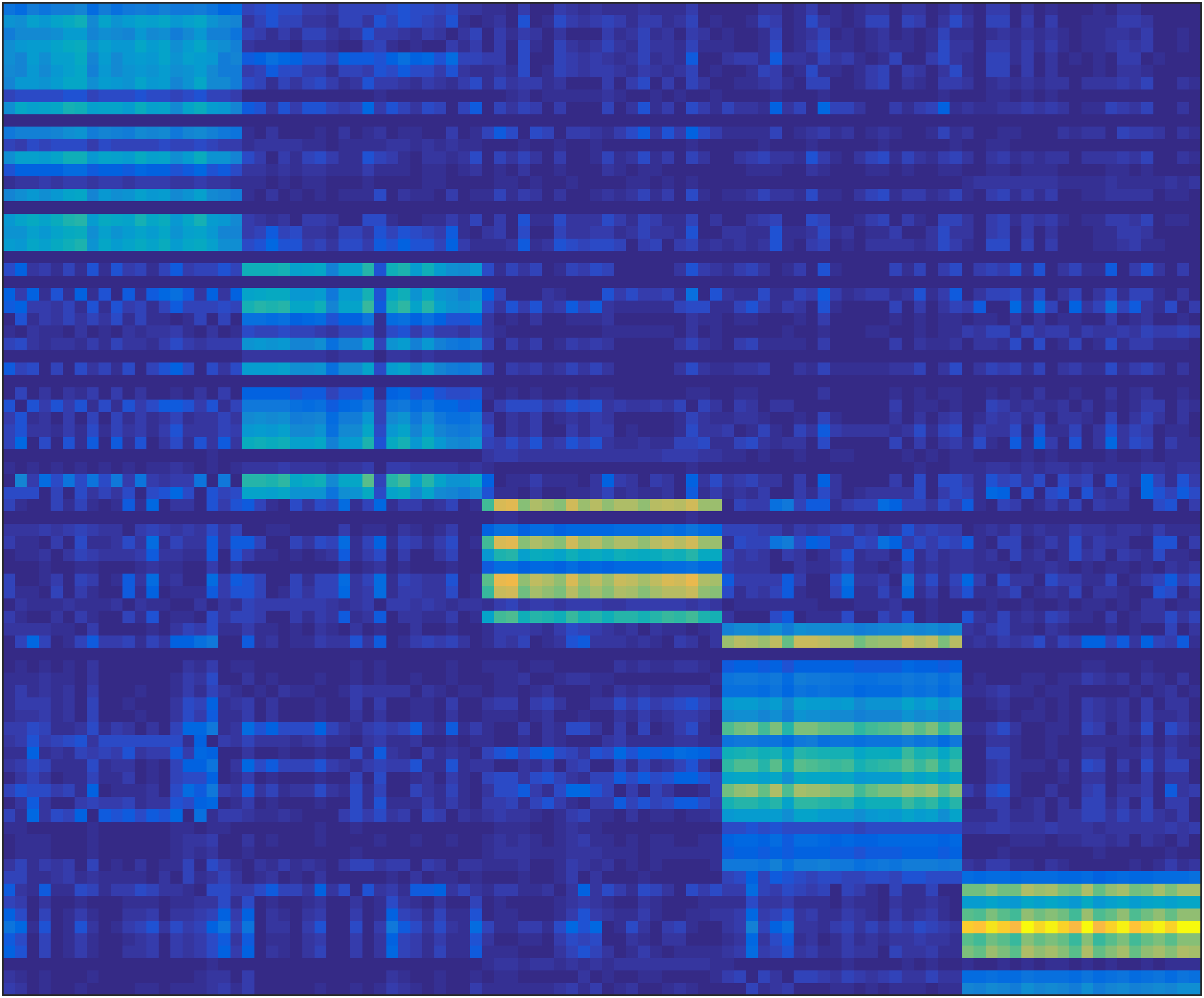} 
    }   
  \subfigure[TFCL (Ours)]{
    \includegraphics[width=0.3\columnwidth]{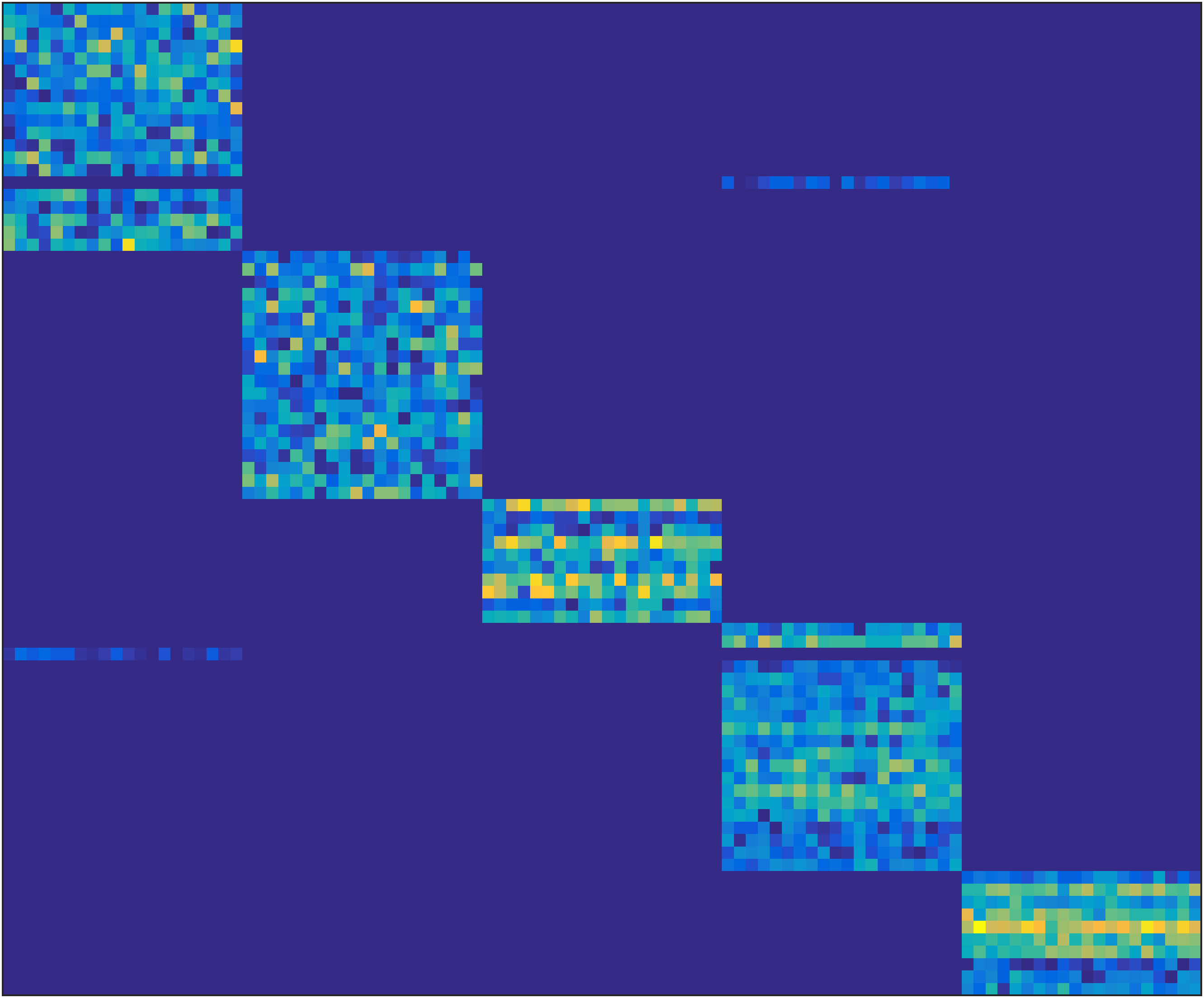} 
  }
  \subfigure[Ground Truth]{
    \includegraphics[width=0.3\columnwidth]{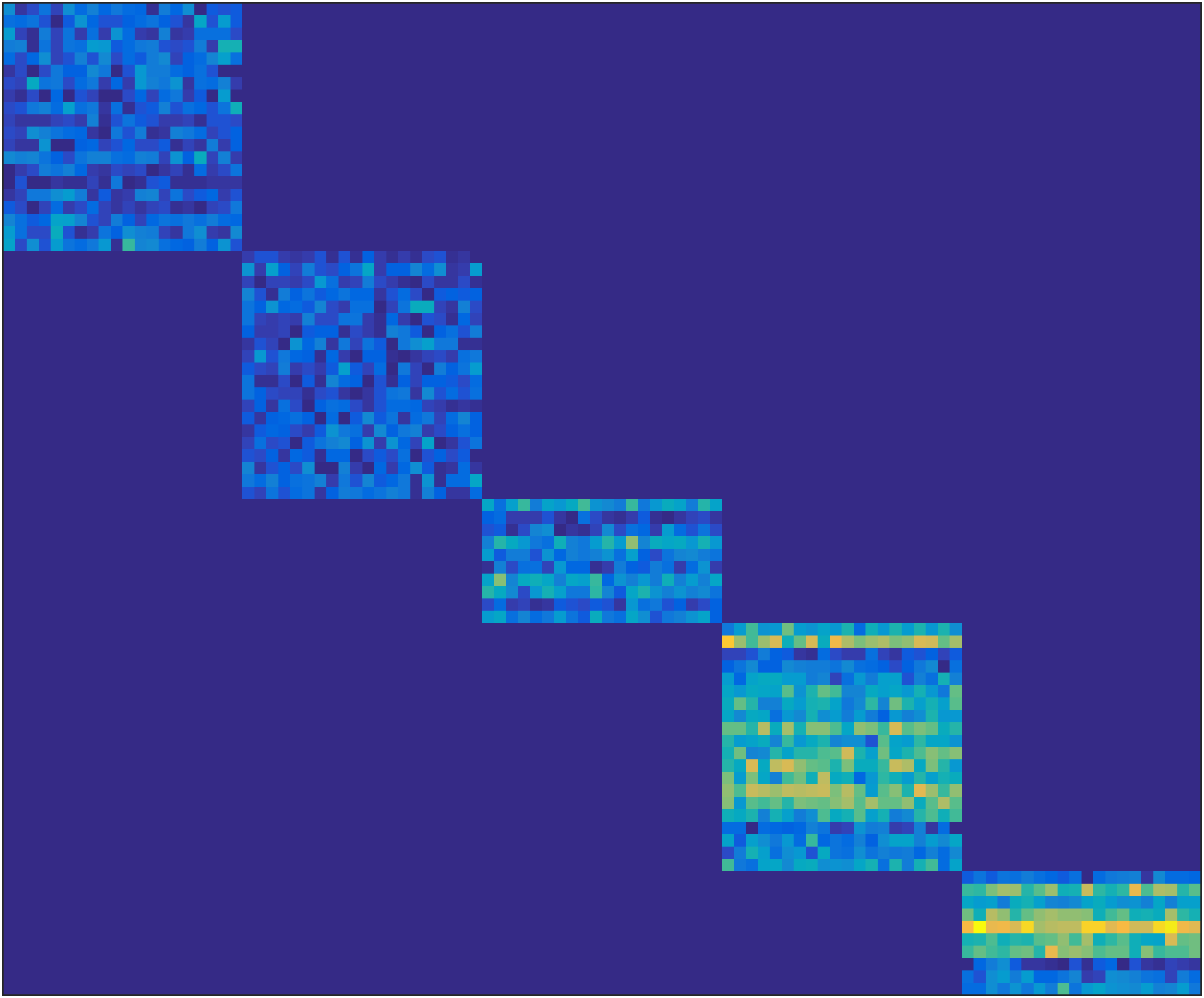} 
  }
  \caption{\textbf{Structural Recovery on Simulation Dataset.} The $x$-axis represents the users, the $y$-axis represents the feature. Compared with the competitors, TFCL could leverage a clearer block diagonal structure as the ground-truth.  }
  \label{fig:struc}
\end{figure}

\begin{figure*}[!h] 
  \begin{center}
\subfigure[Shoes BR]{\label{fig:perf_shoes_brown} 
      \includegraphics[width=0.3\textwidth]{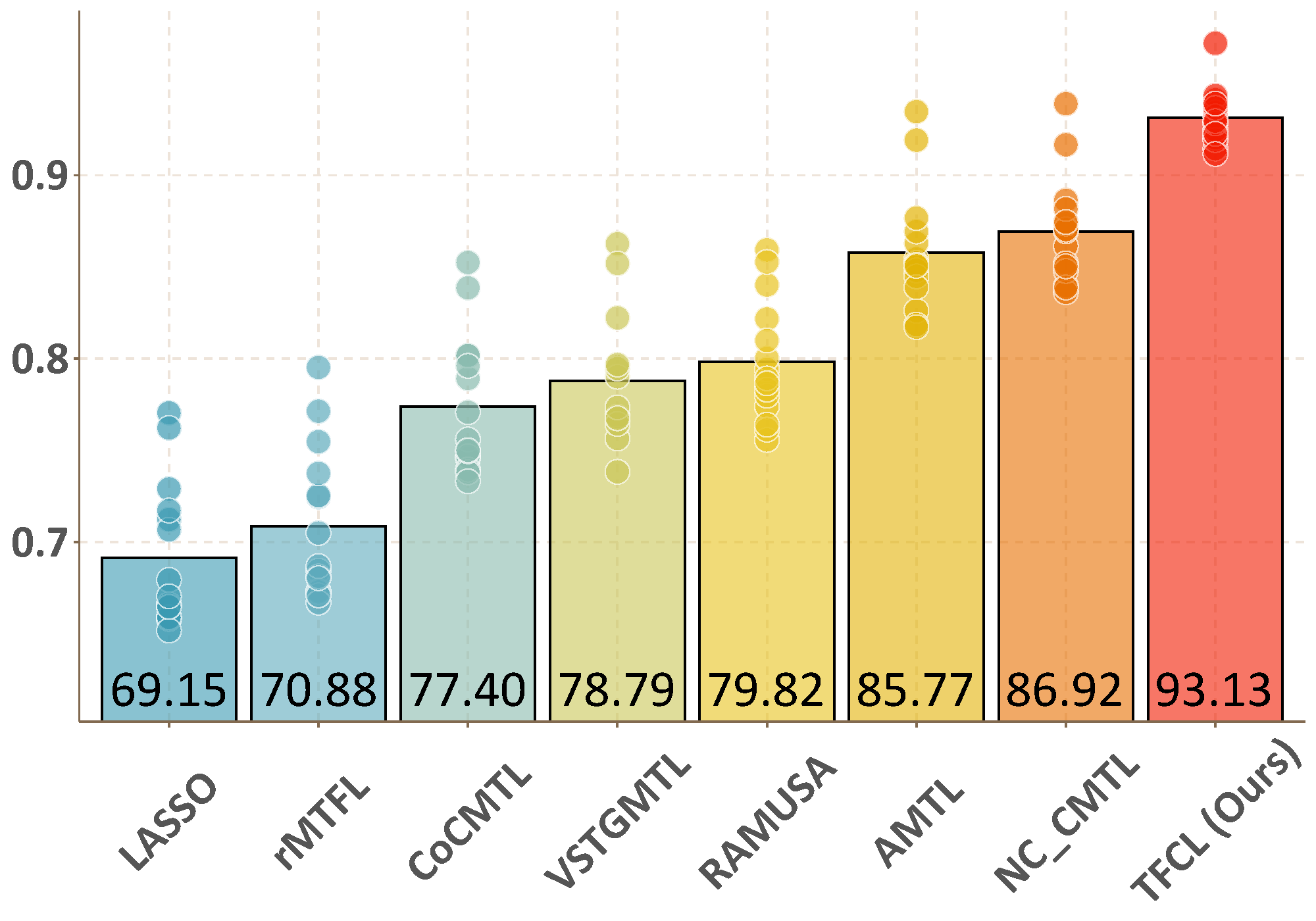} 
    } 
  \subfigure[Shoes CM]{\label{fig:perf_shoes_comf} 
      \includegraphics[width=0.3\textwidth]{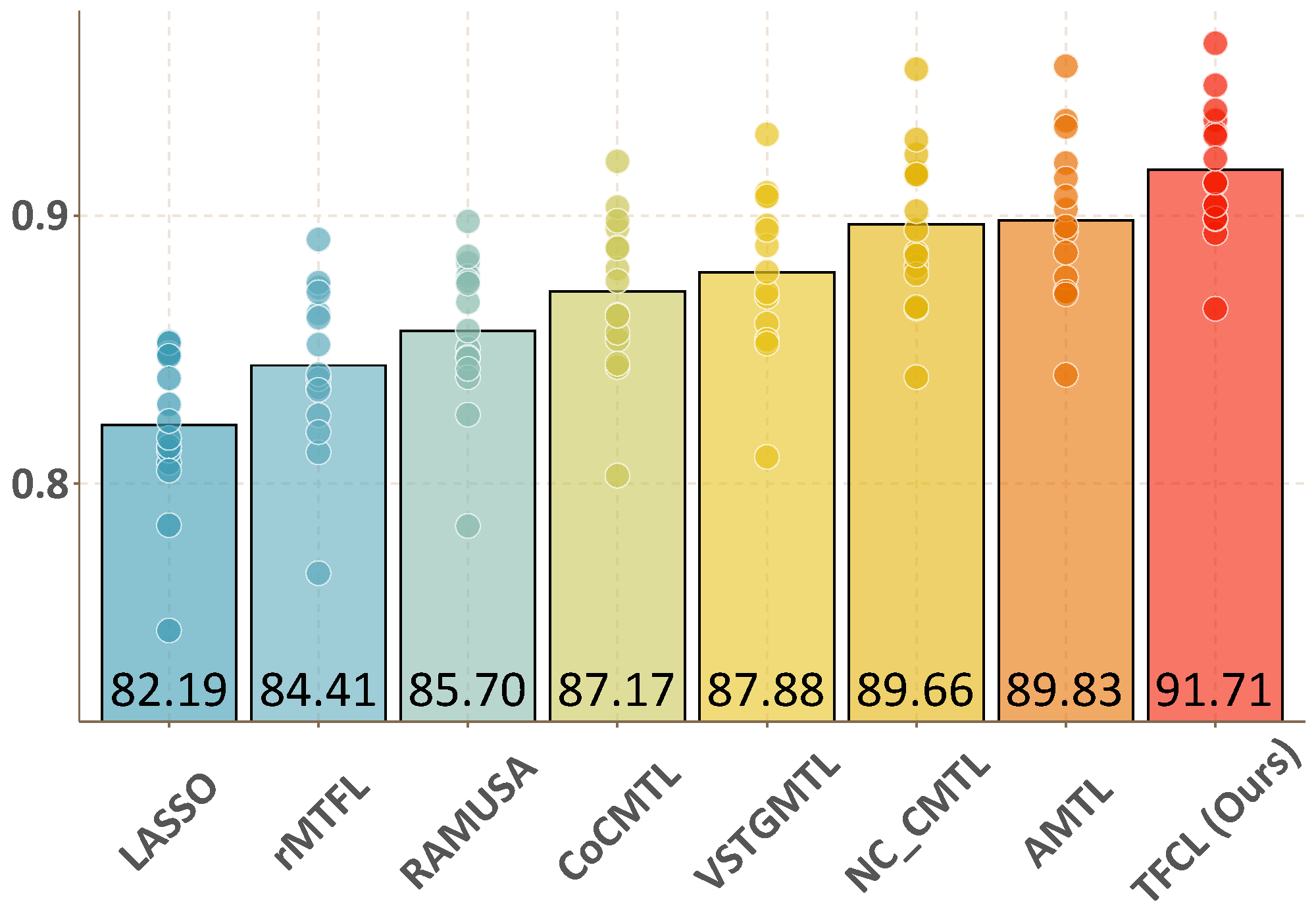} 
    } 
  \subfigure[Shoes FA]{\label{fig:perf_shoes_fash} 
      \includegraphics[width=0.3\textwidth]{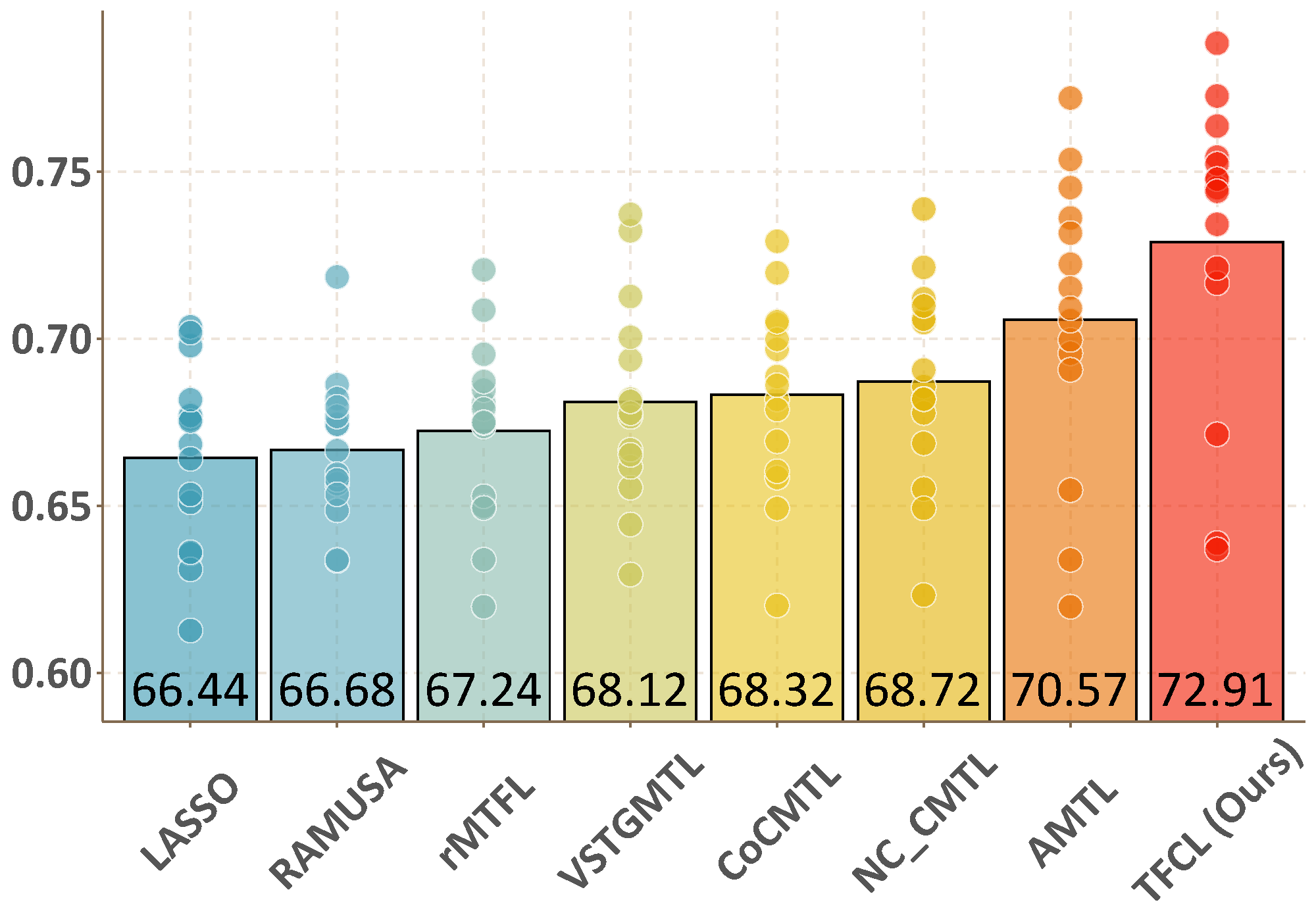} 
    } 

  \subfigure[Shoes FM]{\label{fig:perf_shoes_fm} 
      \includegraphics[width=0.3\textwidth]{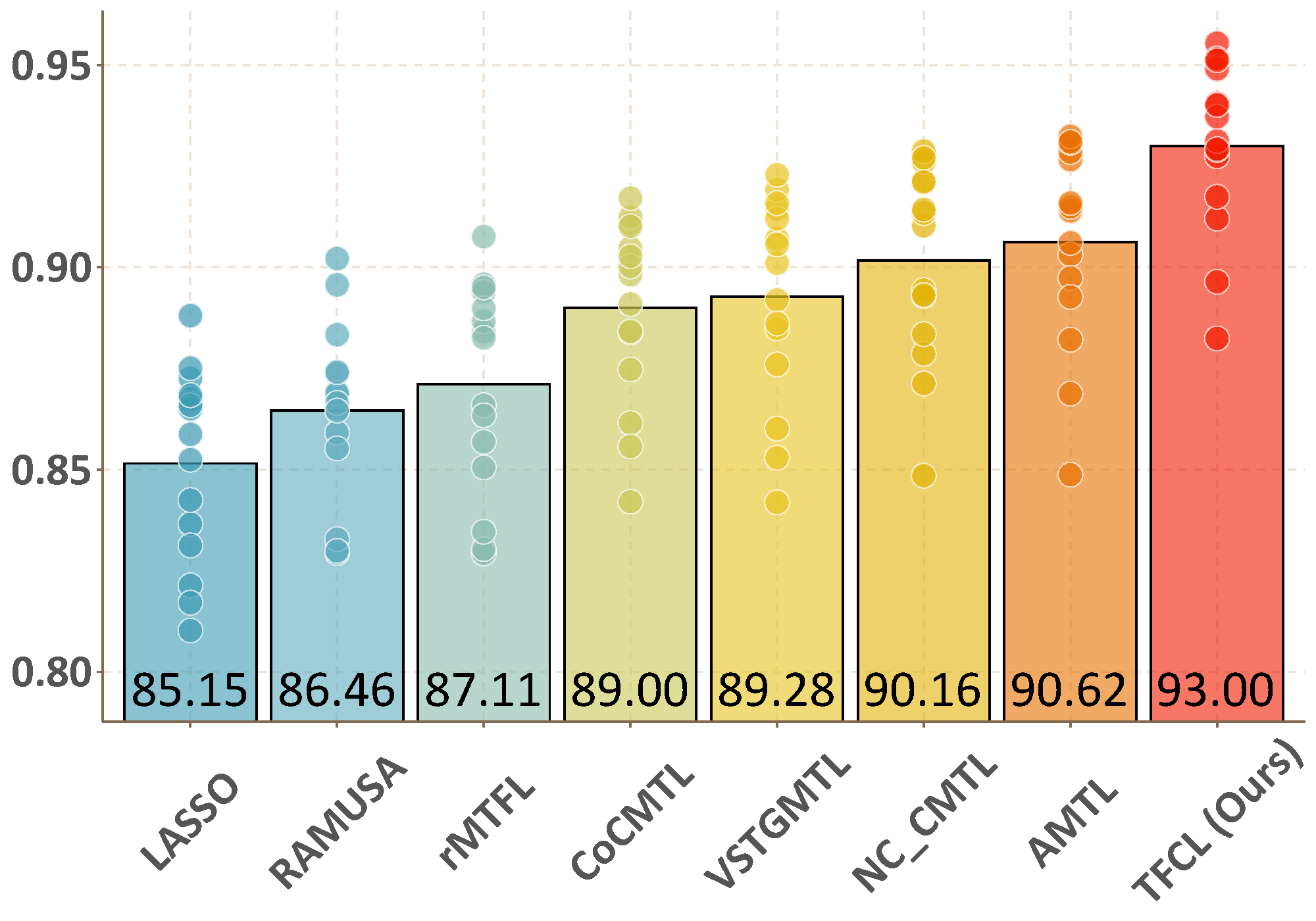} 
    } 
  \subfigure[Shoes OP]{\label{fig:perf_shoes_op} 
      \includegraphics[width=0.3\textwidth]{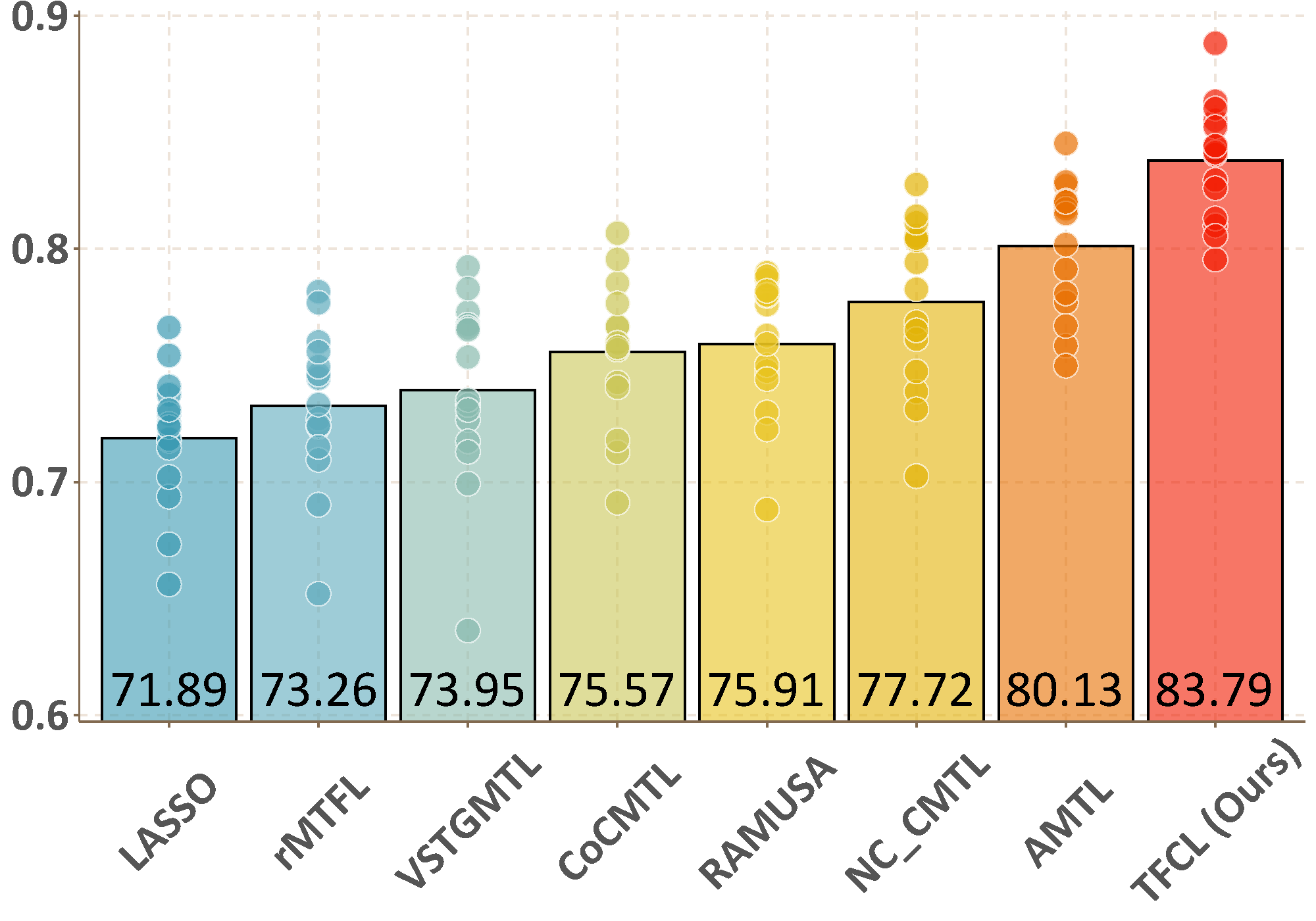} 
    } 
  \subfigure[Shoes OR]{\label{fig:perf_shoes_br} 
      \includegraphics[width=0.3\textwidth]{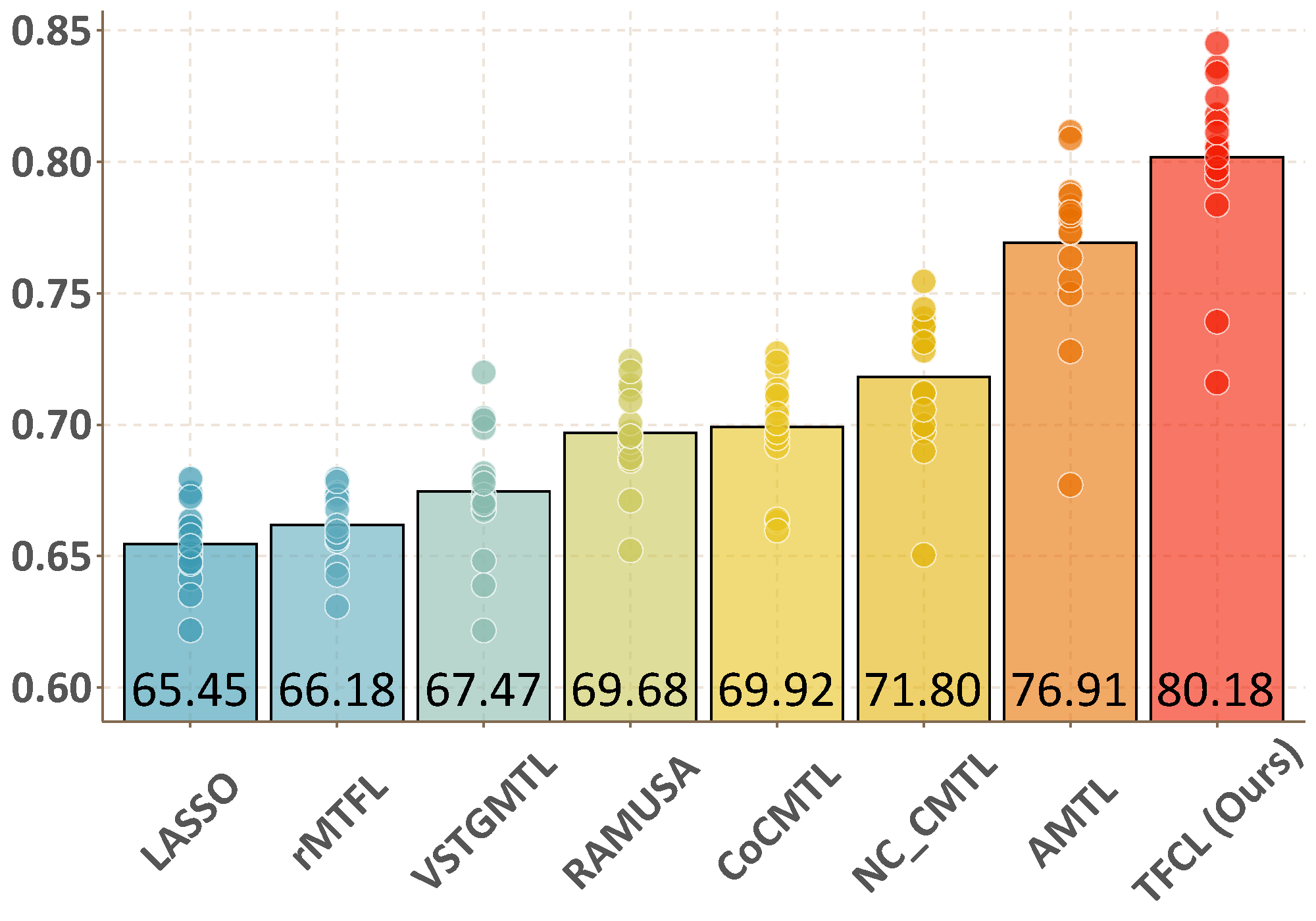} 
    } 
  
  \subfigure[Shoes PT]{\label{fig:perf_shoes_pt} 
      \includegraphics[width=0.3\textwidth]{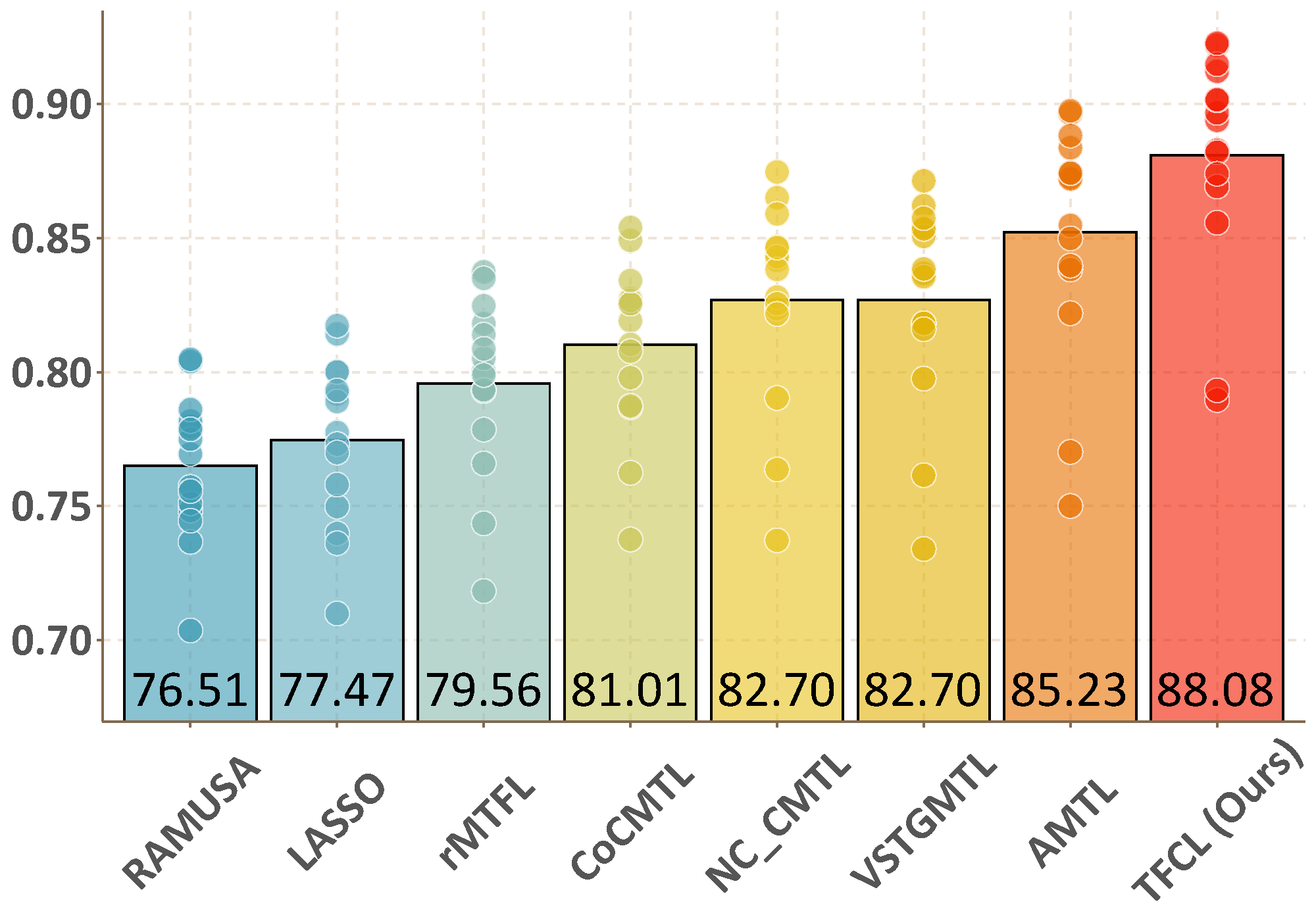} 
    } 
\subfigure[SUN CL]{\label{fig:perf_sun_cl} 
      \includegraphics[width=0.3\textwidth]{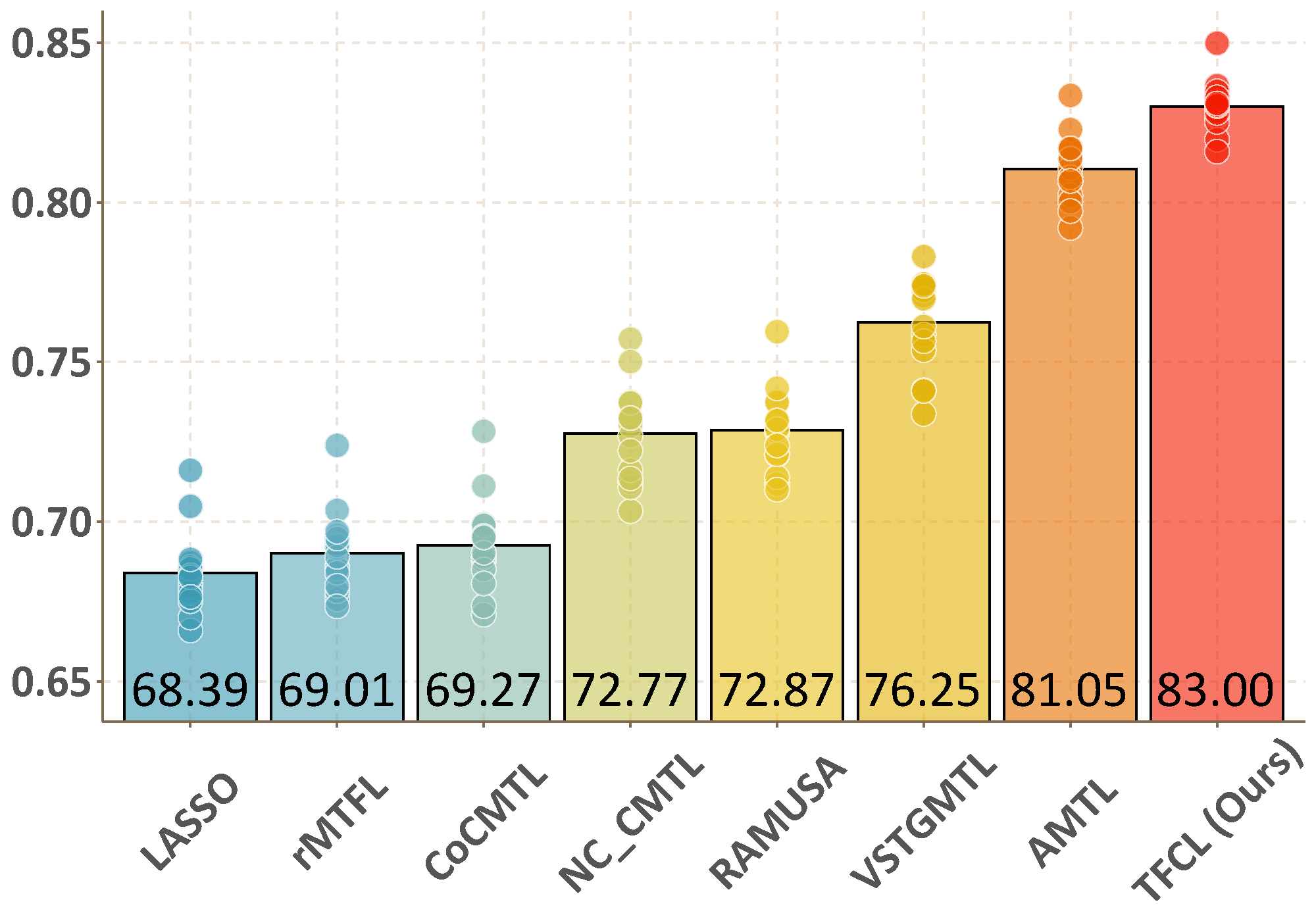}  
    }
\subfigure[SUN MO]{\label{fig:perf_sun_mo} 
      \includegraphics[width=0.3\textwidth]{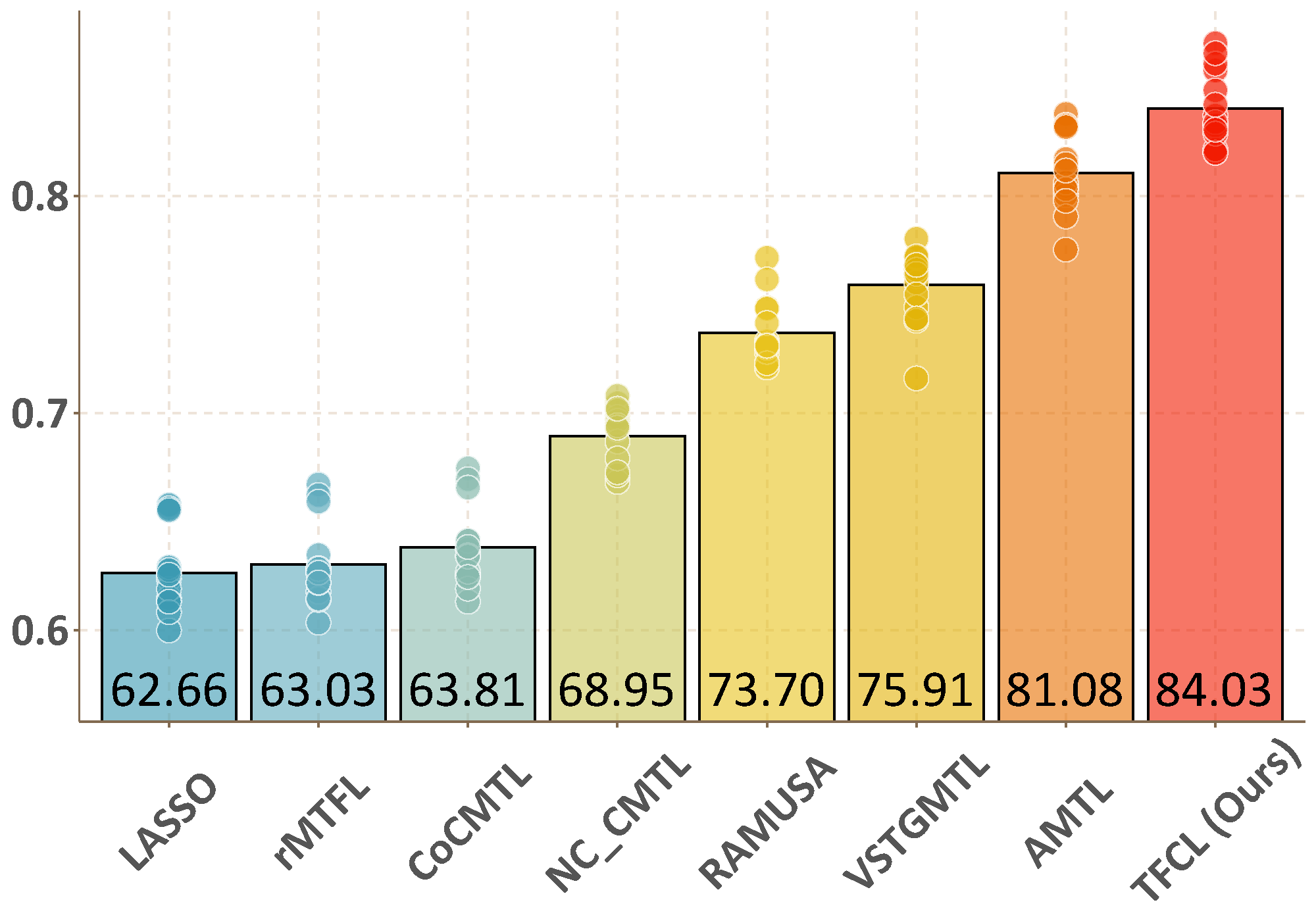}  
    }

\subfigure[SUN OA]{\label{fig:perf_sun_oa} 
      \includegraphics[width=0.3\textwidth]{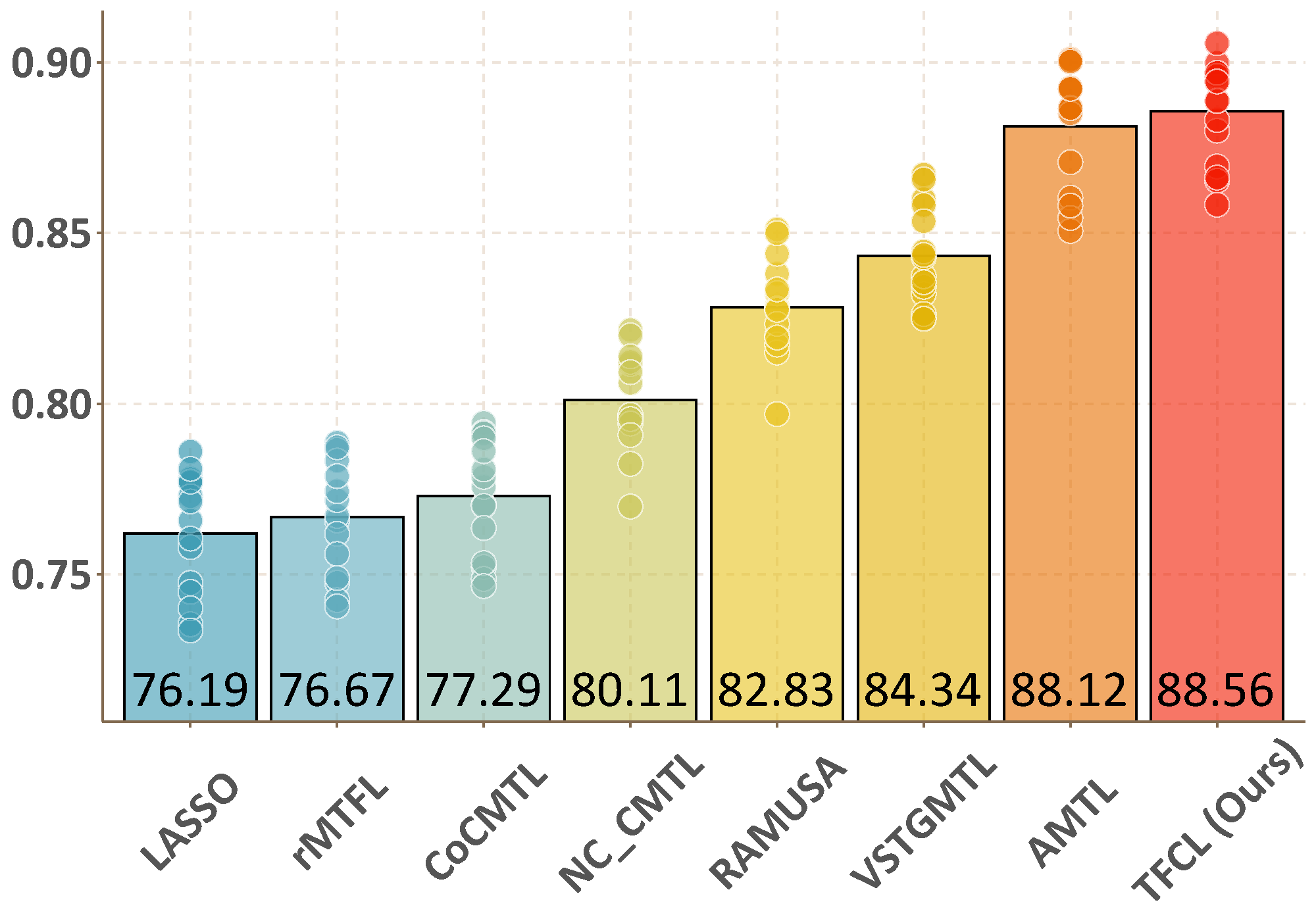}  
    }
\subfigure[SUN RU]{\label{fig:perf_sun_ru} 
      \includegraphics[width=0.3\textwidth]{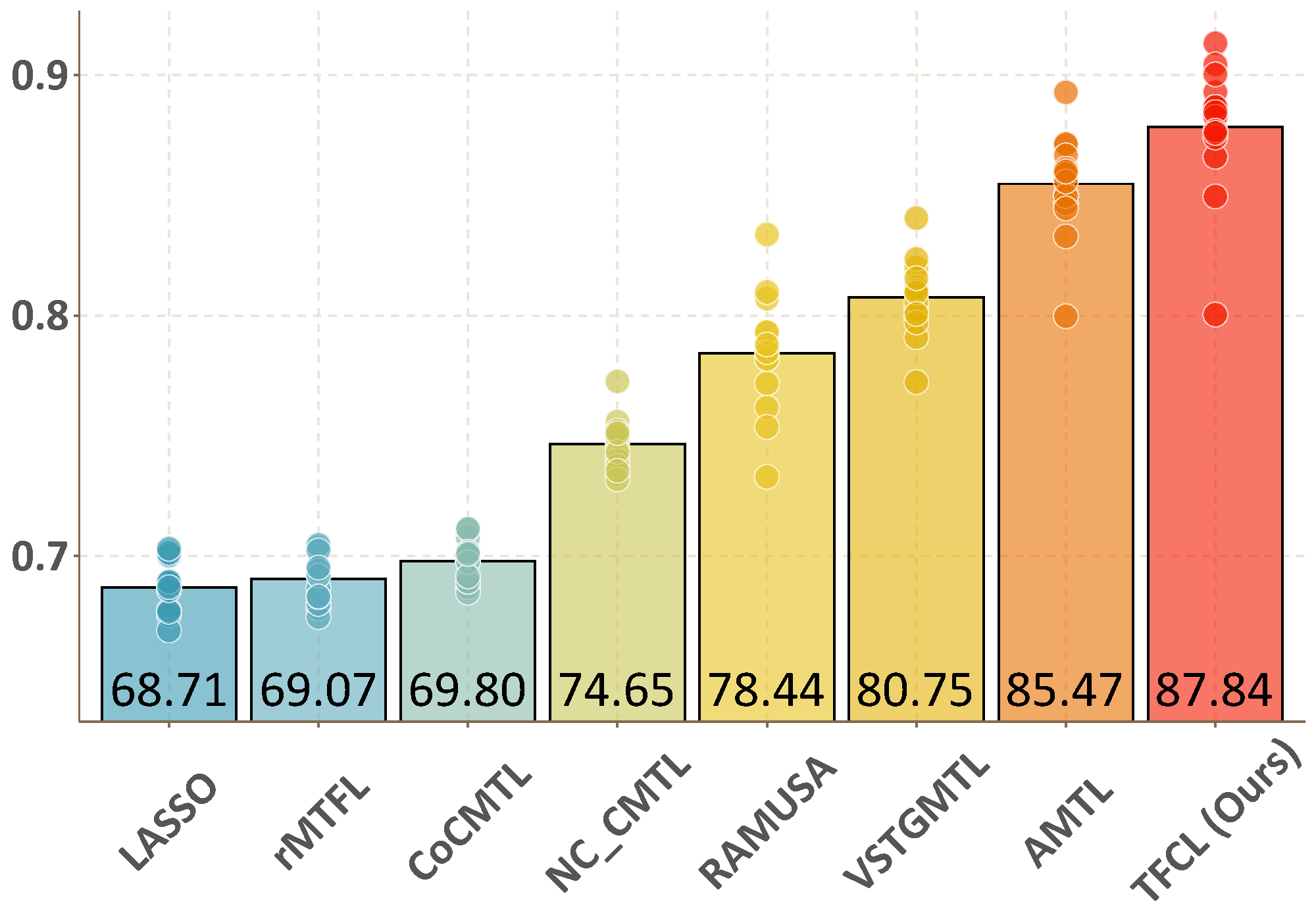}  
    }
\subfigure[SUN SO]{\label{fig:perf_sun_sm} 
      \includegraphics[width=0.3\textwidth]{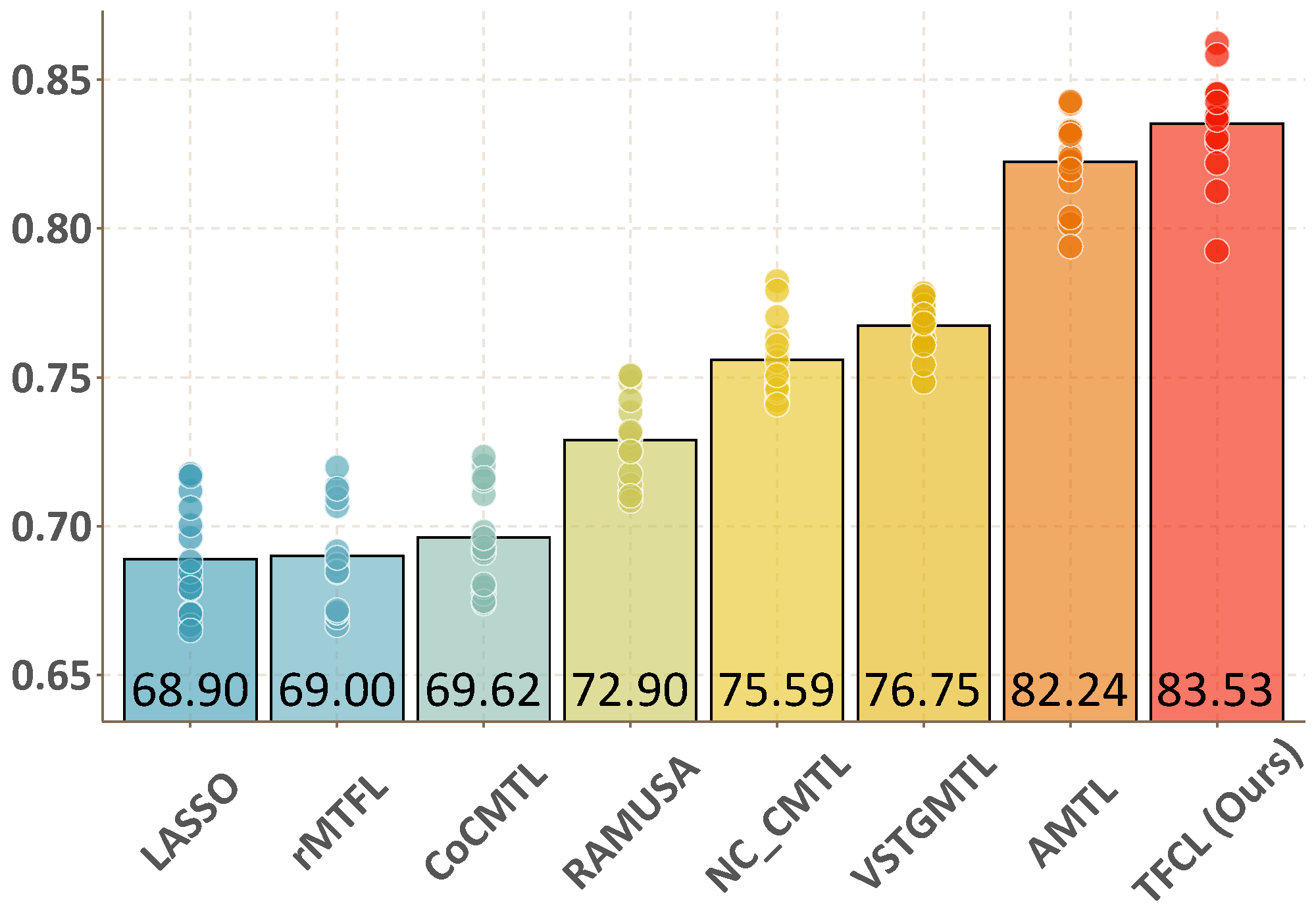}  
    }  
  \end{center}
    \caption{ \label{fig:bar}\textbf{Overall AUC comparisons with Boxplot.}  Here the scatters represent all the results coming from all the attributes for  Shoes Dataset and  Sun Dataset each with 15 repetitions. To show the statistical trends, we plot boxplots for the two datasets, respectively. Here the scatters are the 15 repetitions over the data splits, while the height of the bar represents the mean performance. }
\end{figure*}

\begin{figure}[h]\label{fig:ablation}
  \centering
     \subfigure[Shoes BR]{\label{fig:ab:br}
      \includegraphics[width=0.3\columnwidth]{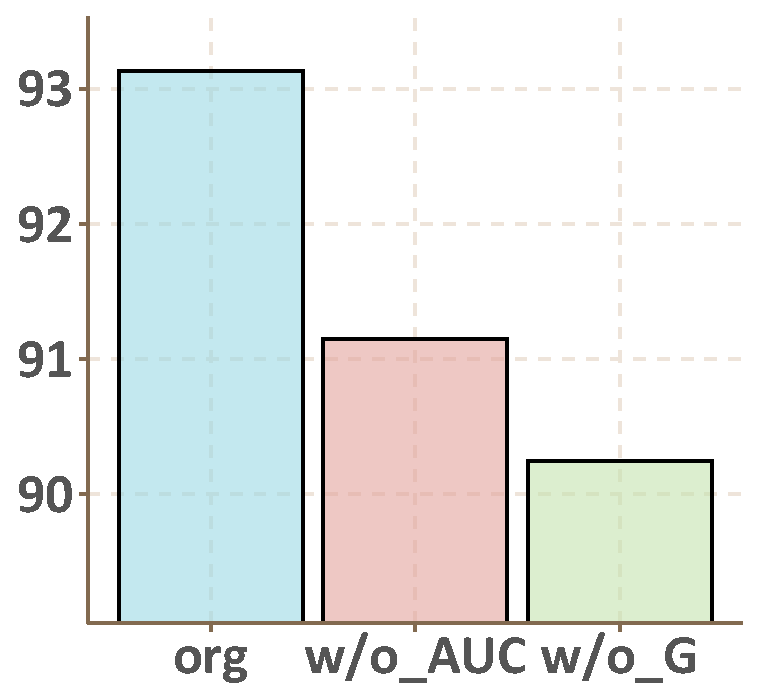}  
     }
     \subfigure[Shoes CM]{
      \includegraphics[width=0.3\columnwidth]{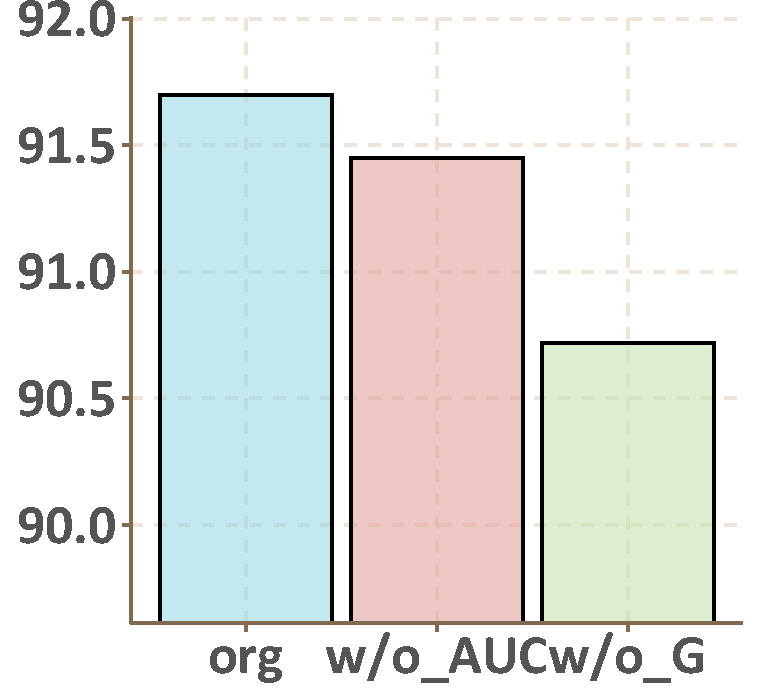}  
     } 
     \subfigure[Shoes FA]{
      \includegraphics[width=0.3\columnwidth]{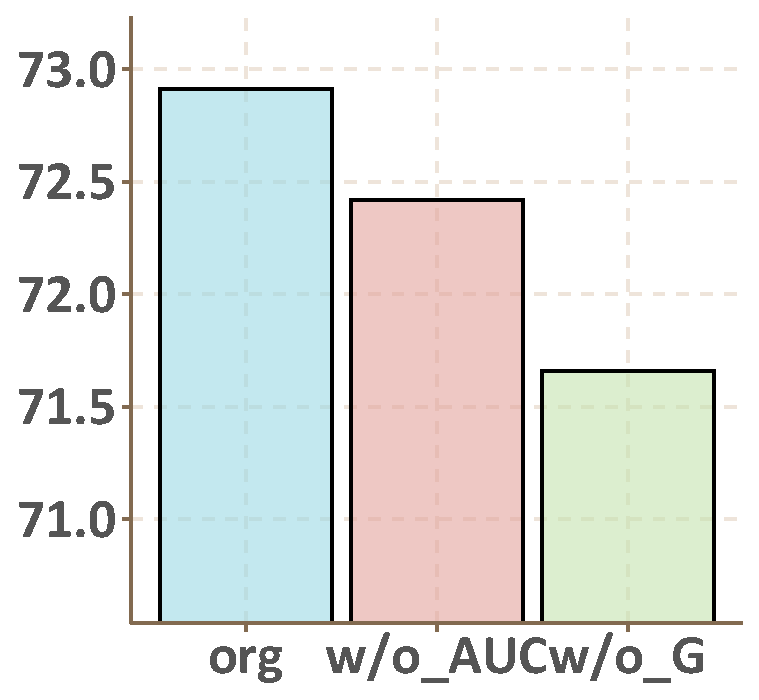}  
     }

    \subfigure[Shoes FM]{
      \includegraphics[width=0.3\columnwidth]{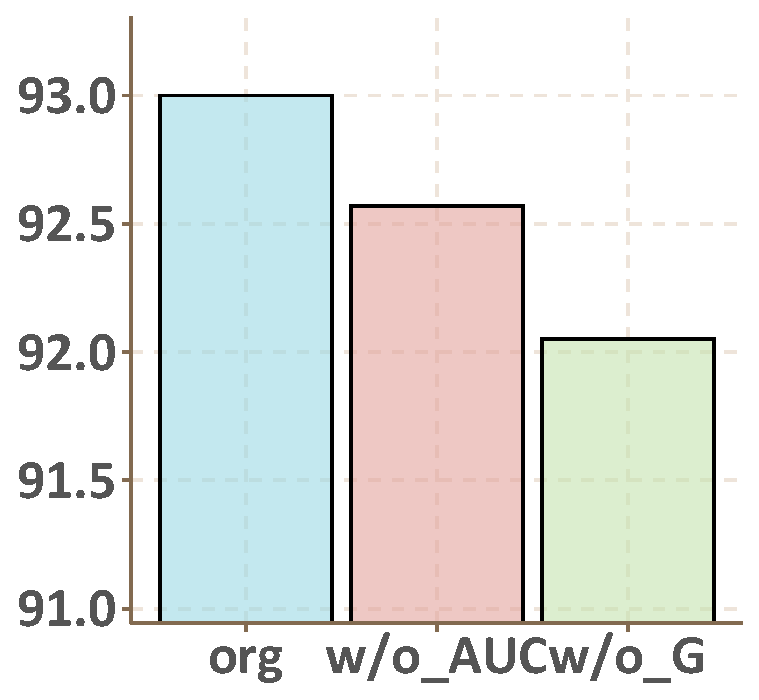}  
     }
     \subfigure[Shoes OP]{
      \includegraphics[width=0.3\columnwidth]{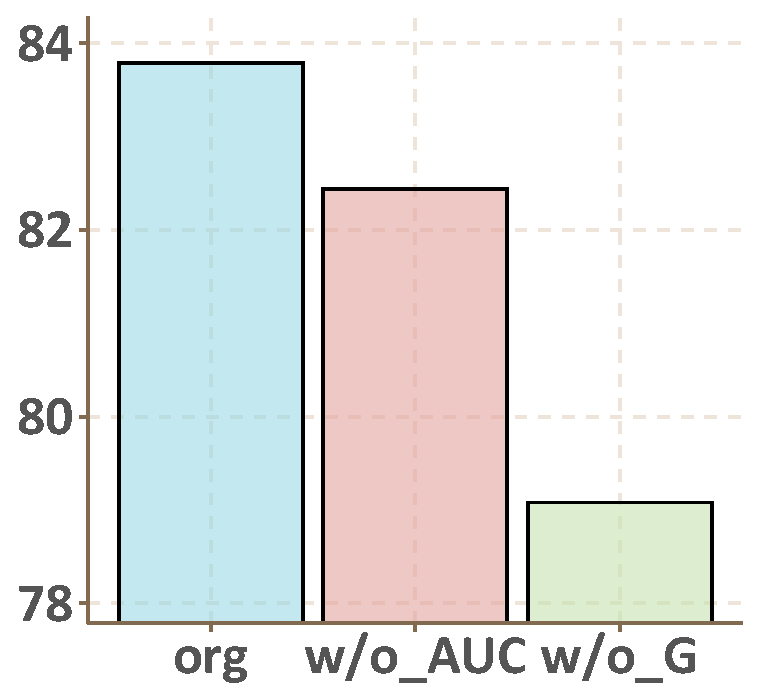}  
     }
     \subfigure[Shoes OR]{
      \includegraphics[width=0.3\columnwidth]{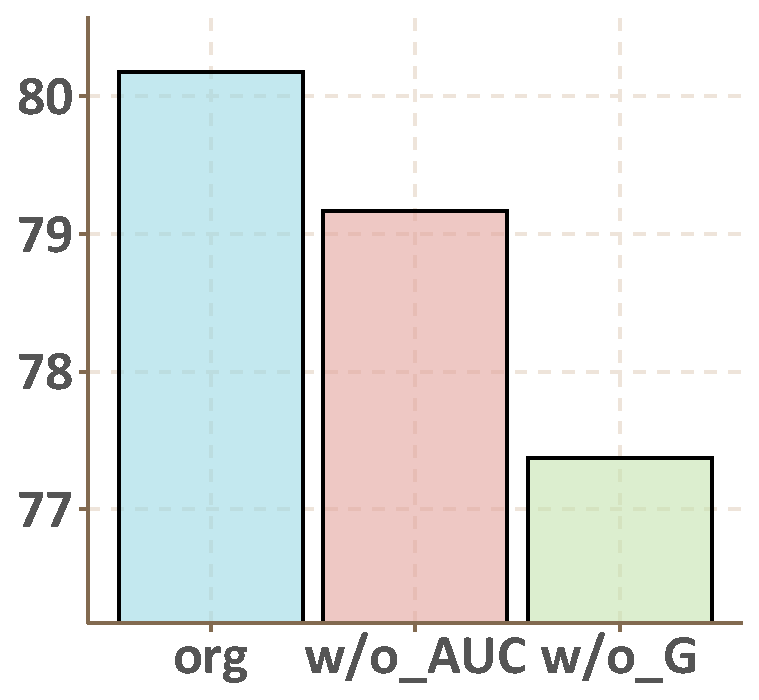}  
     }

    \subfigure[Shoes PT]{{\label{fig:ab:PT}}
      \includegraphics[width=0.3\columnwidth]{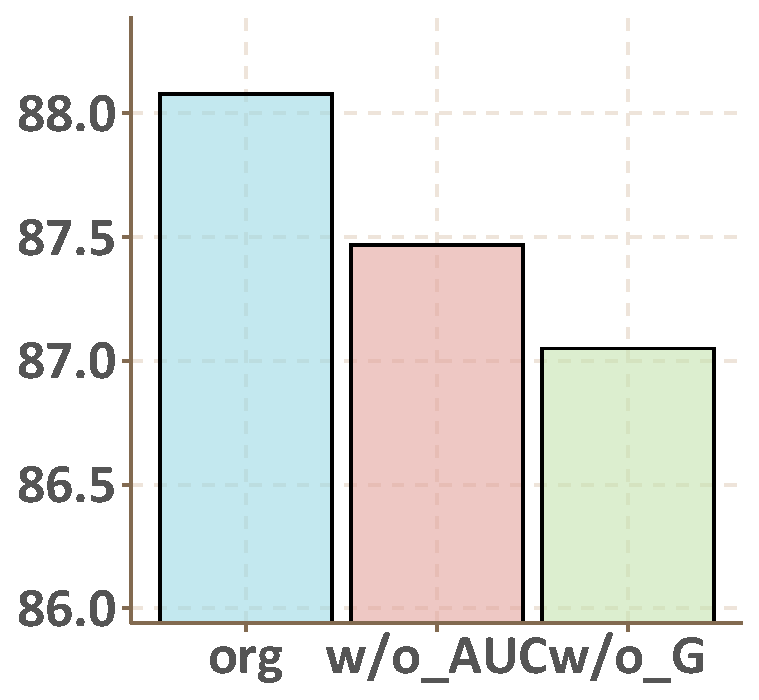}  
     }
     \subfigure[Sun CL]{\label{fig:ab:CL}
      \includegraphics[width=0.3\columnwidth]{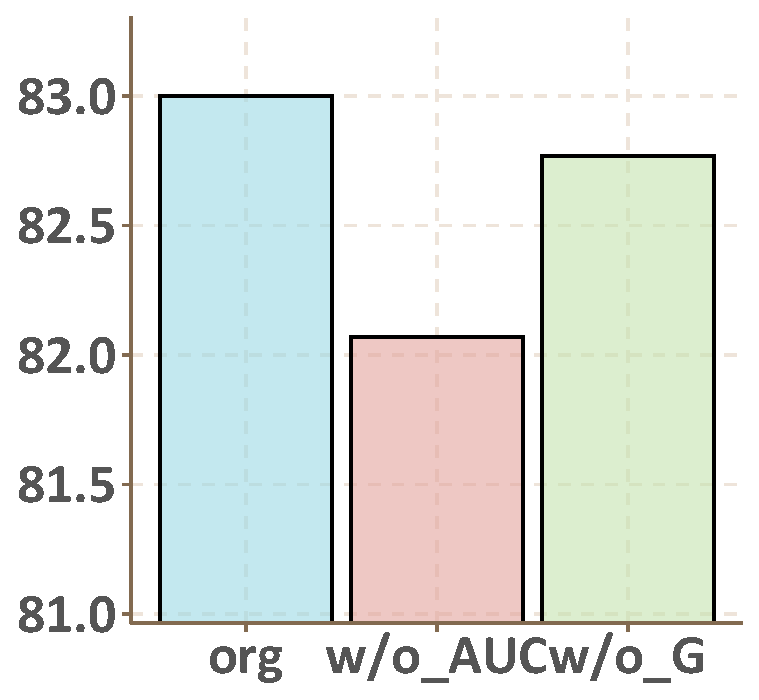}  
     }
     \subfigure[Sun MO]{
      \includegraphics[width=0.3\columnwidth]{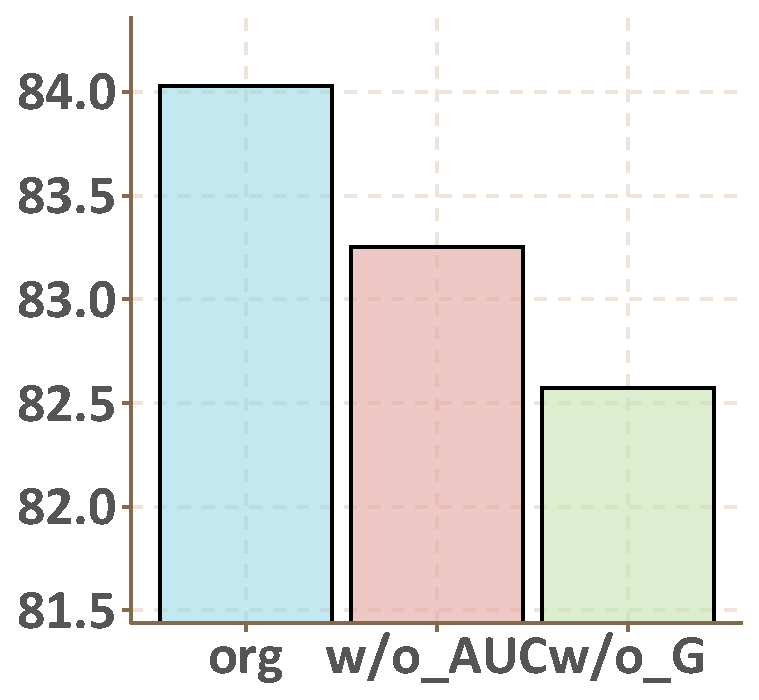}  
     }

   \subfigure[Shoes OA]{
      \includegraphics[width=0.3\columnwidth]{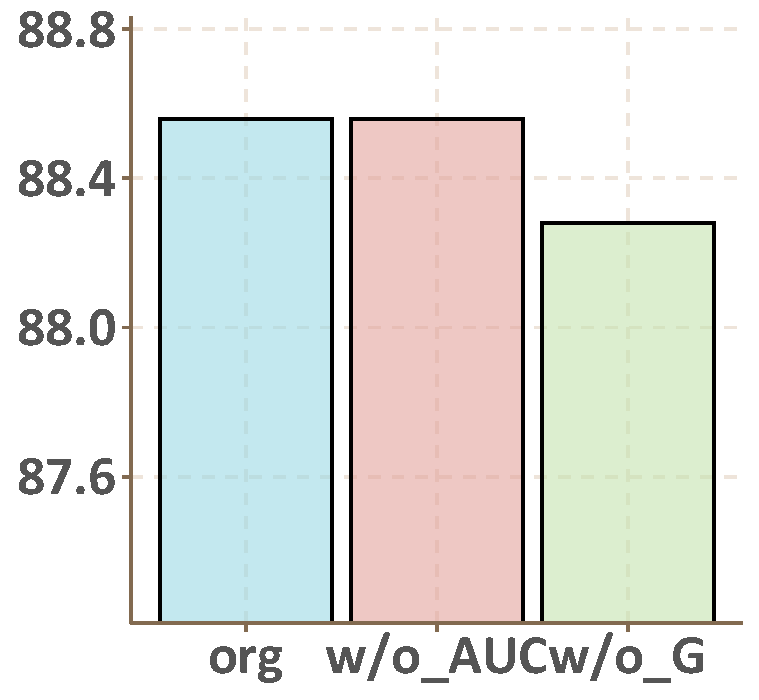}  
     }
     \subfigure[Sun RU]{
      \includegraphics[width=0.3\columnwidth]{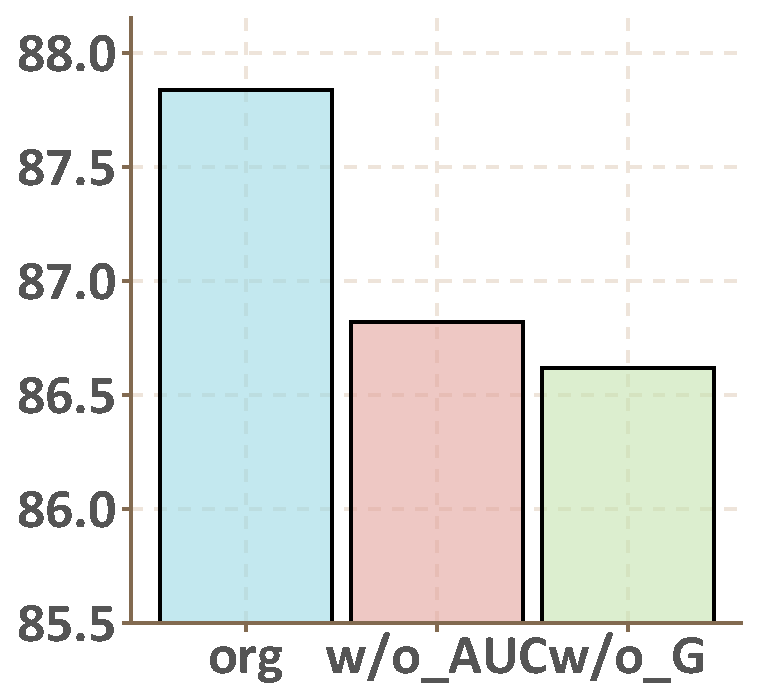}  
     }
     \subfigure[Sun SO]{\label{fig:ab:SO}
      \includegraphics[width=0.3\columnwidth]{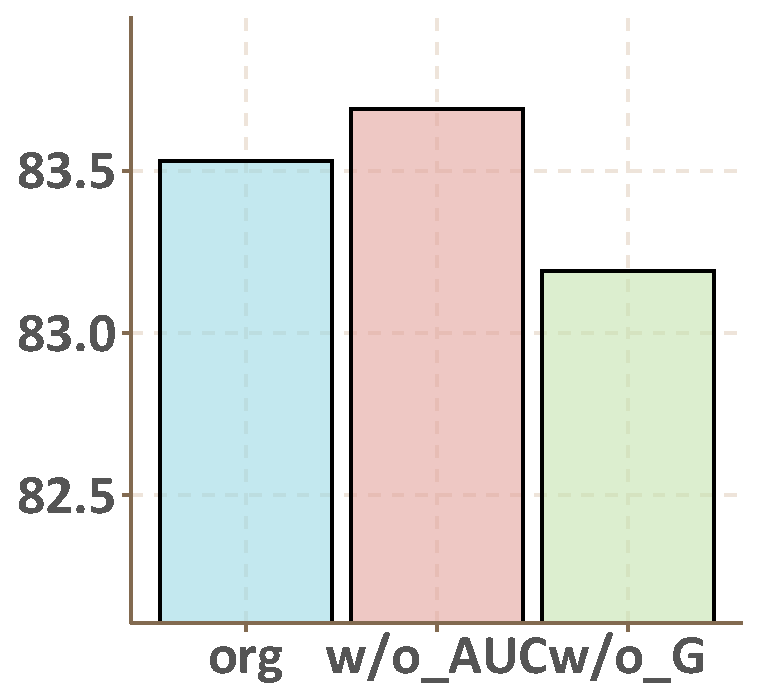}  
     }

  \caption{\textbf{Ablation Results (\uppercase\expandafter{\romannumeral 1})} The $y$-axis represents the average AUC score on the test set, and the $x$-axis represents different algorithms: \textbf{Org} shows the performance of our original TFCL algorithm; \textbf{w/o\_AUC} shows the performance of our algorithm when the AUC loss is replaced with the squared loss; \textbf{w/o\_G} shows the performance when our proposed co-grouping factor is removed from the model.  }
  \label{fig:ab1}

\end{figure}

\begin{figure}[h]
  \centering
     \subfigure[Shoes BR]{\label{fig:ab2:br}
      \includegraphics[width=0.3\columnwidth]{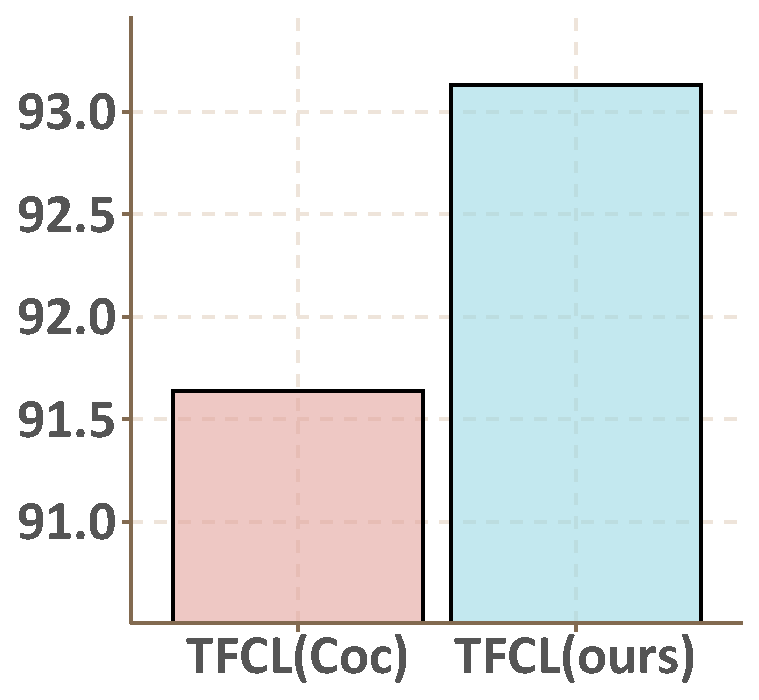}  
     }
     \subfigure[Shoes CM]{
      \includegraphics[width=0.3\columnwidth]{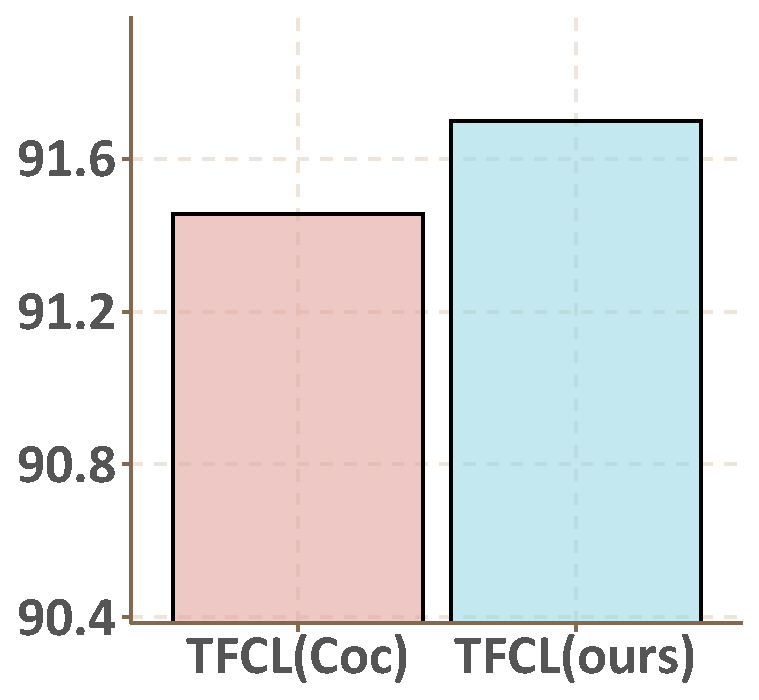}  
     } 
     \subfigure[Shoes FA]{
      \includegraphics[width=0.3\columnwidth]{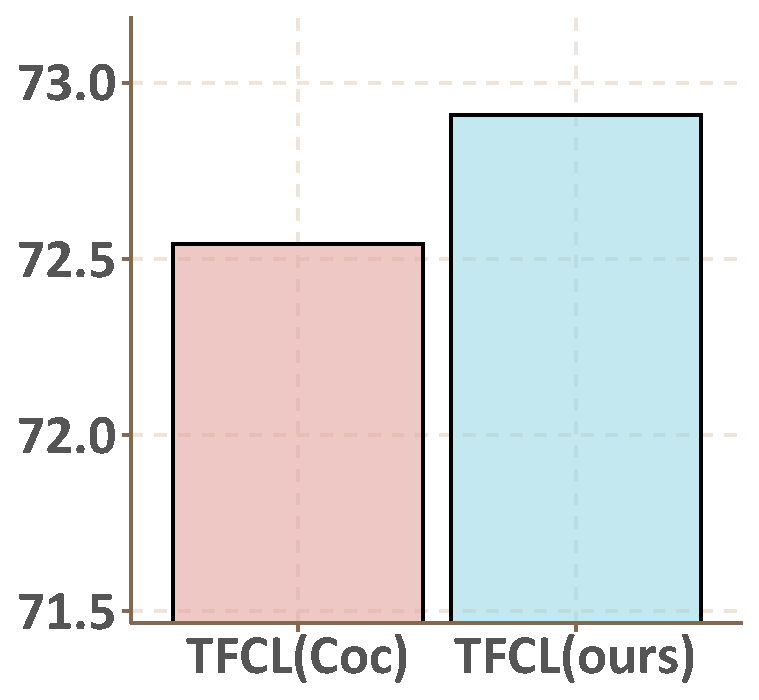}  
     }

    \subfigure[Shoes FM]{
      \includegraphics[width=0.3\columnwidth]{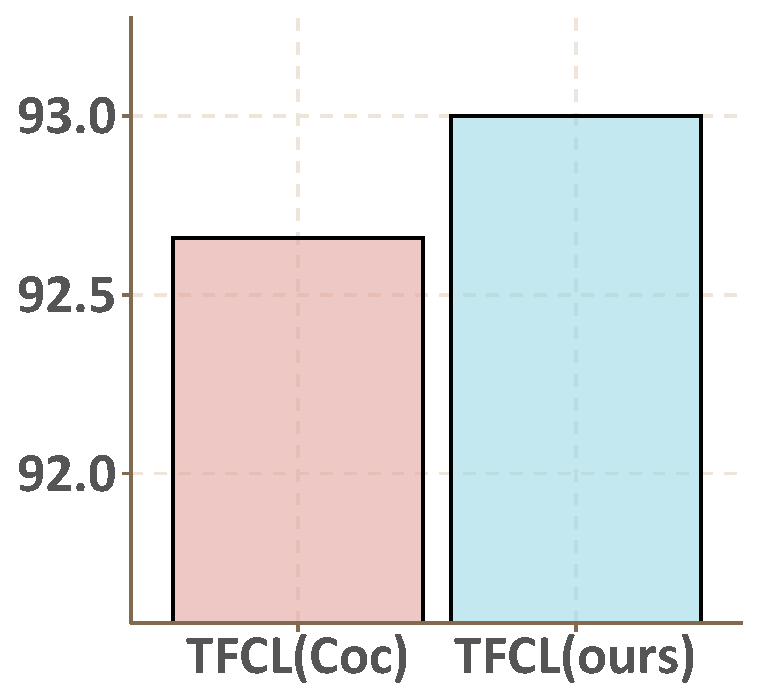}  
     }
     \subfigure[Shoes OP]{
      \includegraphics[width=0.3\columnwidth]{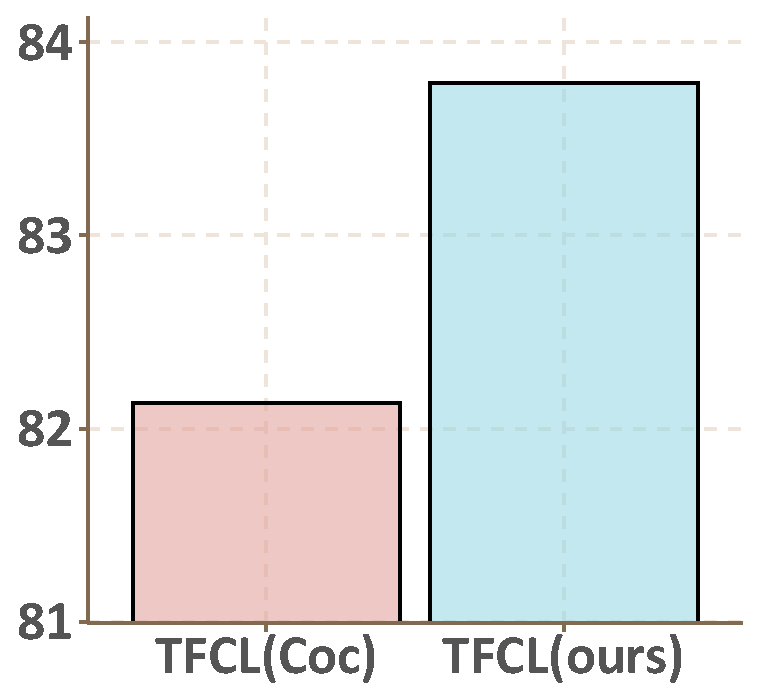}  
     }
     \subfigure[Shoes OR]{
      \includegraphics[width=0.3\columnwidth]{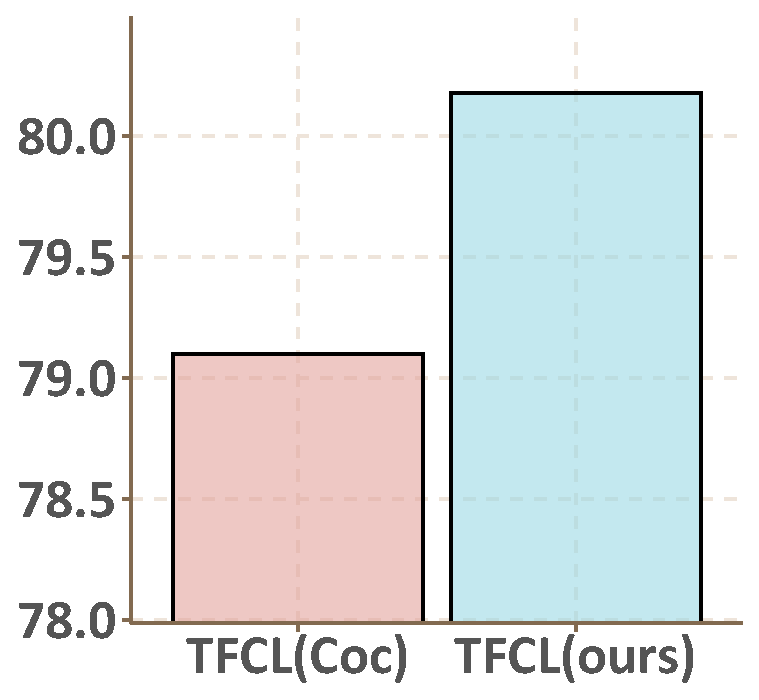}  
     }

    \subfigure[Shoes PT]{{\label{fig:ab2:PT}}
      \includegraphics[width=0.3\columnwidth]{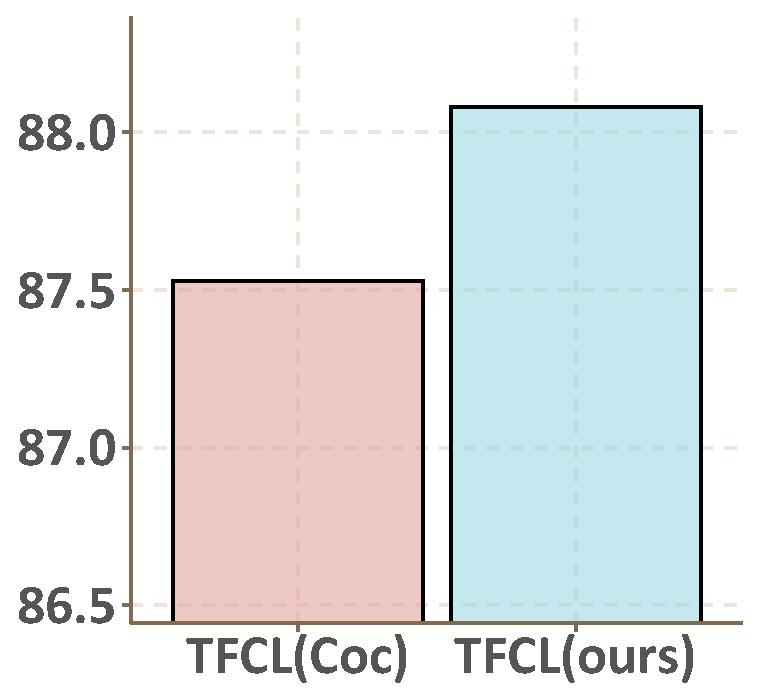}  
     }
     \subfigure[Sun CL]{\label{fig:ab2:CL}
      \includegraphics[width=0.3\columnwidth]{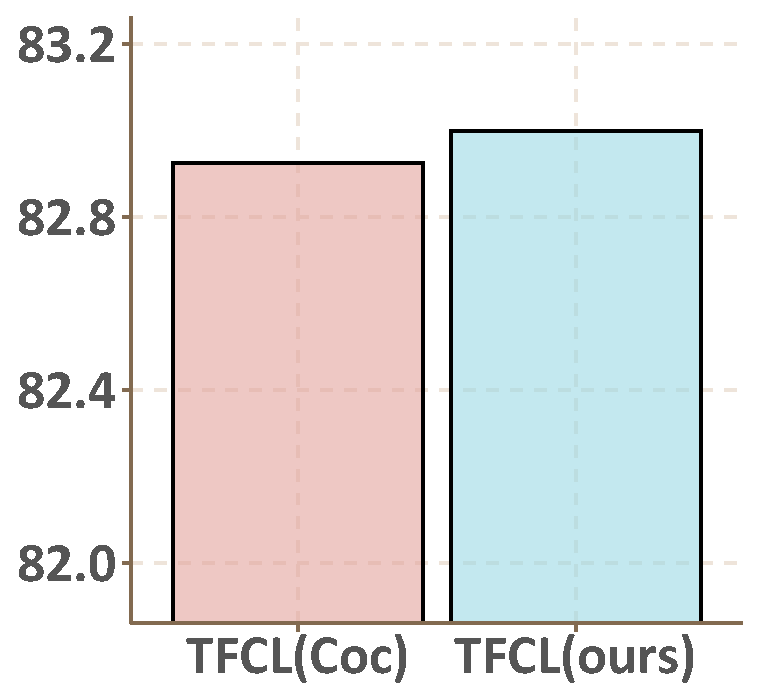}  
     }
     \subfigure[Sun MO]{
      \includegraphics[width=0.3\columnwidth]{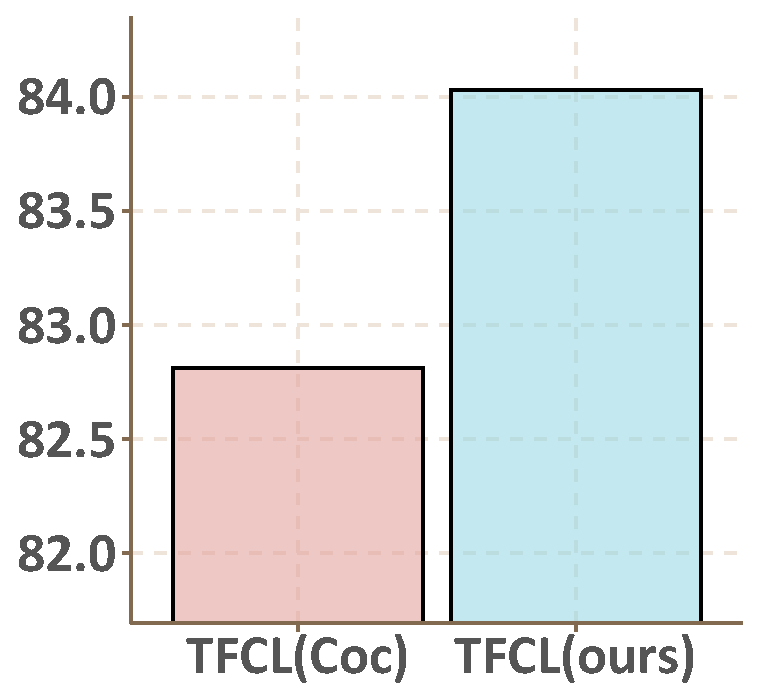}  
     }

   \subfigure[Shoes OA]{
      \includegraphics[width=0.3\columnwidth]{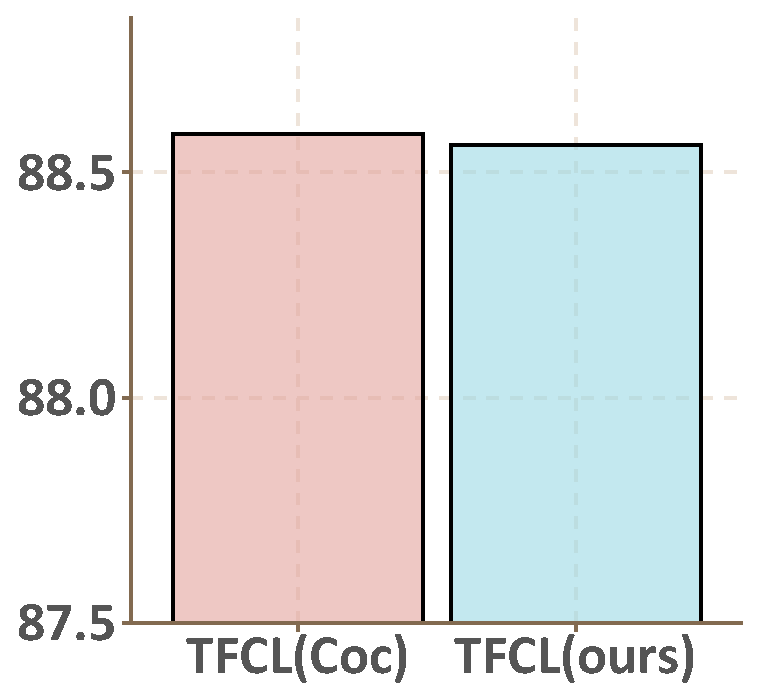}  
     }
     \subfigure[Sun RU]{
      \includegraphics[width=0.3\columnwidth]{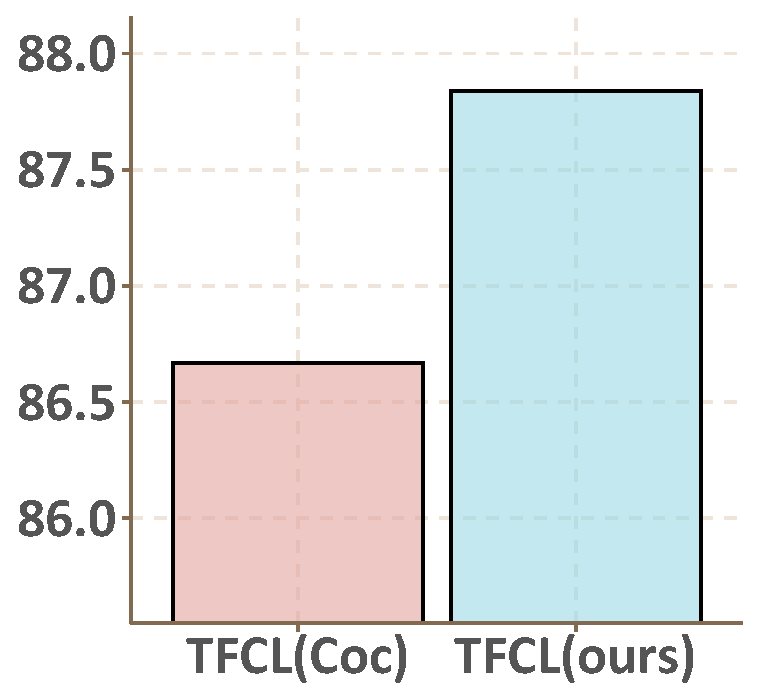}  
     }
     \subfigure[Sun SO]{\label{fig:ab2:SO}
      \includegraphics[width=0.3\columnwidth]{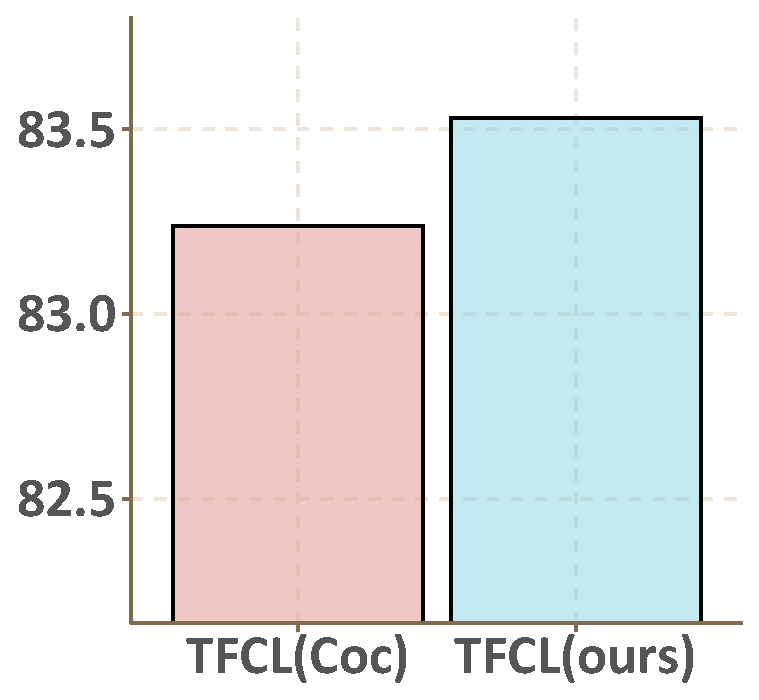}  
     }
  \caption{{\textbf{Ablation Results (\uppercase\expandafter{\romannumeral 2})} The $y$-axis represents the average AUC score on the test set, and the $x$-axis represents different algorithms: TFCL(Coc) shows the performance of TFCL algorithm with our co-grouping regularizer replaced by the corresponding regularizer in CocMTL; TFCL(ours)  shows the performance of our original algorithm.}  }
  \label{fig:ab2}
\end{figure}

\begin{figure}[!h]
\centering
    \subfigure[Shoes Brown]{\label{fig:dist_shoes_brown} 
      \includegraphics[width=0.45\columnwidth]{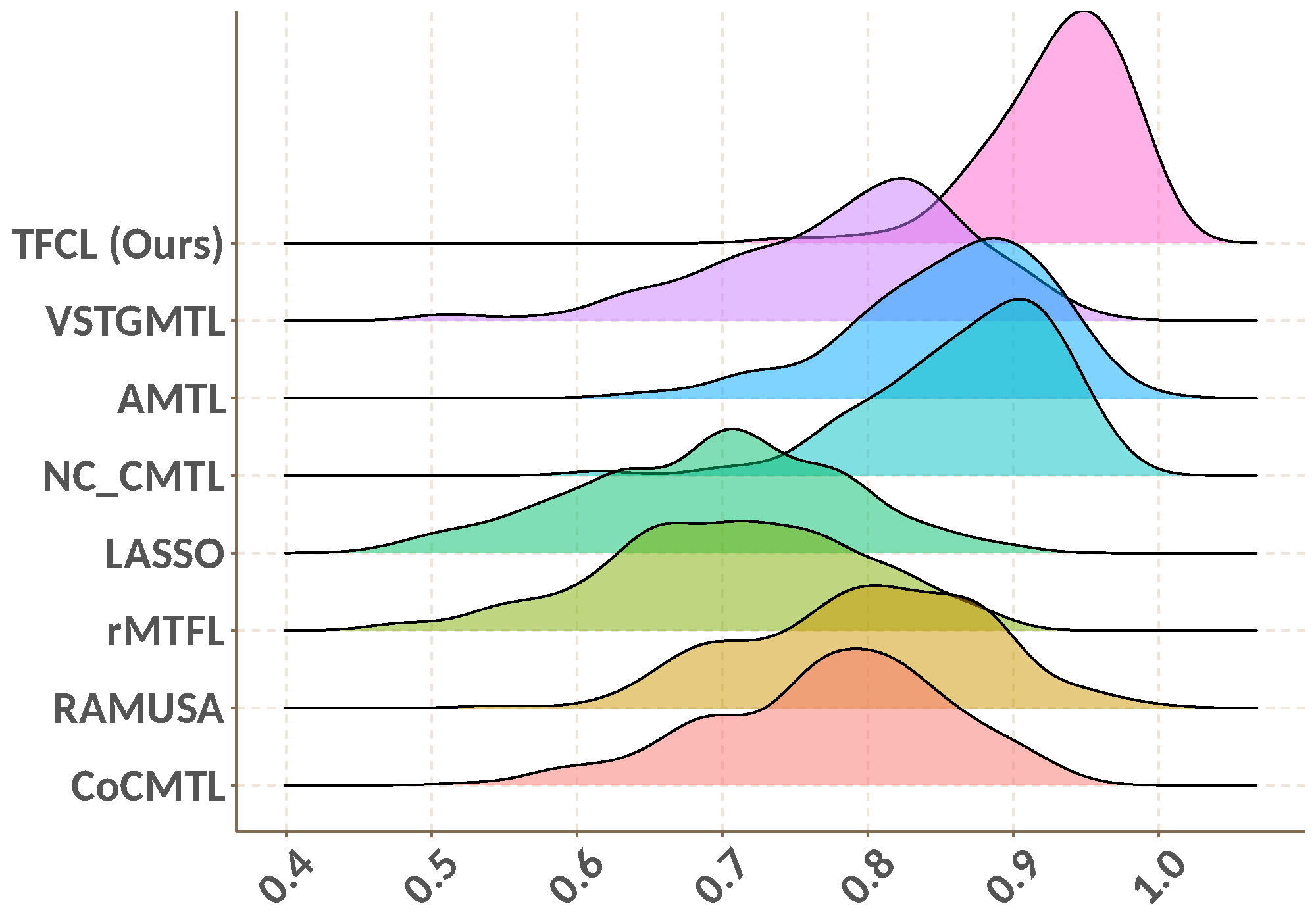} 
    } 
\subfigure[SUN Open Area]{\label{fig:dist_sun_oa} 
      \includegraphics[width=0.45\columnwidth]{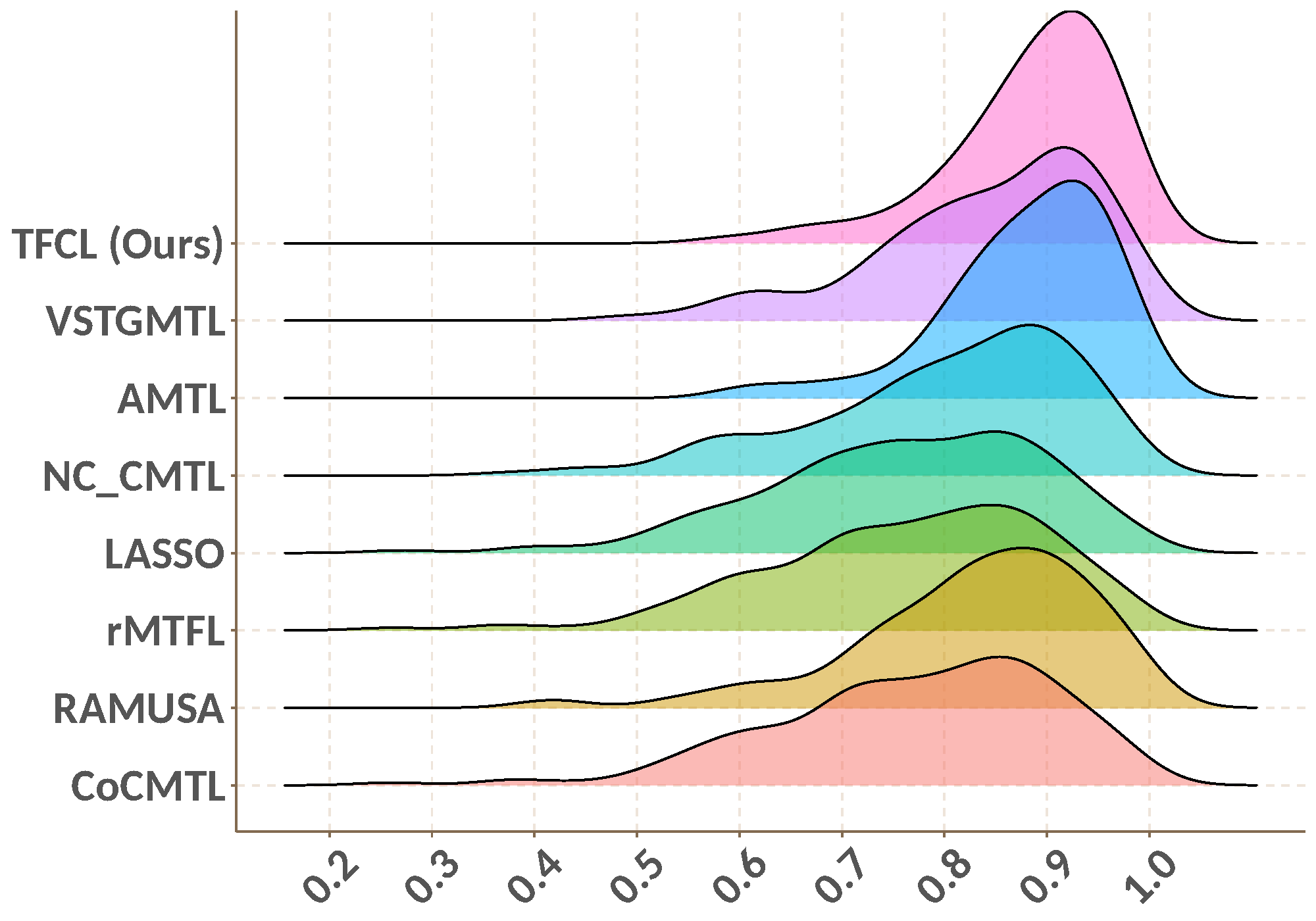}  
    }

    \caption{\textbf{Fine-grained comparison based on User AUC Score Distributions}. In this figure, we plot the user-specific performance distribution produced by all the involved algorithms for (a) the Brown attribute of shoes Dataset and (b) the Open Area attribute for Sun Dataset. In this figure, we investigate whether TFCL could benefit the performance distribution over users. The results show that TFCL tends to leverage more compact performance distribution. }
\end{figure}

\subsection{Simulated Dataset}\label{sec:sim}
To test the effectiveness of the basic TFCL framework, we generate a simple simulated annotation dataset with 100 simulated users, where the features and AUC scores are produced according to linear models with a block-diagonal task matrix. For each user, the 200 samples are generated as $\boldsymbol{X}^{(i)} \in \mathbb{R}^{200\times80}$ and $\boldsymbol{x}^{(i)}_k \sim \mathbb{N}(0,\boldsymbol{I}_{80})$. We generate a \textit{block-diagonal} task matrix $\boldsymbol{W}$ in a manner as the following.  Specifically, we create 5 blocks for $\boldsymbol{W}$, in a way that $\bm{W} = \bigoplus_{i=1}^5\boldsymbol{W}_i$ where $\boldsymbol{W}_1 \in \mathbb{R}^{20 \times 20}$, $\boldsymbol{W}_2 \in \mathbb{R}^{20 \times 20}$, $\boldsymbol{W}_3 \in \mathbb{R}^{10 \times 20}$, $\boldsymbol{W}_4 \in \mathbb{R}^{20 \times 20}$, $\boldsymbol{W}_5 \in \mathbb{R}^{10 \times 20}$. For each of the block, the elements are generated from the distribution $\mathbb{N}(C_i, 2.5^2)$ (generated via element-wise sampling) where $C_i \sim \mathbb{U}(0,K_i)$ is the centroid for the corresponding cluster, $K_1= 5, K_2 =5, K_3 =10, K_4 = 15, K_5=20$. For each user, the scoring functions are generated as $\boldsymbol{s}^{(i)}= \boldsymbol{X}^{(i)}(\boldsymbol{W}^{(i)}+\epsilon^{(i)})$, where $\epsilon^{(i)} \in \mathbb{R}^{200 \times 1}$, and $\epsilon^{(i)} \sim \mathbb{N}(0,0.1^2\boldsymbol{I}_{200})$. To generate the labels $\boldsymbol{Y}^{(i)}$ for each $i$, the top 50 instances with the highest scores are labeled as 1, while the remaining instances are labeled as -1.\\
\noindent{\textbf{Performance Comparison.}} The performance of all the involved algorithms on the simulated dataset is recorded in Fig.\ref{fig:box}. The corresponding results show that our proposed algorithm consistently outperforms other competitors. 
In particular, over the average results for 15 repetitions, our method achieves an AUC score of 93.66, while the second-best method obtains a score of 91.07. This leads to a 2.59 AUC improvement with respect to the second-best algorithm.\\
\noindent \textbf{Ablation Study.} Next, we carry out an ablation study to see how the single effect of AUC loss and the grouping factor work. Specifically, the corresponding results are shown in Tab.\ref{tab:abl}, where we compare our original performance with a baseline where the AUC loss is replaced with the instance-wise weighted squared loss. The results show that: (1) the original result outperforms this baseline, which suggests that adopting an AUC optimization is effective; (2) the baseline outperforms the other competitors, which suggests that the grouping factor is more effective than other competitors under the simulated dataset.\\
\noindent\textbf{Convergence analysis.} As shown in Fig.\ref{fig:conv:obj}-Fig.\ref{fig:conv:du}, the proposed method enjoys good convergence property on both the objective function and parameter sequence, which coincides with the theoretical results. \\
\noindent \textbf{Visualization of the Spectral Embeddings.} To show how the spectral embeddings evolve through the algorithm iterations, we plot the corresponding embeddings in the first five iterations in Fig.\ref{fig:embed}. The results suggest that spectral embeddings rapidly form stable and clear clusters after the second iteration. This fact validates our theoretical analysis concerning the grouping power of spectral embeddings. From Fig.\ref{fig:conv} and Fig.\ref{fig:embed}, one can find a close connection between the iteration curve and the evolution of the embedding space. To see this, recall the details of Alg.1, the spectral embeddings are optimized along with $\bm{V}$ in one of our subproblem, it thus contributes to the reduction of the loss function. Practically, this could be validated by Fig.\ref{fig:conv}  and Fig.\ref{fig:embed}. In Fig.\ref{fig:conv}, we see that the loss reduces fast and reaches convergence after 5 iterations. In Fig.\ref{fig:embed}, one can see that the embeddings also converge to their corresponding clusters within 5 iterations.        \\
\textbf{Structure Recovery.} Besides generalized performance, we could also empirically verify the ability of our algorithm to recover the expected structures on parameters $\boldsymbol{W}$. With the same simulated dataset, we compare the parameters $\boldsymbol{W}$ learned from the involved algorithms and the Ground Truth in Fig.\ref{fig:struc}. The results show that our proposed method could recover a much clearer structure than other competitors. Meanwhile, we see that all the competitors could roughly recover a block-diagonal outline. However, different methods suffer from different degrees of off-diagonal noises. This could be understood from an algebraic analysis. For linear models, the predictive function is defined as $\hat{\bm{y}}(\bm{X}) = \bm{X}\bm{W}$. Moreover, we note that the true parameters are generated by a linear function $\bm{X}\bm{W}^\star$. If $\bm{X}$ is not fully ranked, we have $\bm{X}\bm{W} = \bm{X}\bm{W}^\star$ whenever $\bm{W} = \bm{W}^\star + \bm{W}'$ and $\bm{W}' \in \mathsf{Null}(\bm{X})$. Generally, $\bm{W}'$ has off-diagonal elements and naturally leads to the observed noise in Fig.\ref{fig:struc}. Obviously, without a block-diagonal regularizer, it is hard to avoid $\bm{W}'$ even for a global optimal solution. Moreover, since $\bm{W}'$ is only related to the observed data $\bm{X}$, there is a risk of over-fitting, especially in our case when off-diagonal noise is not compatible with the simulated dataset. By eliminating the off-diagonal noise, our model shows a significant performance improvement shown in Fig.\ref{fig:box}.

\subsection{Shoes Dataset}\label{sec:shoes}
\textbf{Dataset Description.} The Shoes Dataset \cite{user2} is a popular attribute prediction benchmark, which consists of 14,658 online shopping shoe images with 7 attributes (BR: brown, CM: comfortable, FA: fashionable, FM: formal, OP: open, ON: ornate, PT: pointy). In this dataset, annotators with various knowledge are invited to judge whether a specific attribute is present in an image. Specifically, each user is randomly assigned with 50 images, and there are at least 190 users for each attribute who take part in the process, which results in a total volume of 90,000 annotations.\\
\noindent \textbf{Pre-processing.} For input features, we  adopt the GIST and color histogram provided in \cite{shoes} as input features.  Then we perform PCA to reduce the redundancy of these features before training. Meanwhile, we notice that users who extremely prefer to provide merely one class of labels may lead to large biases. To eliminate such effect, we manually remove users who give less than 8 annotations for the minority class.\\
\noindent \textbf{Performance comparison.}
The average performances of 15 repetitions for Shoes dataset are shown in the left side of Fig.\ref{fig:bar} where the scatters show the 15 observations over different dataset splits and the bar plots show the average performance over 15 repetitions. Then we could make the following observations: 1) Our proposed algorithm consistently outperforms all the competitors significantly for all the attributes on Shoes dataset. It is worth mentioning that, for Shoes dataset, our method outperforms the second best method by the AUC score of 6.22, 1.88, 2.34, 2.38, 3.66, 3.27, 2.85 in terms of \emph{BR}, \emph{CM}, \emph{FA}, \emph{FM}, \emph{OP}, \emph{OR}, \emph{PT}, respectively. 2) The models using low-rank constraints achieve much higher scores than those using sparsity constraints (LASSO and rMTFL). This is because there exist obvious correlations among users' annotations and the low-rank assumption could model these task (user) correlations much better. 3) AMTL outperforms the other low-rank constrained methods on all the datasets, as it explicitly models and reduces the influence of negative transfer via asymmetric learning. 4) The superiority of our method over the other feature-task correlation learning approaches (CoCMTL and VSTGMTL) justifies the TFCL framework. 5) Our proposed method outperforms AMTL over most attributes on both datasets. One possible reason is that our framework reasonably models the user annotation behaviors and thus is more suitable to the personalized attribute prediction problem. Moreover, our method provides extra effort to prevent the negative transfer across features and tasks.\\
\noindent \textbf{Ablation Study.} (\uppercase\expandafter{\romannumeral 1}) Our proposed algorithm has two major components: one is the grouping factor with its regularizer and the other is the surrogate AUC loss. Now we see how these two components contribute to the improvements respectively. Specifically, we remove these two parts respectively from our original model, and show the corresponding results in Fig.\ref{fig:ab:br}-Fig.\ref{fig:ab:PT}.  In these figures, \textbf{Org} shows the performance of our original TFCL algorithm, \textbf{w/o\_AUC} shows the performance when the AUC loss is replaced with the instance-wise squared loss, and \textbf{w/o\_G} shows the performance when our proposed co-grouping factor is removed from the model. From the experimental results, we have the following observations: 1) Our original algorithm outperforms both competitors, which shows that the joint effect of the two components is better than the single effects. 2) In most cases, we find that removing the grouping factor shows a more significant performance reduction than replacing the AUC loss. This implies that the grouping effect tends to have a more effective impact on performance.  (\uppercase\expandafter{\romannumeral 2}) Note the CocMTL algorithm also developed a co-grouping regularizer, we further record the performance when our regularizer is replaced with the corresponding term in CoCMTL. This is shown in Fig.\ref{fig:ab2:br}-Fig.\ref{fig:ab2:PT}. The results suggest that our proposed regularizer could outperform the co-grouping regularizer in CoCMTL.\\
\noindent \textbf{Fine-grained comparison.} After showing the effectiveness of our proposed method in coarse-grained comparison, we then visualize the personalized predictions to investigate fine-grained conclusions at user level. Take the \emph{brown} attribute on Shoes dataset as an example, we visualize the testing AUC score distributions over users for all the methods in Fig.\ref{fig:dist_shoes_brown}. As shown in this figure, our model achieves a higher mean and lower performance variance than the competitors. On the contrary, traditional approaches suffer from an obvious long-tail problem. The reason for such long-tail effect might be two-fold: 1) The traditional methods are more sensitive toward hard tasks, 2) the negative transfer issue is not sufficiently addressed. It is thus indicated that our method indeed promotes the collaborative learning to improve the performance of those hard tasks.\\

\subsection{Sun Dataset}\label{sec:sun}

\noindent \textbf{Dataset Description.} The Sun Dataset \cite{user2} contains 14,340 scene images from SUN Attribute Database \cite{sundata}, with personalized annotations over 5 attributes (CL: Cluttered, MO: Modern, OP: Opening Area, RU: Rustic, SO: Soothe). With a similar annotating procedure, 64,900 annotations are obtained in this dataset.\\
\noindent \textbf{Pre-processing.} For input features, we deploy the 2048-dim feature vectors extracted by the Inception-V3 \cite{inception} network for Sun's data. {The reason leading us to different feature extraction strategies lies in that the images in Shoes dataset are photographed on a white background, while images in Sun dataset usually suffer from much more complicated backgrounds.} The rest of the pre-processing follows that for the Shoes dataset.\\
\noindent \textbf{Performance comparison.}
The average performances of 15 repetitions for Sun dataset are shown in the left side of Fig.\ref{fig:bar} where the scatters show the 15 observations over different dataset splits and the bar plots show the average performance over 15 repetitions}. Similar to the Shoes dataset, we could make the following observations: 1) Our proposed algorithm consistently outperforms all the competitors significantly for all the attributes. Our method outperforms the second best method by the AUC score of 1.95, 2.95, 0.44, 2.37, 1.29 in terms of \emph{CL}, \emph{MO}, \emph{OA}, \emph{RU}, and \emph{SO}, respectively.  2) Moreover, we have a similar observation to the 2)-5) for the shoes dataset.\\
\noindent \textbf{Ablation Study.} Similar to the Shoes Dataset, we show the corresponding ablation results for Sun Dataset in Fig.\ref{fig:ab:CL}-Fig.\ref{fig:ab:SO} and  Fig.\ref{fig:ab2:CL}-Fig.\ref{fig:ab2:SO}. From the experimental results, we have the following observations: 1) Our original algorithm outperforms both competitors for all attributes except Soothing. 2) Removing the grouping factor shows a relatively significant effect on performance reduction than replacing the square AUC loss. 3) Our proposed regularizer could outperform the co-grouping regularizer in CoCMTL.
\\ \noindent \textbf{Fine-grained comparison.} We then visualize the user-specific performance distribution for the sun dataset. Take the \emph{Open Area} attribute as an example, we visualize the testing AUC score distributions over users for all the methods in Fig.\ref{fig:dist_sun_oa}. As shown in this figure, our model achieves a more compact distribution than other competitors. This again shows our method could alleviate the negative transfer issue.\\

\section{Conclusion}
In this paper, we develop a novel multi-task learning method called TFCL, which prevents negative transfer simultaneously at the feature and task level via a co-grouping regularization. An  optimization method is then proposed to solve the model parameters, which iteratively solves convex subproblems. Moreover, we provide a novel closed-form solution for one of the subproblems, which paves the way to our proof of the global convergence property.  Meanwhile, the solution produced by the optimization method shows a close connection between our method and the optimal transport problem, which brings new insight into how negative transfer could be prevented across features and tasks. We further extend the TFCL method to the problem of personalized attribute learning via a hierarchical model decomposition scheme. In order to validate the proposed methods, we perform systematic experiments on a simulated dataset and two real-world datasets. Results on the simulated dataset show that TFCL could indeed recover the correct co-grouping structure with good performance, and results on the real-world datasets further verify the effectiveness of our proposed model on the problem of personalized attribute prediction.
\bibliographystyle{abbrv}
\bibliography{cite}
\ifCLASSOPTIONcaptionsoff
\newpage
\fi

\vspace{0cm}\begin{IEEEbiography}[{\includegraphics[width=1in,height=1.25in,clip,keepaspectratio]{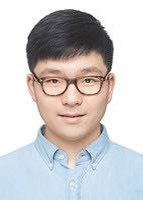}}]{\textbf{Zhiyong Yang} received the M.E. degree in computer
science and technology from University of Science
and Technology Beijing (USTB) in 2017. He is currently pursuing the Ph.D. degree with University of
Chinese Academy of Sciences. His research interests
lie in theoretical and algorithmic aspects of machine
learning, with special focus on multi-task learning,
meta-learning, and learning with non-decomposable
metrics. He has authored or coauthored several academic papers in top-tier international conferences and journals
including NeurIPS/CVPR/AAAI/T-PAMI/T-IP. He served as a reviewer for several top-tier conferences such as ICML, NeurIPS and AAAI.
}
\end{IEEEbiography}
\vspace{0cm}\begin{IEEEbiography}[{\includegraphics[width=1in,height=1.25in,clip,keepaspectratio]{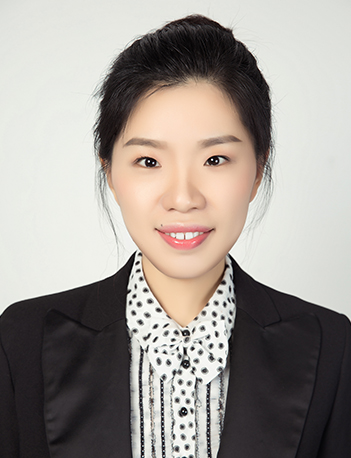}}]{\textbf{Qianqian Xu} received the B.S. degree in computer
		science from China University of Mining
		and Technology in 2007 and the Ph.D. degree
		in computer science from University of Chinese
		Academy of Sciences in 2013. She is currently
		an Associate Professor with the Institute of Computing
		Technology, Chinese Academy of Sciences,
		Beijing, China. Her research interests include
		statistical machine learning, with applications
		in multimedia and computer vision. She has
		authored or coauthored 30+ academic papers in
		prestigious international journals and conferences, including T-PAMI/T-IP/T-KDE/ICML/NeurIPS/CVPR/AAAI, {etc}. She served as a reviewer for several top-tier conferences such as ICML, NeurIPS, ICLR, CVPR, ECCV, AAAI, IJCAI, and ACM MM, etc.}
\end{IEEEbiography}
\vspace{0cm}\begin{IEEEbiography}[{\includegraphics[width=1in,height=1.25in,clip,keepaspectratio]{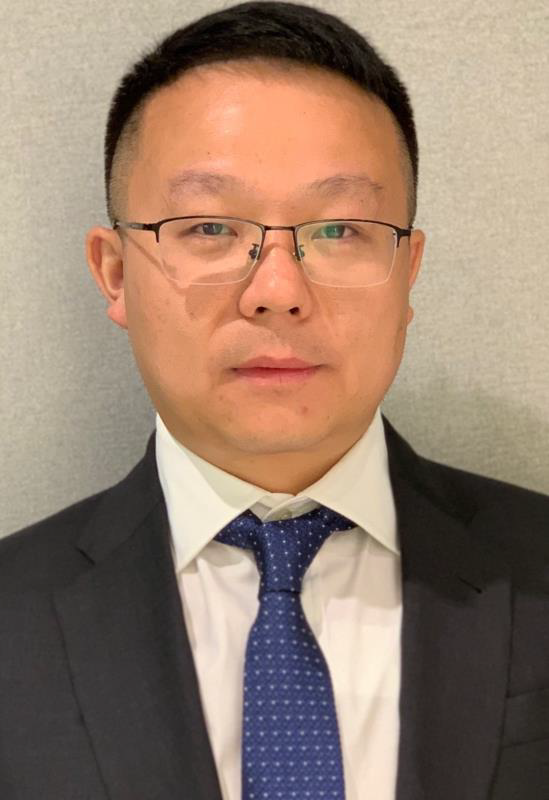}}]{\textbf{Xiaochun Cao}, Professor of the Institute of Information Engineering, Chinese Academy of Sciences. He received the B.E. and M.E. degrees both in computer science from Beihang University (BUAA), China, and the Ph.D. degree in computer science from the University of Central Florida, USA, with his dissertation nominated for the university level Outstanding Dissertation Award. After graduation, he spent about three years at ObjectVideo Inc. as a Research Scientist. From 2008 to 2012, he was a professor at Tianjin University. He has authored and coauthored over 100 journal and conference papers. In 2004 and 2010, he was the recipients of the Piero Zamperoni best student paper award at the International Conference on Pattern Recognition. He is a fellow of IET and a Senior Member of IEEE. He is an associate editor of IEEE Transactions on Image Processing, IEEE Transactions on Circuits and Systems for Video Technology and IEEE Transactions on Multimedia.}
\end{IEEEbiography}

\vspace{0cm}\begin{IEEEbiography}[{\includegraphics[width=1in,height=1.25in,clip,keepaspectratio]{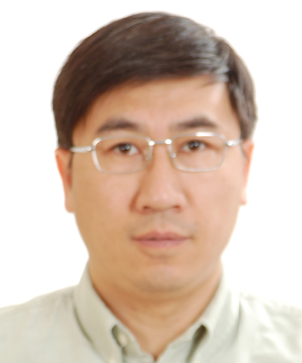}}]{\textbf{Qingming Huang} is a chair professor in the University of Chinese Academy of Sciences and an adjunct research professor in the Institute of Computing Technology, Chinese Academy of Sciences. He graduated with a Bachelor degree in Computer Science in 1988 and Ph.D. degree in Computer Engineering in 1994, both from Harbin Institute of Technology, China. His research areas include multimedia computing, image processing, computer vision and pattern recognition. He has authored or coauthored more than 400 academic papers in prestigious international journals and top-level international conferences. He is the associate editor of IEEE Trans. on CSVT and Acta Automatica Sinica, and the reviewer of various international journals including IEEE Trans. on PAMI, IEEE Trans. on Image Processing, IEEE Trans. on Multimedia, etc. He is a Fellow of IEEE and has served as general chair, program chair, track chair and TPC member for various conferences, including ACM Multimedia, CVPR, ICCV, ICME, ICMR, PCM, BigMM, PSIVT, etc.}
\end{IEEEbiography}

\newpage \onecolumn
\section*{\centering{Supplementary Materials}}

\appendices

\section{Proof of Theorem  \ref{eigsol}}
	\begin{proof}{\color{white} dsdads}
 \\ Denote $\alpha = \frac{\alpha_3}{2C}$, we have:\\
	\textbf{1) $\bm{\breve{\delta}(\LGbi) > \alpha > 0}$}.
	
		According to Sec.5.9.2  in \cite{cvx}, let $\Omega_1, \Omega_2 \in \mathbb{R}^{(d+T) \times (d+T)}$ be the Lagrangian multipliers for constraints $-\bm{U}\preceq 0$,  $\bm{U}-\bm{I}\preceq 0$, respectively. Moreover, denote $\beta$ as the  Lagrangian multiplier for $tr(\bm{U}) = k$. We come to the KKT condition\footnote[1]{We ignore the constraint $\bm{U} = \bm{U}^\top$, since the optimal solution without this constraint is feasible.}:
		\begin{align}
		&\LGbi  + \alpha\bm{U} - \Omega_1 + \Omega_2 +\beta\bm{I}  = 0 \\
		&\left<-\bm{U}, \Omega_1\right>  = 0, \left<\bm{U}- \bm{I}, \Omega_2\right>  = 0, \\
		&\Omega_1  \succeq 0, \Omega_2 \succeq 0, \bm{U}\succeq 0,  \bm{I} - \bm{U}\succeq 0.   
		\end{align}
		Denote $\bm{\omega}_1 = diag(\omega_{1i})$,   $\bm{\omega}_2 = diag(\omega_{2i})$, $\bm{\lambda} = diag(\lambda_{i})$ be the eigenvalues of $\Omega_1$, $\Omega_2$ and $\bm{U}$, respectively.
		According to Lem. 2 of \cite{sdpopt}, one can then reach:
		\begin{align}
		&-\LGbi   =  \bm{V}_U( \alpha \bm{\lambda} +   \bm{\omega}_2 -\bm{\omega}_1 + \beta\bm{I})\bm{V}_U^\top,\\
		&\omega_{1i} \cdot \lambda_{i} = 0, \forall i \in [N],\\
		&\omega_{2i} \cdot (\lambda_{i} -1) = 0, \forall i \in [N], \\
		&\gamma_i \ge 0, 	\omega_{1i} \ge 0, \omega_{2i} \ge 0, ~~\forall \  i \in [N],\\
		& 1 \ge \lambda_i \ge 0, \forall \  i \in [N],\\
		&tr(\bm{U}) =k.
		\end{align}
		where $\bm{V}_U$ contains the eigenvectors of $\bm{U}$.  We then prove from the following cases. \\
		\textbf{case(1): $\bm{p = k-1, q = k}$}: This implies that the desired eigengap is not null. Then the solutions for primal and dual variables are :
		\begin{align}
		\omega_{1i} & = \left(\laml{i} - \laml{k} - \alpha \right)\cdot \mathbbm{1}\left[i > k\right],  \\
		\omega_{2i} & = \left(\laml{k} - \laml{i}\right)\cdot \mathbbm{1}\left[i \le k \right], \\
		\beta_{i} & = -\laml{k}  - \alpha,  \\
		\bm{U}^\star & = \sum_{i=1}^k\bm{v}_i\bm{v}_i^\top.  
		\end{align}
		
		\noindent\textbf{case(2)}. We continue to show that the result holds when \textbf{$\bm{p \neq  k-1, q \neq  k}$}. Specifically, we show from the following four cases.\\
		\textbf{case (2a): $\bm{p \neq 0, q \neq N}$ }.
		The following primal and dual variables satisfy the KKT condition:  
		\begin{align}
		\omega_{1i} & = \left(\laml{i} - \laml{p+1} - \frac{k-p}{q-p}\alpha \right)\cdot \mathbbm{1}\left[i > q\right],  \\
		\omega_{2i} & = \left(\laml{p+1} - \laml{i} - \frac{q-k}{q-p}\alpha\right)\cdot \mathbbm{1}\left[i \le p \right], \\
		\beta_{i} & = -\laml{p+1}  - \frac{k-p}{q-p} \alpha,  \\
		\bm{U}^\star & = \sum_{i=1}^p\bm{v}_i\bm{v}_i^\top + \frac{k-p}{q-p}\sum_{j= p+1}^{q}\bm{v}_j\bm{v}_j^\top. 
		\end{align}
		
		\noindent \textbf{case (2b): $\bm{p = 0, q \neq N}$}.
		According to the spectral properties of graph Laplacian matrix, we must have 
		\begin{equation}
		0 =\lam{1} = \cdots = \lam{k} = \cdots =  \lam{q}  \le \lam{q+1} \le \lam{N}.
		\end{equation}
		We have the following primal and dual variables satisfy the KKT condition:	
		\begin{align}
		\omega_{1i} & = \left(\laml{i} - \frac{k}{q}\alpha \right)\cdot\mathbbm{1}\left[i > q\right],  \\
		\omega_{2i} & = {0}, \\
		\beta_{i} & = -\frac{k}{q} \alpha,  \\
		\bm{U}^\star & = \frac{k}{q}\sum_{j= 1}^{q}\bm{v}_j\bm{v}_j^\top.
		\end{align}
		
		\noindent \textbf{case (2c): $\bm{p \neq 0, q = N}$}. Since $\LGbi \neq \bm{0}$, we have:
		\begin{equation}
		\lam{1} \le \cdots \le  \lam{p}   < \lam{p+1} = \cdots = \lam{k} = \cdots =\lam{N}.
		\end{equation}
		The  following primal and dual solutions satisfy the KKT condition:
		\begin{align}
		\omega_{1i} & =0,  \\
		\omega_{2i} & = \left(\laml{p+1} - \laml{i} -\frac{N-k}{N-p}\alpha\right)\cdot \mathbbm{1}\left[i \le p \right], \\
		\beta_{i} & = -\laml{p+1}- \frac{k-p}{N-p} \alpha,  \\
		\bm{U}^\star & = \sum_{i=1}^{p}\bm{v}_i\bm{v}_i^\top +  \frac{k-p}{N-p}\sum_{j= p+1}^{N}\bm{v}_j\bm{v}_j^\top.
		\end{align}
		\textbf{case(2d): $\bm{p=0, q= N}$}. This case is impossible due to the fact that $\LGbi \neq \bm{0}$.\\
		The proof is now complete for $\alpha > 0$, since in all possible cases, we have $\bm{U}^\star= \bm{V}\tilde{\bm{\Lambda}}\bm{V}^\top$.

\noindent \textbf{2)$\bm{\alpha = 0}$}. Since $\bm{U}^\star$ is a feasible solution, it suffices to show that \refeq{eq:opt} is an optimal solution of the original problem. 
According to Thm. \ref{thm:eig}, we have:
\begin{equation}
\bm{U}^\star \in \argmin_{\bm{U} \in \Gamma} \left<\LGbi, \bm{U}\right> \  \ \text{if} \ \  \left<\LGbi, \bm{U}^\star \right>  = \sum_{i=1}^k \lam{i}.
\end{equation}
Furthermore, we have:
	\begin{equation}
	\begin{split}
	\left<\LGbi, \bm{U}^\star\right> & = tr\left(\tilde{\bm{\Lambda}}\boldsymbol{V}^\top  \LGbi \boldsymbol{V} \right)\\
 &=tr\left(\bm{\tilde{\Lambda}}\bm{V}^\top \bm{V}\bm{\Lambda}\bm{V}^\top\bm{V}\right) \\
	&= tr\left(\tilde{\bm{\Lambda}}\bm{\Lambda}\right) = \sum_{i=1}^p\lam{i} + \frac{k-p}{q-p}\sum_{i = p+1}^q\lam{p+1}\\
	 &= \sum_{i=1}^p \lam{i} + (k-p) \cdot \lam{p+1} = \sum_{i=1}^k \lam{i}.
	\end{split}
	\end{equation}   
	Then we end the proof of \textbf 2).
	\end{proof}
\section{Preliminary of Nonconvex-Nonsmooth Optimization}
\subsection{Sub-Differential}
Now we introduce the generalized sub-differential for proper and lower semi-continues functions (but not necessarily convex) \cite{var}, which underpins the convergence analysis in this paper. 
\begin{defi}
	Consider a function $f: \mathbb{R}^n \rightarrow \bar{\mathbb{R}}$ and a point $\bar{x}$ with $f(\bar{x})$ finite. For a vector $v \in \mathbb{R}^n$, we have:
	\begin{enumerate}[itemindent=0pt, leftmargin =15pt,label={(\arabic*)}]
		\item  $v$ is a \textit{\textbf{regular subgradient}} of $f$ at $\bar{x}$, written $v \in {\hat{\partial} f(\bar{x})}$, if 
		\begin{equation}\label{key}
		f(x) \ge f(\bar{x}) + \left<v, x-\bar{x}\right> + o(|x - \bar{x}|);
		\end{equation}
		\item  v is a \textit{\textbf{(general) subgradient}} of $f$ at $\bar{x}$, written as $v \in \partial f(\bar{x})$, if there are sequences $x^{\nu} \rightarrow \bar{x}$, with $f(x^\nu) \rightarrow f(\bar{x})$, and
		$v^\nu \in \hat{\partial} f(x^\nu)$ with $v^\nu \rightarrow v$.
	\end{enumerate}
\end{defi}
\noindent Note that notation $o$ in a), stands for a short-hand for a one-sided limit condition:
\begin{equation}\label{key}
\liminf_{x\rightarrow \bar{x}, x \neq \bar{x}}\dfrac{f(x) -f(\bar{x}) - \left<v,x-\bar{x}\right>}{|x - \bar{x}|} \ge 0.
\end{equation}
Specifically, the generalized subgradient has the following properties.
\begin{proper}[existence of the subgradient](Corollary 8.10 of \cite{var})
	If the function $f: \mathbb{R}^n \rightarrow \bar{\mathbb{R}}$ is finite and lower semi-continuous at $\bar{x}$, then we have $\partial f(\bar{x})$ is not empty.
\end{proper}
\noindent With Property 1, we see that generalized subgradient always exists for finite and lower semi-continuous functions.
\begin{proper}[Generalized Fermat rule for local minimums](Thm. 10.1 of \cite{var})
	If a proper function $f: \mathbb{R}^n \rightarrow \bar{\mathbb{R}}$ has a local minimum at $\bar{x}$, then we have $0 \in \partial f(\bar{x})$.
\end{proper}

\noindent With Property 2, we define $\bar{x}$ as a \textbf{\textit{critical point}} of $f$, if $0 \in \partial{f}$.
\begin{proper}[General subgradient for convex functions](Prop. 8.12 of \cite{var})\label{proper:conv}
  If a proper function $f: \mathbb{R}^n \rightarrow \bar{\mathbb{R}}$ is convex, then we have:
  \begin{equation}
  \partial{f}(\bm{x}) = \{\bm{v}: f(\bm{y}) \ge f(\bm{x}) + \left<\bm{v}, \bm{y} - \bm{x}\right>, \bm{y} \in \mathbb{R}^n\} = \hat{\partial}f(\bm{x}).
  \end{equation}
\end{proper}
\noindent The property above shows that this subgradient definition is compatible with the definition for convex function.

\subsection{KL Functions}
\begin{defi}[Kurdyka-\L ojasiewicz (KL) property]\cite{KL,icmlzeng19}
	The function $G: \mathbb{R}^{n} \rightarrow \mathbb{R}\cup {+\infty}$
	is said to have the KL property at $\bm{x} \in dom\{\partial{G}\}$ if there exists a $\eta \in (0, +\infty)$, and an open   ball $\mathcal{B}(\bm{x},\rho)$ centered at  $\bm{x}$ with radius $\rho$ and a  concave function $\phi(t)$ that (1) is continuous at 0 and (2) satisfies $\phi(0) = 0$, (3) $\phi \in \mathcal{C}^1\left((0,\eta)\right)$, (4) $\phi'(x) >0, \forall \ x \in(0,\eta) $, such that
	for all $ \bm{y} \in \mathcal{B}(\bm{x},\rho) \cap [G(\bm{x}) < G(\bm{y}) < G(\bm{x}) + \eta],$ the following KL inequality holds:
	\begin{equation}
	\phi'\big(G(\bm{y}) - G(\bm{x})\big) \cdot dist\big(0, \partial G(\bm{y})\big) \ge 1,
	\end{equation}
	where for a set $\mathcal{S} \subset \mathcal{R}^n$, $dist\left(\bm{x},\mathcal{S}\right) = \inf\limits_{\bm{y} \in \mathcal{S}} \norm{\bm{x} - \bm{y}}.$
\end{defi}
\begin{defi}[KL functions]\cite{KL, icmlzeng19}
	Proper lower semicontinuous functions which satisfy the Kurdyka-\L ojasiewicz inequality at each point of $dom \{\partial G\}$ are called KL functions.
\end{defi}
\noindent In this paper, we use two subclasses of the KL functions: semi-algebraic functions, and definable functions. First we provide the definition for semi-algebraic functions and semi-algebraic sets.
\begin{defi}[Semi-algebraic function\cite{realana, icmlzeng19, sparsezeng}] Semi-algebraic sets and semi-algebraic functions could be defined as follows:
\begin{enumerate}[itemindent=1pt, leftmargin =15pt,label={(\arabic*)}]
\item	A set $\mathcal{A} \subset \mathbb{R}^n$ is called semi-algebraic if it can be represented as: \[\mathcal{A} = \bigcup_{i=1}^{m}\bigcap_{j=1}^n \left\{\bm{x} \in \mathbb{R}^n: p_{ij}(\bm{x})=0, q_{ij}(\bm{x}) > 0 \right\},\]
where $p_{ij}, q_{ij}$ are real polynomial functions for $ i \in [m], j \in [n]$.
\item A function h is called semi-algebraic if its graph \[Gr(h) =\{(x, h(x)) : x \in dom(h)\}\] is a semi-algebraic set.
\end{enumerate}
\end{defi}
\noindent Before introducing the definable functions, we need an extension of semi-algebraic sets, called the o-minimal structure \cite{cite2}, which is defined as follows:
\begin{defi}[o-minimal structure]
	\textit{An o-minimal structure on} $(\mathbbm{R}, +, \cdot)$ \textit{is a sequence of boolean algebras} $\mathcal{O}_n$ \textit{of definable subsets of} $\mathbbm{R}^n$\textit{, such that for each} $n \in \mathbbm{N}$:
	\textit{
		\begin{enumerate}[itemindent=0pt, leftmargin =15pt,label={(\arabic*)}]
			\item if $A$ belongs to $\mathcal{O}_n$, then $A \times \mathbbm{R}$ and $\mathbbm{R} \times A$ belong to $\mathcal{O}_{n+1}$;
			\item if $\Pi:\mathbbm{R}_{n+1} \rightarrow \mathbbm{R}_{n}$ is the canonical projection onto $\mathbbm{R}_{n}$, then for any $A$ in $\mathcal{O}_{n+1}$, the set $\Pi(A)$ belongs to $\mathcal{O}_n$;
			\item $\mathcal{O}_n$ contains the family of algebraic subsets of $\mathbb{R}_n$, that is, every set is in  the form:
			\begin{equation*}
				\{ x \in \mathbbm{R}_n:p(x)=0 \},
			\end{equation*}
			where $p:\mathbbm{R}_{n} \rightarrow \mathbbm{R}$ is a polynomial function.
			\item The elements of $\mathcal{O}_1$ are exactly finite unions of intervals and points.
		\end{enumerate}}
		\label{define:def1}
	\end{defi}
	
\noindent Based on the definition of o-minimal structure, we can show the definition of the \textit{definable function}.
	
	\begin{defi}[Definable  function]
		\textit{Given an o-minimal structure $\mathcal{O}$ (over $(\mathbbm{R}, +, \cdot)$), a function $f:\mathbbm{R}_{n} \rightarrow \mathbbm{R}$ is said to be definable in $\mathcal{O}$ if its graph belongs to $\mathcal{O}_{n+1}$.}
		\label{define:def2}
	\end{defi}
\begin{rem}\label{rem:def}
	\noindent According to \cite{cite1, cite2}, there are some important facts of the o-minimal structure, shown as follows.
	
	\begin{enumerate}[itemindent=0pt, leftmargin =15pt,label={(\arabic*)}]
		\item The collection of \textit{semialgebraic} sets is an o-minimal structure. Recall the semi-algebraic sets are Boolean combinations of sets in the form
		\begin{equation*}
		\{ x \in \mathbbm{R}_n:p(x)=0, q_1(x)<0 \ldots q_m(x)<0 \},
		\end{equation*}
		where $p$ and $q_i$'s are polynomial functions in $\mathbbm{R}_n$.
		\item The o-minimal structure is stable under the sum, composition, the inf-convolution and several other classical operations of analysis.
	\end{enumerate}
\end{rem}

\noindent The following properties about semi-algebraic sets, semi-algebraic functions and definable functions are necessary for our analysis: 
\begin{prop}\label{prop:semi}
	The following facts hold:
	\begin{enumerate}[itemindent=0pt, leftmargin =16pt,label={(\arabic*)}]
	\item Indicator functions of semi-algebraic sets are semi-algebraic.\cite{blockzeng} 
	\item  Finite sums and products of semi-algebraic functions are semi-algebraic.\cite{blockzeng}
	\item   Intersection and finite union of semi-algebraic sets are semi-algebraic sets.\cite{blockzeng}
    \item  Polynomial functions are semi-algebraic functions.
	\item Semi-algebraic functions are definable functions.
	\item Definable functions are KL functions.	\cite{cite2}
	\end{enumerate}
\end{prop}
\begin{proof}
\textbf{proof of (4)} For any polynomial functions $y =h(x) = p_n(x)$, we could reformulate its graph as:
\begin{equation}
Gr(h) = \{(x,y): y - p_n(x) = 0, y-p_n(x) +1 >0 \}.
\end{equation}
Obviously both $y - p_n(x)$ and $y - p_n(x)+1$ are real polynomials, and thus we complete the proof.\\
\text{proof of (5)} follows Rem. \ref{rem:def}-(1).
\end{proof}

\section{Proof of the Convergence Property for TFCL}
In this section, we will provide a proof for Thm. \ref{thm:glob_star} and Thm. \ref{thm:global}, based on the preliminaries mentioned in the last section. In this section, we denote $\mathcal{F}(\boldsymbol{W}, \boldsymbol{U})$ by the overall objective function for the surrogate problem $(\bm{P}^\star)$, and denote $\widetilde{\mathcal{F}}(\boldsymbol{W}, \boldsymbol{U})$ by corresponding function for the original problem $(\bm{P})$.\\
We first analyze the $(\bm{P}^\star)$ problem.
\noindent\begin{lem}\label{lem:cond} Let $\mathcal{F}(\boldsymbol{W}, \boldsymbol{U}) = \mathcal{J}(\boldsymbol{W}) +  \alpha 
	\cdot \left<\LGbi, \boldsymbol{U}\right> + \iota_{\Gamma}(\boldsymbol{U})$, where $\iota_{\Gamma}(\cdot)$ is the indicator function for the set $\Gamma$, and let $\boldsymbol{W}^t,\boldsymbol{U}^t$ be the parameters obtained at iteration $t$, the following properties hold for Alg.~\ref{alg:opt} if $\mathcal{J}(\boldsymbol{W})$ is a definable function, $\nabla_{\boldsymbol{W}}\mathcal{J}(\boldsymbol{W})$ is $\varrho$-Lipschitz continuous, and $\bm{W}_t \neq 0, \forall t$.
	\begin{enumerate}[itemindent=0pt, leftmargin =15pt,label={(\arabic*)}]
		\item (Sufficient Decrease Condition): If \ $C > \varrho$, the sequence $\{ \fk{t} \}$ is non-increasing in the sense that: 
		\begin{equation*}
		\begin{split}
		&\fk{t+1} \le \fk{t} - \min\left\{\frac{C-\varrho}{2},\frac{\alpha_3}{2}\right\}
		\norm{\Delta(\boldsymbol{\Theta}^t)}^2.
		\end{split}
		\end{equation*}
		where  $\Delta(\boldsymbol{\Theta}^t) = \big[vec(
		\Delta(\bm{W})^t ); vec(\Delta(\bm{U}^t))\big],   \Delta(\boldsymbol{W}^t) = \boldsymbol{W}^{t+1} - \boldsymbol{W}^t, \  \Delta(\boldsymbol{U}^t) = \boldsymbol{U}^{t+1} - \boldsymbol{U}^t  $.
		\item(Square Summable): $\sum_{i=1}^\infty ||\Delta(\bm{\Theta}^t)||_F^2 < \infty$. Furthermore, we have $\lim_{t\rightarrow} ||\Delta(\bm{U}^t)||  =0, ~\lim_{t\rightarrow} ||\Delta(\bm{W}^t)||  =0$.

		\item(Continuity Condition): There exist a subsequence of $\{\bm{W}^{k_j}, \bm{U}^{k_j}\}_{j}$ with a limit point $\{\bm{W}^\star,\bm{U}^\star\}$, such that 
		\[\{\bm{W}^{k_j}, \bm{U}^{k_j}\} \rightarrow \{\bm{W}^\star,\bm{U}^\star\}, ~~ \text{and} ~~ \mathcal{F}(\bm{W}^{k_j}, \bm{U}^{k_j}) \rightarrow  \mathcal{F}(\bm{W}^{\star}, \bm{U}^{\star}). \]
		\item (KL Property): $\mathcal{F}(\cdot,\cdot)$ is a KL function.
		\item (Relative Error Condition): There holds:
		\begin{equation}
		\begin{split}
&\distFl{t+1} \le \left[{C+ \varrho} + \alpha_1 (\sqrt{d}+ \sqrt{T} +2)\right] \cdot  \left\lvert\left\lvert \bm{\Theta}^{t+1} - \bm{\Theta}^t\right\rvert\right\rvert_F, \  \text{for} \ t \in \mathbb{N}.
		\end{split}
		\end{equation}
	\end{enumerate}
\end{lem}
\subsection{Proof of Lemma \ref{lem:cond}}
\begin{proof} {\color{white} dsadsa}\\
	
\noindent \textbf{proof of (1)}:\\ 

\noindent Since ${\mathcal{J}}(\cdot)$ is $\varrho$-Lipschitz continuous, for every iteration $t+1$:
	\begin{equation}\label{lp}
	\begin{split}
	\lk{t+1} \le& \lk{t} + \left<\nabla_{\boldsymbol{W}}\mathcal{J}(\boldsymbol{W}^t), \Delta(\boldsymbol{W}^t)\right> + \frac{\varrho}{2}\norm{\Delta(\boldsymbol{W}^t)}^2_F. 
	\end{split}
	\end{equation}
With $\bm{W}$ fixed as $\bm{W}^t$, $\boldsymbol{U}$ subproblem reaches the unique solution $\boldsymbol{U}^{t+1}$. Since the subproblem is $\alpha_3$-strongly convex, we have:
	\begin{equation}\label{optu}
	\begin{split}
	&\alpha_1 \cdot \left<\bm{D}^{t+1},|\boldsymbol{W}^t|\right> + \iota_{\Gamma}({\bm{U}}^{t+1}) +\frac{\alpha_3}{2}\norm{\bm{U}^{t+1}}^2_F \le 	
\alpha_1\cdot \left<\bm{D}^{t},|\boldsymbol{W}^t|\right> + \iota_{\Gamma}(\boldsymbol{U}^{t}) + \frac{\alpha_3}{2}\norm{\bm{U}^t}^2_F - \frac{\alpha_3}{2}||\Delta\bm{U}^t||_F^2.
	\end{split}
	\end{equation}  
	For the $\boldsymbol{W}$ subproblem:
	\begin{equation}
	\argmin_{\boldsymbol{W}} \dfrac{1}{2} \left\norm{\boldsymbol{W} -\widetilde{\boldsymbol{W}^t} \right}_F^2 + \frac{\alpha_1}{C}\cdot \left<\bm{D}^{t+1}, |\boldsymbol{W}|\right> + \frac{\alpha_2}{2C} \norm{\bm{W}}_F^2,
	\end{equation}
	we know it is  strongly convex. This implies that the solution $\boldsymbol{W}^{k+1}$ is the minimizer of this subproblem. In this sense, we have:
	
	\begin{equation}\label{optW}
	\begin{split}
	& \left<\nabla_{\bm{W}} \mathcal{J}(\bm{W}^t), \Delta\bm{W}^t\right> + \frac{C}{2}\norm{\Delta\bm{W}^t}_F^2  +  \alpha_1\cdot\left<\bm{D}^{t+1}, |\bm{W}^{t+1}|\right> + \frac{\alpha_2}{2}\norm{\bm{W}^{t+1}}_F^2   \le  \alpha_1 \cdot \left<\bm{D}^{t+1},|\boldsymbol{W}^t|\right> + \frac{\alpha_2}{2}\norm{\bm{W}^t}_F^2.
	\end{split}
	\end{equation}

\noindent 	Combining (\ref{lp}), (\ref{optu}), (\ref{optW}), we have:
	\begin{equation}\label{eq:dec}
	\fk{t+1} \le \fk{t} -  \min\{\frac{C-\varrho}{2}, \frac{\alpha_3}{2}\} 
	\left( ||\Delta(\Theta^t)||^2\right),
	\end{equation}
	Then we complete the proof for 1).
	\\
	
\noindent	\textbf{proof of (2)}:\\ 
	
	\noindent By adding up (\ref{eq:dec}) for $k=1,2,\cdots$, and the fact that $\mathcal{F}(\cdot,\cdot) \ge 0$, we have
	\begin{equation}\label{eq:finite}
	\sum_{t=1}^\infty \min\{\frac{C - \varrho}{2}, \frac{\alpha_3}{2} \} \cdot ||\Delta(\boldsymbol{\Theta}^t)||_F^2  \le \fk{0} .
	\end{equation}
	Since by assumption $\fk{0} < \infty$, this immediately implies that : $\Delta(\boldsymbol{W}^k) \overset{t \rightarrow \infty}{\rightarrow} 0$ and $\Delta(\boldsymbol{U}^t) \overset{t \rightarrow \infty}{\rightarrow} 0$.\\

	\noindent	\textbf{proof of (3)}:\\ 

	Next, we prove that there exists at least a limit point for the sequence $\{\boldsymbol{W^t},\boldsymbol{U}^t\}_t$. Since, for all $t \in \mathbb{N}$,  $\bm{U}_t \in \Gamma$, we have $\{\bm{U}^t\}$ is bounded. Similarly, since the loss sequence $\{\mathcal{F}(\bm{W}^t, \bm{U}^t)\}$ is non-increasing, we have $\norm{\bm{W}^t} \le  \sqrt{\frac{2{\mathcal{F}\left(\bm{U}^0, \bm{W}^0\right)}}{\alpha_2}}$, then the sequence $\{\bm{W}^t\}$ is bounded. 
	According to the Bolzano-Weierstrass theorem, any bounded sequence must have convergent subsequence, which immediately suggests the existence of at least one limit point for $\{\bm{W}^t, \bm{U}^t\}$. 
	
	Now we proceed to proof that the continuous condition. Picking an arbitrary convergent subsequence $\{\boldsymbol{W}^{t_j},\boldsymbol{U}^{t_j}\}_j$ with a limit point $\boldsymbol{
		W}^\star, \boldsymbol{U}^\star$. By the assumption that $\mathcal{J}$ is continuous, we know that $\mathcal{F}(\bm{W},\bm{U}) - \iota(\bm{U})$ is continuous, we must have $\lim_{j \rightarrow}\mathcal{F}(\bm{W}^{k_j},\bm{U}^{k_j}) - \iota(\bm{U}^{k_j}) = \mathcal{F}(\bm{W}^\star,\bm{U}^\star) - \iota(\bm{U}^\star)$. It only remains to proof that $\lim_{j \rightarrow \infty}\iota(\bm{U}^{k_j}) =  \iota(\bm{U}^{\star})$.  Since $\iota(\cdot)$ is lower semi-continuous, we have $\lim\inf_{j \rightarrow \infty}\iota(\bm{U}^{k_j}) \ge \iota(\bm{U}^{\star})$. It only remain to proof that $\lim\sup_{j \rightarrow \infty}\iota(\bm{U}^{k_j}) \le\iota(\bm{U}^{\star})$. Given any $t_j$, we known that $\bm{U}^{t_j}$ reaches the optimum of the $\bm{U}$-subproblem with fixed $\bm{W}^{t_j-1}$. We then have:
\begin{equation}
	\begin{split}
	&\alpha_1 \cdot \left<\bm{D}^{t_j},|\boldsymbol{W}^{t_j-1}|\right> + \iota_{\Gamma}({\bm{U}}^{t_j}) +\frac{\alpha_3}{2}\norm{\bm{U}^{t_j}}^2_F \le 	
\alpha_1\cdot \left<\bm{D}^{\star},|\boldsymbol{W}^{t_j-1}|\right> + \iota_{\Gamma}(\boldsymbol{U}^{\star}) + \frac{\alpha_3}{2}\norm{\bm{U}^\star}^2_F.
	\end{split}
	\end{equation} 
Taking $j \rightarrow \infty$ on both sides of the above inequality, we have
\begin{equation}
	\begin{split}
	\iota_{\Gamma}({\bm{U}}^{t_j})  \le 	
 \iota_{\Gamma}(\boldsymbol{U}^{\star}), ~~~ j \rightarrow \infty.
	\end{split}
	\end{equation} 
This suggests that $\lim\sup_{j \rightarrow \infty}\iota(\bm{U}^{k_j}) \le\iota(\bm{U}^{\star})$.

\noindent \textbf{proof of (4)}:\\
	
\noindent According to Prop. \ref{prop:semi}, polynomials functions are semi-algebraic functions, it then follows that the regularization terms $\regw$, $\regu$ are semi-algebraic functions.\\
	Now we show that $\inner$ is a semi-algebraic function. Since  
	\[\inner = \sum_{i=1}^T\sum_{j=d}^{d+T} \LGbi{_{ij}} {U}_{ij} +  \sum_{i=d}^{d+T}\sum_{j=1}^{d} \LGbi{_{ij}} {U}_{ij},\]
	and 
	\[\LGbi{_{ij}} {U}_{ij} = \begin{cases}
	\left[\mathbbm{1}[i = j] \cdot(\sum_{k=1}^T {|W|_{ik}}U_{ij})\right] - |W|_{i,j-d} U_{ij} &, i \le d, j > d, \\
    \left[\mathbbm{1}[i = j] \cdot (\sum_{k=1}^d{|W|_{kj}}U_{ij})\right] - |W|_{j,i-d} U_{ij}    &, i > d, j \le d, \\
	 0&,  otherwise,
	\end{cases}, \]
\noindent we know that $\inner$ is a sum of functions having the form $y = |x_1| \cdot x_2$. Together with  Prop. \ref{prop:semi}, $\inner$ is semi-algebraic if  $y = |x_1| \cdot x_2$ is semi-algebraic. The graph of this function could be formulated as 
\[\{(x_1,x_2, y): y + x_1\cdot x_2 =0, x_1 <0\} \cup  \{(x_1,x_2, y): y -x_1\cdot x_2 =0, x_1 >0\},  \]
which is obviously semi-algebraic. Hence  $\inner$ is a semi-algebraic function.
	
\noindent Now we show that the indicator function $\iota_\Gamma(\bm{U})$ is semi-algebraic. According to Prop. \ref{prop:semi}, we only need to show that $\Gamma$ is a  semi-algebraic set. It is easy to reformulate $\Gamma$ as:
	\begin{equation}
	\Gamma = \{\bm{U}: \bm{U} = \bm{U}^\top,\ \bm{U} \succeq \bm{0},\ \bm{I} - \bm{U} \succeq \bm{0}, \ tr(\bm{U}) =k  \}. 
	\end{equation} 
	Obviously
	\begin{equation}
\Gamma_1 =  \{\bm{U}: \bm{U} = \bm{U}^\top\} = \cap_{1\le i \neq j \le n}\{U_{ij}: U_{ij} = U_{ji}\} =  \cap_{1\le i \neq j \le n}\{U_{ij}: U_{ij} = U_{ji},\  U_{ij} - U_{ji} +1 > 0\}.
	\end{equation}

\noindent Then we have $\Gamma_1$ is  semi-algebraic.\\
 \noindent For $\Gamma_2$, we have
\begin{equation}
\begin{split}
\Gamma_2= \{\bm{U}: tr(\bm{U})=k\}=  \{\bm{U}: tr(\bm{U})= k,  tr(\bm{U}) -k +1 > 0  \} ,
\end{split}
\end{equation}
which suggests that $\Gamma_2$ is  semi-algebraic.\\
Now we prove that $\Gamma_3  = \{\bm{U} \in \mathbb{R}^{N \times N}: \bm{U} \succeq 0 \}$ is semi-algebraic. From the basic properties of positive semi-definite matrices, we know that $ \bm{U} \succeq 0 $ if and only if all its principal minors are non-negative. More precisely, given $\bm{U}$ and $1\le l \le N$, and  $ k_1,k_2,\cdots k_l \in \{1,2, \cdots l\}$ , we say $\bm{U}_{k_1,k_2,\cdots k_l}$ forms an $l \times l$ principal submatrix of $\bm{U}$ by choosing the  $k_1\text{-th},k_2\text{-th},\cdots k_l\text{-th}$ elements on each of the  $k_1\text{-th},k_2\text{-th},\cdots k_l\text{-th}$ row. Moreover, the determinant of $\bm{U}_{k_1,k_2,\cdots k_l}$ is defined as a principal minor of $\bm{U}$. In this way $\bm{U} \succeq 0$  is equivalent to $-Det(\bm{U})_{k_1,k_2,\cdots k_l} \le 0$ for all $l$ and all choices of $ k_1,k_2,\cdots k_l$, which could be written as:
	\begin{equation}
	\bm{U} \in \bigcap_{ 1 \le l \le N} \bigcap_{k_1,k_2, \cdots k_l} \{\bm{U}: -Det(\bm{U}_{k_1,k_2, \cdots k_l} ) \le 0 \}.
	\end{equation}
	By the definition of determinant, $Det(\bm{U}_l)$ could be represented as a polynomial  function of the elements in $\bm{U}_l$, thus we have that $\Gamma_3$ is semi-algebraic ($0$ is a real polynomial). The proof for $\Gamma_4 =\{\bm{U} \in \mathbb{R}^{N \times N}: \bm{I} - \bm{U} \succeq \bm{0}\}$ follows that for $\Gamma_3.$\\
	According to Prop. \ref{prop:semi}, $\Gamma = \Gamma_1 \cap \Gamma_2 \cap\Gamma_3 \cap\Gamma_4$ is semi-algebraic and so should be $\iota_\Gamma(\cdot)$. \\
	 \noindent Since semi-algebraic functions are definable, we know that all the summands in $\mathcal{F}$ are definable, which suggests that $\mathcal{F}(\bm{W},\bm{U})$ is a KL function.\\
	 
\noindent \textbf{proof of (5)}:\\
	
\noindent	From the optimality of $\bm{U}^{t+1}$ w.r.t the $\bm{U}$ subproblem at iteration $t+1$.
	\begin{equation}
	\begin{split}
	\bm{0} &\in \alpha_1 \cdot \LGbi^t   + \partial_{\boldsymbol{U}} \iota_{\Gamma}(\boldsymbol{U}^{t+1}) + {\alpha_3}\bm{U}^{t+1}\\
	& = \partial_{\bm{U}}\mathcal{F}(\bm{W}^{t+1}, \bm{U}^{t+1}) + \alpha_1  \cdot (\LGbi^{t+1} - \LGbi^{t}).
	\end{split}
	\end{equation}
	Accordingly, we have
	$\exists \bm{g}_{\bm{U}} \in \partial_{\bm{U}}\mathcal{F}(\bm{W}^{t+1}, \bm{U}^{t+1}) $, such that 
	\begin{align}
\norm{\bm{g}_{\bm{U}}}  &= \alpha_1 \cdot  ||\LGbi^{t+1} - \LGbi^{t}|| \\
& \le \alpha_1 \cdot \bigg(\left\lvert\left\lvert diag(|\bm{W}^{t}|\bm{1}) - diag(|\bm{W}^{t+1}|\bm{1}) \right\rvert\right\rvert 
+ \left\lvert\left\lvert diag(|\bm{W}^{t\top}|\bm{1}) - diag(|\bm{W}^{t+1\top}|\bm{1})\right\rvert\right\rvert + 2 \cdot \left\lvert\left\lvert|\bm{W}^{t}| - |\bm{W}^{t+1}|\right\rvert\right\rvert  \bigg) \\
&\le \alpha_1 (\sqrt{k} + \sqrt{T} +2)\cdot \norm{\bm{W}^t -\bm{W}^{t+1} }.
	\end{align}
	Thus 
	\begin{equation}
	dist(\bm{0}, \partial_{\bm{U}}\mathcal{F}(\bm{W}^{t+1}, \bm{U}^{t+1}) ) \le\alpha_1 (\sqrt{d} + \sqrt{T} +2) \cdot \norm{\bm{W}^t -\bm{W}^{t+1}}.
	\label{eq:grad1}
	\end{equation}
	From the optimality of $\bm{W}^{t+1}$ w.r.t the $\bm{W}$ subproblem at iteration $t+1$:
	\begin{equation}
	\begin{split}
	\bm{0} &\in \nabla_{W}\mathcal{J}(\boldsymbol{W}^{t}) + C\cdot (\boldsymbol{W}^{t+1} - \boldsymbol{W}^{t}) + 
	\\ &\alpha \cdot \partial_{\boldsymbol{W}}\boldsymbol{\Omega}(\boldsymbol{W}^{t +1},\bm{U}^{t + 1}) + \alpha_2 \cdot \bm{W}^{t+1} \\
	& =  \partial_{\bm{W}}\mathcal{F}(\bm{W}^{t+1}, \bm{U}^{t+1}) \\ & +{C}\cdot(\boldsymbol{W}^{t+1} - \boldsymbol{W}^{t}) +  \nabla_{W}\mathcal{J}(\boldsymbol{W}^{t}) - \nabla_{W}\mathcal{J}(\boldsymbol{W}^{t+1}).   
	\end{split}
	\end{equation}
	Accordingly, we have:
	$\exists \bm{g}_{\bm{W}} \in \partial_{\bm{W}}\mathcal{F}(\bm{W}^{t+1}, \bm{U}^{t+1}) $, such that \[||\bm{g}_{\bm{W}}||  =  ||{C}\cdot(\boldsymbol{W}^{t+1} - \boldsymbol{W}^{t}) +  \nabla_{W}\mathcal{J}(\boldsymbol{W}^{t}) - \nabla_{W}\mathcal{J}(\boldsymbol{W}^{t+1})||.\] Since $\mathcal{J}$ is $\varrho$-Lipschitz continuous, we have:   \[||\bm{g}_{\bm{W}}|| \le  ({C+ \varrho}) \cdot ||\boldsymbol{W}^{t+1} - \boldsymbol{W}^{t}||.\] Thus
	\begin{equation}
	dist(\bm{0}, \partial_{W}\mathcal{F}(\bm{W}^{t+1}, \bm{U}^{t+1})) \le   ({C+ \varrho})\cdot ||\boldsymbol{W}^{t+1} - \boldsymbol{W}^{t}||.
	\label{eq:grad2}
	\end{equation}
	According to Eq.(\ref{eq:grad1}) and Eq.(\ref{eq:grad2}), we have: 
	\begin{equation*}
	\begin{split}
	&dist(\bm{0}, \partial_{\Theta}(\mathcal{F}(\bm{W}^{t+1}, \bm{U}^{t+1}))  \\ &\le \distFlu{t+1}
	+ \distFlw{t+1}  \\ & \le \left[{C+ \varrho} + \alpha_1 (\sqrt{d}+ \sqrt{T} +2)\right] \cdot  \left\lvert\left\lvert \bm{W}^{t+1} - \bm{W}^t\right\rvert\right\rvert_F \\ & \le   \left[{C+ \varrho} + \alpha_1 (\sqrt{d}+ \sqrt{T} +2)\right] \cdot  \left\lvert\left\lvert \bm{\Theta}^{t+1} - \bm{\Theta}^t\right\rvert\right\rvert_F .
	\end{split}
	\end{equation*}
\end{proof}
\subsection{Proof of Theorem \ref{thm:glob_star}}
\begin{pfthm4}
	Following Lem. 2.6 in \cite{globalconv}, (1) holds according to the Sufficient Decrease Condition, Continuity Condition, KL property and Relative Error Condition (See Lem. \ref{lem:cond}).\\
	According to Lem. \ref{lem:cond}-(1), the loss function $\mathcal{F}$ is non-increasing with a lower bound 0, which implies that  $\{ \mathcal{F}(\bm{W}^t,\bm{U}^t) \}_t$ converges. Moreover, since $(\bm{W}^t, \bm{U}^t)  \overset{t \rightarrow \infty}{\rightarrow} (\bm{W}^\star, \bm{U}^\star) $, $(\bm{W}^\star, \bm{U}^\star)$ is the unique limit point of the parameter sequence. According to \ref{lem:cond}-(3), we have $ \mathcal{F}(\bm{W}^t,\bm{U}^t) \overset{t \rightarrow \infty}{\rightarrow} {\mathcal{F}}(\bm{W}^\star,\bm{U}^\star) $. This immediately suggests 2).\\
	According to Lem. \ref{lem:cond}-(5), we have: 
	\begin{equation}
    \begin{split}
&\sum_{i=1}^T \distFl{t}^2 \le +\infty,
    \end{split}
	\end{equation}
	
\noindent which implies (3).
\qed
\end{pfthm4}
{\color{white}dssad}\\
\noindent Finally, based on Thm.\ref{thm:eig} and Thm.\ref{thm:glob_star}, we can show the convergence property with respect to the original problem $(\bm{P})$.

\subsection{Proof of Theorem \ref{thm:global}}
\begin{pfthm5} {\color{white}dssad}\\
\\
\textbf{proof of (1)}\\

\noindent According to Thm. \ref{thm:glob_star}, we know that $(\bm{W}^t, \bm{U}^t)  \overset{t \rightarrow \infty}{\rightarrow} (\bm{W}^\star, \bm{U}^\star)$,  $ \mathcal{F}(\bm{W}^t,\bm{U}^t) \overset{t \rightarrow \infty}{\rightarrow}{\mathcal{F}}(\bm{W}^\star,\bm{U}^\star)$. It suffices to prove that $(\bm{W}^\star, \bm{U}^\star)$ is a critical point of $\widetilde{\mathcal{F}}$.\\

From Lem.\ref{lem:cond}-(5), we know that $dist(\bm{0},\partial_{\bm{W}}\mathcal{F}(\bm{W}^t, \bm{U}^t)) \rightarrow 0$  and $dist(\bm{0},\partial_{\bm{U}}\mathcal{F}(\bm{W}^t, \bm{U}^t)) \rightarrow 0$. This implies that 
\begin{equation}\label{eq:limsub1}
0 \in \partial_{\bm{W}}\mathcal{F}(\bm{W}^t, \bm{U}^t), ~ t \rightarrow \infty,
\end{equation}
and 
\begin{equation}\label{eq:limsub2}
0 \in \partial_{\bm{U}}\mathcal{F}(\bm{W}^t, \bm{U}^t), ~ ~ t \rightarrow \infty.
\end{equation}
$\mathcal{F}(\bm{W}, \bm{U})$ is convex w.r.t to $\bm{W}$ with $\bm{U}$ fixed. Meanwhile,  $\mathcal{F}(\bm{W}, \bm{U})$ is convex w.r.t to $\bm{U}$ with $\bm{W}$ fixed. Considering Property \ref{proper:conv},  Eq.(\ref{eq:limsub1}) and  Eq.(\ref{eq:limsub2}) imply:
\begin{equation}
\lim_{t \rightarrow \infty} \mathcal{F}(\bm{W}, \bm{U}^t) -\mathcal{F}(\bm{W}^t, \bm{U}^t) \ge 0
\end{equation}
and 
\begin{equation}
\lim_{t \rightarrow \infty} \mathcal{F}(\bm{W}^t, \bm{U}) - \mathcal{F}(\bm{W}^t, \bm{U}^t)  \ge 0
\end{equation}
Since $\mathcal{F}$ is continuous w.r.t. $\bm{W}$, we have $\lim_{t \rightarrow \infty} \mathcal{F}(\bm{W}^t, \bm{U}) =   \mathcal{F}(\bm{W}^\star, \bm{U})$. Moreover, since $\lim_{t \rightarrow \infty} \mathcal{F}(\bm{W}, \bm{U}^t) - \iota_{\Gamma}(\bm{U}^t) = \mathcal{F}(\bm{W}, \bm{U}^\star) - \iota_{\Gamma}(\bm{U}^\star)$  and $\lim_{t \rightarrow \infty} \iota_{\Gamma}(\bm{U}^t) =  \iota_{\Gamma}(\bm{U}^\star)$ (we used the fact that $\bm{U}^\star$ is unique limit point and the proof in Lem.\ref{lem:cond}-(3)), we have $\lim_{t \rightarrow \infty} \mathcal{F}(\bm{W}, \bm{U}^t) = \mathcal{F}(\bm{W}, \bm{U}^\star) $. Together with the fact that $ \mathcal{F}(\bm{W}^t,\bm{U}^t) \overset{t \rightarrow \infty}{\rightarrow}{\mathcal{F}}(\bm{W}^\star,\bm{U}^\star)$, we have:

\[0 \in \partial_{\bm{W}}\mathcal{F}(\bm{W}^\star, \bm{U}^\star), \] 
and 
\[0 \in \partial_{\bm{U}}\mathcal{F}(\bm{W}^\star, \bm{U}^\star). \] 
Obviously, we have $\partial_{\bm{W}}\mathcal{F}(\bm{W}^\star, \bm{U}^\star) = \partial_{\bm{W}}\widetilde{\mathcal{F}}(\bm{W}^\star, \bm{U}^\star)$, which implies that $0 \in \partial_{\bm{W}}\widetilde{\mathcal{F}}(\bm{W}^\star, \bm{U}^\star)$.  Since  $0 < \alpha_3 < 2C \min_t \breve{\delta}(\LGbi^t)$, $\bm{U}^\star$ is the unique optimal solution of the $U$-subproblem. This show that  $\bm{U}^\star$ must follow thm.\ref{eigsol}, which is also a solution for $\min_{\bm{U}}\widetilde{\mathcal{F}}(\bm{W}^\star, \bm{U})$. This means that $0 \in \partial_{\bm{U}}\widetilde{\mathcal{F}}(\bm{W}^\star, \bm{U}^\star).$ Taking all together, we then have:
\begin{equation}
0 \in   \partial\widetilde{\mathcal{F}}(\bm{W}^\star, \bm{U}^\star).
\end{equation}

\noindent \textbf{proof of (2) and (3)  }\\
Following a similar logic as in the proof of (1), we have 

\begin{equation}
0 \in \partial_{\bm{W}}\mathcal{F}(\bm{W}^t, \bm{U}^t), ~~
\text{and}~~ 0 \in \partial_{\bm{U}}\mathcal{F}(\bm{W}^t, \bm{U}^t) ~~ \text{implies that} ~~  0 \in \partial_{\bm{W}}\widetilde{\mathcal{F}}(\bm{W}^t, \bm{U}^t), ~~
\text{and}~~ 0 \in \partial_{\bm{U}}\widetilde{\mathcal{F}}(\bm{W}^t, \bm{U}^t)
\end{equation}
Then (2) and (3) could be proved in a similar way with Thm.\ref{thm:glob_star} and Lem.\ref{lem:cond}.
\qed
\end{pfthm5}

\section{Proof of the Grouping Effect}
\begin{lem}\label{lem:sets}
\begin{enumerate}
\item[(a)]Assume that for all $\infty> \kappa >0$, $\sup_{\norm{\bm{W}}_F \le \kappa}\norm{\nabla_{\boldsymbol{W}}\mathcal{J}(\boldsymbol{W})}_\infty \le \varpi(\kappa) < \infty$. If Alg.~\ref{alg:opt} terminates at the $\mathcal{T}$-th iteration, we have:
	\begin{equation}\label{eq:g2}
   \mathsf{Supp}(\bm{W}^\mathcal{T}) \subseteq \big\{(i,j): || \bm{f}^\mathcal{T}_i - \bm{f}^\mathcal{T}_{d+j} ||_2^2 < \delta_1 \big\}
	\end{equation}
	where 
	\begin{equation*}
	\begin{split}
	&\delta_1  =\frac{C}{\alpha_1} \cdot \left( C_0+ \dfrac{\varpi(C_0)}{C}\right) = \dfrac{C}{\alpha_1}\kappa_0, \  C_0 = \left(\dfrac{2}{\alpha_2} \cdot \epsilon_{\mathcal{T}-1}\right)^{1/2}.
	\end{split}
	\end{equation*}
\item[(b)] If $\min_{(i,j)} |\widetilde{W}^{\mathcal{T}}_{i,j}| \ge \delta_0 >0$, we have :
	\begin{equation}\label{eq:g2}
  \big\{(i,j): || \bm{f}^\mathcal{T}_i - \bm{f}^\mathcal{T}_{d+j}||_2^2 < \delta_2 \big\}  \subseteq \mathsf{Supp}(\bm{W}^\mathcal{T})
	\end{equation}
	where $\delta_2 =  \dfrac{C}{\alpha_1} \cdot \delta_0$.
\end{enumerate}

\end{lem}

\begin{proof} {\color{white}dsadsa}\\

\noindent \textbf{Proof of (a)}: \\
Since $\widetilde{W}^\mathcal{T}_{ij} = {W^{\mathcal{T}-1}_{ij} - \frac{\left[\nabla_{\boldsymbol{W}}\mathcal{J}(\boldsymbol{W}^{\mathcal{T}-1})\right]_{ij}}{C}}$, and  by assumption $\sup_{\norm{\bm{W}}_F \le \kappa}\norm{\nabla_{\boldsymbol{W}}\mathcal{J}(\boldsymbol{W})}_\infty \le \varpi(\kappa) < \infty$, we have: 
\begin{equation*}
\begin{split}
|\widetilde{W}^{\mathcal{T}}_{ij}|  \le& \left( \max_{(i,j)}|{{W}_{ij}^{\mathcal{T}-1}}| + \frac{\varpi(C_0)}{C}\right).
\end{split}
\end{equation*}
Moreover, we have: \[ \max_{(i,j)}|{{W}_{ij}^{\mathcal{T}-1}}| \le \sqrt{\dfrac{2}{\alpha_2} \cdot \epsilon_{\mathcal{T}-1} } = C_0, \ \max_{(i,j)}|{\widetilde{W}_{ij}^{\mathcal{T}}}| \le  \kappa_0. \]
Since $\bm{W}^{\mathcal{T}} \gets sgn(\bm{\widetilde{W}}^{\mathcal{T}})\left(\left|\dfrac{\bm{\widetilde{W}}^{\mathcal{T}}}{1+\frac{\alpha_2}{C}}\right| - \frac{\alpha_1}{C+\alpha_2} \bm{D}^\mathcal{T}\right)_+$, when $(i,j) \in \mathsf{Supp}(\bm{W}^\mathcal{T})$, we have :
\[ 
  ||\bm{f}^\mathcal{T}_i - \bm{f}^\mathcal{T}_{d+j}||^2 ~ = ~ {D}^\mathcal{T}_{i,j} ~<~ \frac{C}{\alpha_1} \cdot  \left|\bm{\widetilde{W}}^{\mathcal{T}}_{i,j}\right|  ~\le~  \dfrac{C}{\alpha_1} \cdot  \left(\max_{(i,j)}|{W}_{ij}^{\mathcal{T}-1}| + \frac{\varpi(C_0)}{C} \right) ~\le~ \frac{C}{\alpha_1} \cdot \left(C_0 +  \frac{\varpi(C_0)}{C} \right) ~=~ \delta_1.
\]

\noindent \textbf{Proof of (b)}: \\
By assumption of (b), when $ \dfrac{C}{\alpha_1}\delta_{0} > D^\mathcal{T}_{i,j}$, we have $\left|\dfrac{\bm{\widetilde{W}}^\mathcal{T}_{i,j}}{1+\frac{\alpha_2}{C}}\right| > \frac{\alpha_1}{C+\alpha_2} {D}^\mathcal{T}_{i,j}$, which suggests that $|\bm{W}^{\mathcal{T}}_{i,j}| > 0$ according to Alg.\ref{alg:opt}. 
\end{proof}

\begin{lem}\label{lem:dist}\cite[Lem.1]{dim}
Let $\bm{X}$ and $\bm{Y}$ be two orthogonal matrices of $\mathbb{R}^{n \times n}$. Let $\bm{X} = [\bm{X}_0, \bm{X}_1]$ and $\bm{Y} = [\bm{Y}_0, \bm{Y}_1]$, where $\bm{X}_0$ and $\bm{Y}_0$ are the first $K$ columns of $\bm{X}$ and $ \bm{Y}$, respectively. Then, we have:
\begin{equation}
\| \bm{X}_0\bm{X}_0^\top  - \bm{Y}_0\bm{Y}_0^\top\|_F \le \sqrt{2} \|\bm{X}_0^\top\bm{Y}_1\|_F.
\end{equation}
\end{lem}

\begin{lem}[sine $\Theta$]\label{lem:sine}\cite[{{Thm.1}}]{dh1}
Let  $\bm{\Sigma}, \hat{\bm{\Sigma}}$ be symmetric with eigenvalues $\lambda_1 \ge \cdots, \lambda_p$ and $\hat{\lambda_1} \cdots \hat{\lambda_p}$, respectively. Fix $1 \le K \le p$,  and let $\bm{X}_0 = [\bm{v}_1, \bm{v}_{2}, \cdots, \bm{v}_K] \in \mathbb{R}^{p \times K }$ and $\hat{\bm{Y}_0} = [\hat{\bm{v}}_1, \hat{\bm{v}}_2, \cdots, \hat{\bm{v}}_K]$  and let $\bm{X}_1 = [\bm{v}_{K+1}, \cdots, \bm{v}_p] $ and $\bm{Y}_1 = [\bm{\hat{v}}_{K+1}, \cdots, \hat{\bm{v}}_p]$. For $1 \le j \le p$, we have $\bm{\Sigma}\bm{v}_j = \lambda_j \bm{v}_j$ and $\hat{\bm{\Sigma}}\hat{\bm{v}}_j = \hat{\lambda}_j \hat{\bm{v}}_j$.  If $\delta = |\hat{\lambda}_{K+1} - \lambda_{K}$ | > 0, we have:
\begin{equation}
\|\bm{X}_0^\top \bm{Y}_1\|_F \le \frac{\|\bm{\Sigma} - \bm{\hat{\Sigma}}\|_F}{\delta}
\end{equation}

\end{lem}

\begin{pfthm6}{\color{white}dsadsa}\\
Let $\LGbi^{\mathcal{T}}$ and $\LGbi^\star$ be the corresponding graph Laplacian matrices for bipartite graphs $\mathcal{G}^\mathcal{T}$ and $\mathcal{G}^\star$. Let $\bm{v}^\mathcal{T}_1 \cdots, \bm{v}^\mathcal{T}_k$, and  $\bm{v}^\star_1 \cdots, \bm{v}^\star_k$ be  the   bottom k eigenvectors for  $\LGbi^{\mathcal{T}}$ and $\LGbi^\star$, and let  $\bm{V}^\mathcal{T}_k = [\bm{v}^\mathcal{T}_1 \cdots, \bm{v}^\mathcal{T}_k]$, $\bm{V}^\star_k = [\bm{v}^\mathcal{T}_1 \cdots, \bm{v}^\mathcal{T}_k]$, $\bm{U}^\mathcal{T} = \bm{V}^\mathcal{T}_k \bm{V}^\mathcal{T^\top}_k$ and $\bm{U}^\star = \bm{V}^\star_k \bm{V}^{\star^\top}_k$.
According to Lem.\ref{lem:dist} and Lem.\ref{lem:sine}, we have
\begin{equation}
\max_{(i,j)} |\bm{U}^{\mathcal{T}}_{i,j}- \bm{U}^\star_{i,j}| ~ \le~ ||\bm{U}^\mathcal{T} - \bm{U}^\star||_F \le \frac{\sqrt{2}||\LGbi^{\mathcal{T}}- \LGbi^\star ||_F}{\lambda_{k+1}(\LGbi^\mathcal{T})}
\end{equation}

\noindent Since $\bm{W}^\mathcal{T} \in \mathcal{H}_{C_0}$ and $\bm{W}^\star \in \mathcal{H}_{C_0}$, we have:
\begin{equation}
||\LGbi^{\mathcal{T}}- \LGbi^\star ||_F ~\le~ (\sqrt{d+T}+\sqrt{2}) \cdot ||\bm{W}^{\mathcal{T}}- \bm{W}^\star ||_F ~\le~ 2\cdot (\sqrt{d+T}+\sqrt{2}) \cdot C_0
\end{equation}

\noindent According to the definition of the spectral embeddings,  we have:
\begin{equation}
 \max_{(i,j)} |D^\mathcal{T}_{i,j} -D^\star_{i,j}| \le 4 \cdot \max_{(i,j)} |\bm{U}^{\mathcal{T}}_{i,j}- \bm{U}^\star_{i,j}| ~\le~ 8\sqrt{2} \cdot (\sqrt{d+T}+\sqrt{2}) \cdot \frac{C_0}{\lambda_{k+1}(\LGbi^\mathcal{T})} =  8 \sqrt{2} \xi.
\end{equation}
Equivalently, we have:
\begin{equation}
  D^\mathcal{T}_{i,j} \in [  D^\star_{i,j} - 8\sqrt{2} \xi,  D^\star_{i,j} + 8\sqrt{2} \xi ]
\end{equation}
Since $\Gbi^\star$ has k connected components, we have \[ D^\star_{i,j} = \begin{cases}
\dfrac{1}{n_{\mathcal{G}(i)}} + \dfrac{1}{n_{\mathcal{G}(j)}}  ~\ge~ \dfrac{1}{n^\uparrow_{1}} + \dfrac{1}{n^\uparrow_{2}} ~=~ \beta, & \mathcal{G}^\star(i) \neq \mathcal{G}^\star(j) \\
0, & \text{otherwise.}
\end{cases}\]

By assumption, we have:
\begin{equation}
\min_{(i,j):\mathcal{G}^\star(i) \neq \mathcal{G}^\star(j) }  D^\mathcal{T}_{i,j} \ge \beta - 8\sqrt{2} \xi  >  8\sqrt{2} \xi \ge \max_{(i,j):\mathcal{G}^\star(i) = \mathcal{G}^\star(j) }  D^\mathcal{T}_{i,j}. 
\end{equation}
For (a), since $ 8\sqrt{2} \xi < \delta_1 < \beta- 8\sqrt{2} \xi $, we have:
  \[\big\{(i,j): || \bm{f}^\mathcal{T}_i - \bm{f}^\mathcal{T}_{d+j} ||_2^2 < \delta_1 \big\} = \big\{(i,j): \mathcal{G}(i) = \mathcal{G}(j) \big\} \]
\noindent For (b), since  $ 8\sqrt{2}\xi < \min\left\{\delta_1,\delta_2\right\} \le  \max\left\{\delta_1,\delta_2\right\} < \beta - 8\sqrt{2} \xi $, we have:
\begin{equation*}
\big\{(i,j): || \bm{f}^\mathcal{T}_i - \bm{f}^\mathcal{T}_{d+j} ||_2^2 < \delta_1 \big\} = \big\{(i,j): || \bm{f}^\mathcal{T}_i - \bm{f}^\mathcal{T}_{d+j} ||_2^2 < \delta_2 \big\} = \big\{(i,j): \mathcal{G}(i) = \mathcal{G}(j) \big\}.
\end{equation*}

The proof is then complete following Lem.\ref{lem:sets}.
\qed
\end{pfthm6}

\section{The Optimization method for the Personalized Attribute Prediction Model}

\subsection{Convergence analysis}
 Now we prove that $\mathcal{J} = \sum_i \ell_i$ defined in Sec.\ref{sec:app_model} of the main paper has Lipschitz continuous gradients. \textit{Note that in the following we will alternatively use three equivalent notations for the empirical loss: $\mathcal{J}, \mathcal{J}(\boldsymbol{\varTheta})$ and $\mathcal{J}(\boldsymbol{\varTheta}_c, \boldsymbol{\varTheta}_g, \boldsymbol{\varTheta}_p)$}, with $\boldsymbol{\varTheta}= [\boldsymbol{\varTheta}_c; vec(\boldsymbol{\varTheta}_g); vec(\boldsymbol{\varTheta}_p)]$.
Now we prove that $\mathcal{J}$ has Lipschitz continuous gradients and a bound on the partial gradient w.r.t $\boldsymbol{\varTheta_g}$ holds.
\begin{lem}\label{lem:conti}
	If the data is bounded in the sense that:
	\[\forall i, ~\norm{\boldsymbol{X}^{(i)}}_2 =\vartheta_{X_i} < \infty, ~n_{+,i} \ge 1, ~n_{-,i} \ge 1.\]
	\begin{enumerate}[itemindent=0pt, leftmargin =15pt,label={(\arabic*)}]
		\item Given two arbitrary distinct  parameters $\boldsymbol{W}, \boldsymbol{W}'$, we have:
		\begin{equation*}
		\norm{\nabla{\mathcal{J}(\boldsymbol{\varTheta})} - \nabla{\mathcal{J}(\boldsymbol{\varTheta}')}}_F \leq \varrho_{\varTheta} \Delta \boldsymbol{\varTheta}
		\end{equation*}
		\item  For all $\infty >\xi_c >0$, $\infty >\xi_g >0$, $\infty >\xi_p >0$, we have:
		\begin{equation}
		\sup\limits_{\norm{\boldsymbol{\varTheta}_c}_2 \le \xi_c,\norm{\boldsymbol{\varTheta}_g}_F \le \xi_g, \norm{\boldsymbol{\varTheta}_p}_F \le \xi_p}\norm{\nabla_{\boldsymbol{\varTheta_g}}\mathcal{J}}_F \le \varkappa(\xi_c,\xi_g,\xi_p)
		\end{equation}
	\end{enumerate}

	where $ \varrho_{\varTheta}  
	=3T\sqrt{(2T+1)}\max_{i} \left\{\dfrac{n_i\vartheta^2_{X_i}}{n_{+,i}n_{-,i}}\right\}
	$, $\varkappa(\xi_c,\xi_g,\xi_p)
	=\dfrac{n_i\vartheta_{X_i}}{\sqrt{n_{+,i}}n_{-.i}}  \sum_{i=1}^T \left((\xi_c+\xi_g+\xi_p)\dfrac{\vartheta_{X_i}}{\sqrt{n_{+,i}}}+ 1 \right)
	$,

   $\boldsymbol{\varTheta}= [\boldsymbol{\varTheta}_c; vec(\boldsymbol{\varTheta}_g); vec(\boldsymbol{\varTheta}_p)]$, $\boldsymbol{\varTheta}'= [\boldsymbol{\varTheta}'; vec(\boldsymbol{\varTheta}_g'); vec(\boldsymbol{\varTheta}_p')]$ ,  $\Delta \boldsymbol{\varTheta} = \norm{\boldsymbol{\varTheta} -\boldsymbol{\varTheta}' }$.
\end{lem}
\begin{proof}{\color{white} dsadsa} \\
	
\noindent \textbf{proof of (1)}\\

\noindent Denote
	\begin{equation*}
	dL_{i} = \boldsymbol{X}^{(i)\top}\mathcal{L}_{AUC}^{(i)}\boldsymbol{X}^{(i)}(\boldsymbol{W}^{(i)} - \boldsymbol{W'}^{(i)}),\
	d_c = \nabla_{\boldsymbol{\varTheta}_c}\mathcal{J}(\boldsymbol{\varTheta}) -\nabla_{\boldsymbol{\varTheta}'_c}\mathcal{J}(\boldsymbol{\varTheta}'),\ d_g^{(i)} = \nabla_{\boldsymbol{\varTheta}_g^{(i)}}\mathcal{J}(\boldsymbol{\varTheta}) -\nabla_{{\boldsymbol{\varTheta}'}_g^{(i)}}\mathcal{J}(\boldsymbol{\varTheta}'),
	\end{equation*}
	\begin{equation*}
	d_p^{(i)} = \nabla_{\boldsymbol{\varTheta}_p^{(i)}}\mathcal{J}(\boldsymbol{\varTheta}) -\nabla_{{\boldsymbol{\varTheta}'}_p^{(i)}}\mathcal{J}(\boldsymbol{\varTheta}').
	\end{equation*}
	Note that:
	\begin{equation*}
	\boldsymbol{W}^{(i)} = \boldsymbol{\varTheta}_c+\boldsymbol{\varTheta}_g^{(i)} +\boldsymbol{\varTheta}_p^{(i)}, \
	\boldsymbol{W}'^{(i)} = \boldsymbol{\varTheta'}_c+\boldsymbol{\varTheta'}_g^{(i)} +\boldsymbol{\varTheta'}_p^{(i)}
	\end{equation*}
	It thus follows that 
	\begin{equation}\centering
	\begin{split}
	&~~~\norm{\nabla{\mathcal{J}(\boldsymbol{\varTheta})} - \nabla{\mathcal{J}(\boldsymbol{\varTheta}')}}\\
	=&~~\left(\norm{d_c}^2+ \sum\limits_{i}\norm{d_g^{(i)}}^2 + \sum\limits_{i}\norm{d_p^{(i)}}^2 \right)^{1/2}\\ {\leq}&~~\norm{d_c}+\sum\limits_{i}\norm{d_g^{(i)}} + \sum\limits_{i}\norm{d_p^{(i)}} \\
	=&~~ \norm{\sum\limits_{i=1}^{T}dL_{i}} + 2\sum\limits_{i=1}^{T}
	\norm{dL_{i}}
	\\\leq&~ ~3C_{max}\sum\limits_{i=1}^{T}\norm{\boldsymbol{W}^{(i)}- \boldsymbol{W}'^{(i)}}\\ \leq&~~ 3C_{max}T\Big(\norm{\boldsymbol{\varTheta}_c -\boldsymbol{\varTheta_c'}} +\sum_{i=1}^{T}\norm{\boldsymbol{\varTheta}_g^{(i)} -\boldsymbol{\varTheta_g^{(i)'} }}+\sum_{i=1}^{T}\norm{\boldsymbol{\varTheta_p}^{(i)} -\boldsymbol{\varTheta_p}^{(i)'}}  \Big)\\ {\leq}& ~~
	3C_{max}T\sqrt{2T+1}\norm{\boldsymbol{\varTheta} - \boldsymbol{\varTheta}' }
	\end{split}
	\end{equation}
	where
	\begin{equation*}
	C_{max} = \max\limits_{i}\left(\norm{\boldsymbol{X}^{(i)^\top}\mathcal{L}_{AUC}^{(i)}\boldsymbol{X}^{(i)}}_2\right).
	\end{equation*} 
	To prove the last inequality, we have:
	\begin{equation}\label{finieq1}
	\forall i, \norm{\boldsymbol{X}^{(i)^\top}\mathcal{L}_{AUC}^{(i)}\boldsymbol{X}^{(i)}}_2 \le \norm{\boldsymbol{X}^{(i)}}^2_2\norm{\mathcal{L}_{AUC}^{(i)}}_2.
	\end{equation} 
	and according to Thm. 3.3 of \cite{zhang2011laplacian}, we have:   
	\begin{equation}\label{finieq2}
	\norm{\mathcal{L}_{AUC}^{(i)}}_2=\frac{n_i}{n_{+,i}n_{-,i}}.
	\end{equation} 
	Combining (\ref{finieq1}) and (\ref{finieq2}), the last inequality holds, thus we complete the proof of (1).\\
	
\noindent \textbf{proof of (2)}\\ 

\noindent Omitting the subscript in the supremum, we have :
\begin{equation}
\begin{split}
  \sup\norm{\nabla_{\boldsymbol{\varTheta_g}}\mathcal{J}}_F  &\le  \sum_{i=1}^T \sup \norm{\nabla_{\boldsymbol{\varTheta_g}^{(i)}}\mathcal{J}}_2\\
  & \le \sum_{i=1}^T\sup \norm{\bm{X}^{(i)^\top} \mathcal{L}^{(i)}_{AUC}\bm{X}^{(i)}}_2 \cdot \norm{\bm{W}^{(i)}}_2 + \norm{\bm{X}^{(i)^\top} \mathcal{L}_{AUC}^{(i)}}_2 \cdot  \norm{\bm{\widetilde{y}}^{(i)}}_2 \\
  &\le  \frac{n_i\vartheta_{X_i}}{\sqrt{n_{+,i}}n_{-.i}}  \sum_{i=1}^T \left((\xi_c+\xi_g+\xi_p)\dfrac{\vartheta_{X_i}}{\sqrt{n_{+,i}}}+ 1 \right) =  \varkappa(\xi_c,\xi_g,\xi_p).
   \end{split}   
\end{equation}

\end{proof}

{\color{white}dsadsa}\\
\noindent Similar to the $(\bm{P})$ problem, here we define a surrogate problem for $(\bm{Q})$, which is named as $(\bm{Q}^\star)$
\begin{equation}
(\boldsymbol{Q}^\star) \min_{\boldsymbol{\varTheta}, \boldsymbol{U} \in \Gamma} ~ \mathcal{J}(\c,\g,\p)  + \frac{\alpha_1}{2}\norm{\boldsymbol{\varTheta}_c}_2^2 +  \alpha_2 \left<{\bm{\varTheta}_{\mathcal{G}}}, \boldsymbol{U}\right> + \frac{\alpha_3}{2} \norm{\g}_F^2  + \alpha_4 \norm{\boldsymbol{\varTheta}_p}_{1,2}  + \frac{\alpha_5}{2} \norm{\bm{U}}_F^2.
\end{equation}
Moreover, we denote the surrogate objective function:
\[{\mathcal{F}}(\c ,\g, \p, \bm{U}) = \widetilde{\mathcal{F}}(\c ,\g, \p, \bm{U}) + \frac{\alpha_5}{2} \norm{\bm{U}}_F^2 . \]
Obviously, we could use an algorithm similar to Alg.  \ref{alg:opt_app} to solve this problem, which enjoys the following properties. 
\begin{lem} \label{lem:cond_app} Denote by  ${\mathcal{F}}_t ={\mathcal{F}}(\ct ,\gt, \pt, \bm{U}^t)$ and denote by $(\ct, \gt, \pt,\boldsymbol{U}^t)$ the parameter obtained at iteration $t$, picking $C > \varrho_{\Theta}$, $0 < \alpha_5 < 2C\min_t \breve{\delta}(\LGbi^t) < +\infty $, the following properties holds:
	\begin{enumerate}[itemindent=0pt, leftmargin =15pt,label={(\arabic*)}]
		\item 	
		the sequence $\{ \mathcal{F}_{t} \}$ is non-increasing in the sense that: 
		\begin{equation*}
		\begin{split}
		&{\mathcal{F}}_{t+1} \le {\mathcal{F}}_{t} - \min\left\{\frac{C-\varrho_{\varTheta}}{2}, \frac{\alpha_5}{2}\right\} \cdot
		\bigg(\norm{\Delta(\boldsymbol{\varTheta}_c^t)}_F^2 + \norm{\Delta(\boldsymbol{\varTheta}_g^t)}_F^2+	\norm{\Delta(\boldsymbol{\varTheta}_p^t)}_F^2 + \norm{\Delta(\bm{U}^t)}_F^2\bigg),
		\end{split}
		\end{equation*}
		where $\Delta(\boldsymbol{\varTheta}_c^t) = \boldsymbol{\varTheta}_c^{t+1} -\boldsymbol{\varTheta}_c^t$, \ \  $\Delta(\boldsymbol{\varTheta}_g^t) = \boldsymbol{\varTheta}_g^{t+1} -\boldsymbol{\varTheta}_g^t$, \ \ $\Delta(\boldsymbol{\varTheta}_p^t) = \boldsymbol{\varTheta}_p^{t+1} -\boldsymbol{\varTheta}_p^t$, $\Delta(\bm{U}^t) =\bm{U}^{t+1} -\bm{U}^t$ .
\item $\sum_{i=1}^\infty \norm{\Delta(\boldsymbol{\varTheta}_c^t)}_F^2 + \norm{\Delta(\boldsymbol{\varTheta}_g^t)}_F^2+ \norm{\Delta(\boldsymbol{\varTheta}_p^t)}_F^2 + \norm{\Delta(\bm{U}^t)}_F^2 < \infty$. Furthermore, we have $\lim_{t\rightarrow} ||\Delta(\boldsymbol{\Theta}_c)||  =0 ~\lim_{t\rightarrow} ||\Delta(\boldsymbol{\Theta}_g)||  =0~  \lim_{t\rightarrow} ||\Delta(\boldsymbol{\Theta}_p)||  =0 ~ \lim_{t\rightarrow} ||\Delta(\bm{U})||  =0$.

    \item There exist a subsequence of $\{{\bm{\Theta}_c}^{k_j},~ {\boldsymbol{\Theta}_g}^{k_j}, ~{\boldsymbol{\Theta}_p}^{k_j},~\bm{U}^{k_j}\}_{j}$ with a limit point $\{{\bm{\Theta}_c}^\star,~{\bm{\Theta}_g}^\star,~{\bm{\Theta}_p}^\star,~\bm{U}^\star\}$, such that 
    \[\{{\bm{\Theta}_c}^{k_j},~ {\boldsymbol{\Theta}_g}^{k_j}, ~{\boldsymbol{\Theta}_p}^{k_j},~\bm{U}^{k_j}\} \rightarrow \{{\bm{\Theta}_c}^\star,~{\bm{\Theta}_g}^\star,~{\bm{\Theta}_p}^\star,~\bm{U}^\star\}, ~~ \text{and} ~~ \mathcal{F}({\bm{\Theta}_c}^{k_j},~ {\boldsymbol{\Theta}_g}^{k_j}, ~{\boldsymbol{\Theta}_p}^{k_j},~\bm{U}^{k_j}) \rightarrow  \mathcal{F}({\bm{\Theta}_c}^\star,~{\bm{\Theta}_g}^\star,~{\bm{\Theta}_p}^\star,~\bm{U}^\star). \]

\item ${\mathcal{F}}(\cdot,\cdot,\cdot,\cdot)$ is a $KL$ function.
\item  The subgradient satisfies:\begin{equation}
dist(\bm{0}, \partial_{\Theta}{\mathcal{F}}(\ct,\gt,\pt,\bm{U}^t)) \le \left[3\cdot(C+\varrho_{\Theta}) + \alpha_1(\sqrt{d}+\sqrt{T}+2) \right] \cdot \norm{\bm{\varTheta}^t - \bm{\varTheta}^{t-1}}.
\end{equation}		
	\end{enumerate}
\end{lem}
\begin{proof}
Notice that, since $\nabla\mathcal{J}$ is $\varrho_{\Theta}$-Lipschitz continuous, the following relation holds:
	\begin{equation}\label{lp-app}
	\begin{split}
	\mathcal{J}_{t+1} \le& 	\mathcal{J}_t  +  \left<\nabla_{\c}	\mathcal{J}_t, \Delta\left(\ct\right)\right> +   \left<\nabla_{\g}	\mathcal{J}_t, \Delta\left(\gt\right)\right> +  \left<\nabla_{\p}	\mathcal{J}_t, \Delta\left(\pt\right)\right>+ \frac{\varrho_{\Theta}}{2}\cdot\bigg[\norm{\Delta(\ct)}^2_F + \norm{\Delta(\gt)}^2_F + \norm{\Delta(\pt)}^2_F \bigg].
	\end{split}
	\end{equation}
The proof of (1) is similar to that of Lem. \ref{lem:cond} with Eq.(\ref{lp-app}) and the strong convexity of the subproblems. The proof of (2) is similar with the proof of (2) in Lem. \ref{lem:cond}, with the optimality condition for each subproblem. The proof of (3) is similar as the proof of (3) of Lem. \ref{lem:cond} with the Bolzano-Weierstrass theorem, the optimality condition for each subproblem and fact that the chosen loss is continuous and lower bounded away from 0. \\
Since the loss function $\mathcal{J}$ is a polynomial function with respect to $\c, \g , \p$, we have $\mathcal{J}$ is definable. Similarly, we have $\norm{\c}_2^2$ is definable. According to \cite{blockzeng}, we know that $||\g||_{1,2}$ is also definable. Combining with the proof in Lem. \ref{lem:cond}, we know $\mathcal{F}(\cdot, \cdot,\cdot, \cdot)$ is definable and thus a KL function, which ends the proof of (4).
\end{proof}
\begin{pfthm7}
The proof follows Lem. \ref{lem:cond_app}, and the  proof of Thm. \ref{thm:glob_star} and Thm. \ref{thm:global}.
\qed
\end{pfthm7}
\noindent Finally, one can easily extend the proof of Thm. \ref{thm:group} to proof Thm. \ref{thm:group_app}.

\begin{pfthm8}
	The proof follows Lem. \ref{lem:conti}-(2) and the proof of Thm. \ref{thm:group}. One only need to notice that, since $\widetilde{F}_t$ is non-increasing, we have:
	\begin{equation}
	\norm{\ct}_F \le \sqrt{\frac{\widetilde{F}_0}{\alpha_3}},  \ \ \norm{\gt}_F \le \sqrt{\frac{\widetilde{F}_0}{\alpha_2}}, \ \  \norm{\pt}_F \le \norm{\pt}_{1,2} \le {2 \cdot \frac{\widetilde{F}_0}{\alpha_4}} .
	\end{equation}

\qed
\end{pfthm8}
\end{document}